\pdfoutput=1
\documentclass{article}

\usepackage[margin=1in]{geometry}
\usepackage{natbib}

\usepackage[utf8]{inputenc}
\usepackage[T1]{fontenc}
\usepackage{hyperref}
\usepackage{url}
\usepackage{graphicx}
\usepackage{booktabs}
\usepackage{amsfonts}
\usepackage{nicefrac}
\usepackage{microtype}
\usepackage[dvipsnames]{xcolor}

\usepackage{wrapfig}

\usepackage{amsmath}
\usepackage{amssymb}
\usepackage{mathtools}
\usepackage{amsthm}
\mathtoolsset{showonlyrefs}

\theoremstyle{plain}
\newtheorem{theorem}{Theorem}[section]
\newtheorem{proposition}[theorem]{Proposition}
\newtheorem{lemma}[theorem]{Lemma}

\theoremstyle{definition}

\newtheorem{assumption}[theorem]{Assumption}
\theoremstyle{remark}
\newtheorem{remark}[theorem]{Remark}

\RequirePackage{amsmath, amsfonts, amssymb, amsthm}
\RequirePackage{mathtools}
\RequirePackage{esint}
\RequirePackage{braket}
\RequirePackage{subcaption}
\RequirePackage{float}
\RequirePackage{caption}

\newcommand{\rev}[1]{{#1}}

\numberwithin{equation}{section}

\def\rr{{\mathbb{R}}}

\def\nn{{\mathbb{N}}}

\def\ee{{\mathbb{E}}}

\def\P{{\mathbb{P}}}

\def\cE{{\cal E}}

\def\cL{{\cal L}}

\def\cN{{\cal N}}

\def\var{{\textrm{Var}}}
\def\cov{{\textrm{Cov}}}

\newcommand{\indep}{\perp \!\!\! \perp}

\makeatletter
\providecommand*{\diff}%
{\@ifnextchar^{\DIfF}{\DIfF^{}}} \def\DIfF^#1{%
\mathop{\mathrm{\mathstrut d}}%
        \nolimits^{#1}\gobblespace}

\def\gobblespace{%
        \futurelet\diffarg\opspace}
\def\opspace{%
        \let\DiffSpace\!%
        \ifx\diffarg(%
            \let\DiffSpace\relax
        \else
            \ifx\diffarg[%
               \let\DiffSpace\relax
            \else
               \ifx\diffarg\{%
                   \let\DiffSpace\relax
               \fi\fi\fi\DiffSpace}

\renewcommand{\d}{\diff}

\DeclareMathOperator{\diag}{diag}

\DeclareMathOperator{\tr}{\textrm{tr}}

\newcommand{\softmax}{\mathsf{softmax}}

\DeclarePairedDelimiterX{\ip}[2]{\langle}{\rangle}{#1, #2}

\DeclareMathOperator{\interior}{int}

\newcommand{\comp}{\mathsf{c}}

\newcommand{\iidsim}{\stackrel{\text{i.i.d.}}{\sim}}
\newcommand{\rc}{\zeta}

\newcommand{\bx}{\boldsymbol{x}}
\newcommand{\by}{\boldsymbol{y}}
\newcommand{\bz}{\boldsymbol{z}}
\newcommand{\bu}{\boldsymbol{u}}
\newcommand{\bv}{\boldsymbol{v}}
\newcommand{\br}{\boldsymbol{r}}
\newcommand{\bfm}{\boldsymbol{m}}
\newcommand{\bzero}{\boldsymbol{0}}
\newcommand{\be}{\boldsymbol{e}}
\newcommand{\bb}{\boldsymbol{b}}
\newcommand{\ba}{\boldsymbol{a}}

\newcommand{\bxi}{\boldsymbol{\xi}}
\newcommand{\bta}{\boldsymbol{\eta}}
\newcommand{\beps}{\boldsymbol{\varepsilon}}
\newcommand{\bfmu}{\boldsymbol{\mu}}

\newcommand{\bw}{\boldsymbol{w}}
\newcommand{\f}{\boldsymbol{f}}

\newcommand{\bX}{\boldsymbol{X}}
\newcommand{\bE}{\boldsymbol{E}}
\newcommand{\bW}{\boldsymbol{W}}
\newcommand{\bfM}{\boldsymbol{M}}
\newcommand{\bQ}{\boldsymbol{Q}}
\newcommand{\bK}{\boldsymbol{K}}
\newcommand{\bV}{\boldsymbol{V}}
\newcommand{\bI}{\boldsymbol{I}}
\newcommand{\bF}{\boldsymbol{F}}
\newcommand{\bA}{\boldsymbol{A}}
\newcommand{\bB}{\boldsymbol{B}}

\newcommand{\bD}{\boldsymbol{D}}
\newcommand{\bS}{\boldsymbol{S}}
\newcommand{\bU}{\boldsymbol{U}}
\newcommand{\bLam}{\boldsymbol{\Sigma}_{\boldsymbol{x}}}
\newcommand{\bSig}{\boldsymbol{\Sigma}}
\newcommand{\bGam}{\boldsymbol{\Gamma}}

\newcommand{\one}{\mathbf{1}}

\hypersetup{
  colorlinks=true,
  linkcolor=MidnightBlue,
  citecolor=MidnightBlue,
  urlcolor=MidnightBlue
}

\usepackage{authblk}

\title{Multi-layer Cross-attention is Provably Optimal for \\ Multi-modal In-context Learning}

\author[1]{Nicholas Barnfield}
\author[1]{Subhabrata Sen}
\author[1]{Pragya Sur}

\affil[1]{Department of Statistics, Harvard University}

\date{\today}

\begin{document}

\maketitle

\begin{abstract}
  Recent progress has rapidly advanced our understanding of the mechanisms underlying in-context learning in modern attention-based neural networks.  However, existing results focus exclusively on unimodal data; in contrast, the theoretical underpinnings of in-context learning  for multi-modal data remain poorly understood.
We introduce a mathematically tractable framework for studying multi-modal learning and explore when transformer-like architectures can recover Bayes-optimal performance in-context. To model multi-modal problems, we assume the observed data arises from a  latent factor model. Our first result comprises a negative take on expressibility: we prove that single-layer self-attention models fail to recover the Bayes-optimal predictor uniformly over the task distribution. To address this limitation, we introduce a novel, linearized cross-attention mechanism, which we study in the regime where both the number of cross-attention layers and the context length are large. We show that this cross-attention mechanism  is provably Bayes optimal when optimized using gradient flow.  Our results underscore the benefits of depth for in-context learning and establish the provable utility of cross-attention for multi-modal distributions.
\end{abstract}

\section{Introduction}

Large language models exhibit striking ``in-context learning'' (ICL) behavior \citep{brown2020languagemodelsfewshotlearners}: given a sequence of input-output pairs for a new task, the model can infer the rule (the function from input to output) without any update to the model weights. This phenomenon has motivated a growing body of  theoretical work that tries to understand when and how ICL succeeds in linearized attention-based models \citep{train_transformers_bartlett,  Lu_2025, zhang2025trainingdynamicsincontextlearning}.

Prior work on ICL has focused on unimodal data; in contrast, modern foundation models routinely process multi-modal data, such as text, images, video \citep{li2024multimodal}, multi-omic data \citep{cui2025towards}, etc.
In this work, we formulate and study an ICL problem on multi-modal data. Our main contributions are as follows:

\begin{itemize}
\item   We introduce a statistical model for ICL on multi-modal data (Section \ref{sec:data}). The proposed data distribution is a latent factor model where the factors are dependent across the data modalities, thereby capturing cross-modal commonalities and differences. \rev{Reframing the objective as in-context regression, we show the multi-modal setup induces a covariance with a random spectrum that is coupled to the noise and task vectors, presenting new mathematical challenges.}

\item We show that a single layer of linear self-attention (LSA) or softmax self-attention (SA) with finitely many heads fails to achieve Bayes-optimal performance in-context over this class of data distributions, and the softmax case reduces to the linear one (Section \ref{sec:ICL_LSA}).

\item  We propose a multi-layer model with cross-attention (CA) plus SA and skip connections to solve the ICL problem in this context (Section \ref{sec:CA_model}). Linearizing the attention mechanism, we show that, even with a restriction of the model parameters, gradient flow on the population loss converges to the Bayes-optimal ICL predictor (Section \ref{sec:main_results}).

\item We complement our theory with numerical experiments in  controlled and benchmark settings (Section~\ref{sec:numerics}). Synthetic experiments verify the qualitative predictions from our analyses. Experiments on corrupted avMNIST provide evidence that the new mechanisms identified by our theory are useful for multi-modal learning in conventional deep-learning architectures. On the velocity task of the cosmological benchmark CosmoBench \citep{huang2025cosmobench}, a $20$--$25\%$ relative improvement in test $R^2$ over the benchmark's baselines demonstrates our model's utility on hierarchical models.
\end{itemize}

To highlight the technical challenge in our setting, note that prior theory \rev{on optimization} for ICL in regression assumes covariates sampled from a fixed distribution (usually Gaussian), independent of the learning task \citep{train_transformers_bartlett, zhang2025trainingdynamicsincontextlearning, Lu_2025}\rev{, or at most from random rotations of a fixed spectrum \citep{bordelon2025theory}}. Indeed, \citep{train_transformers_bartlett, zhang2025trainingdynamicsincontextlearning} note that the invariance of the covariate distribution across prompts is crucial in establishing the success of single-layer LSA. In sharp contrast, \rev{in our multi-modal model the covariance spectrum itself varies across prompts}, and the prediction rule in a prompt is naturally coupled with the covariate distribution.

These differences are significant, and provably rule out any single layer of attention in our setup, linear or softmax (Theorem~\ref{thm:failure_LSA}). To overcome this, we introduce a novel multi-layer model based on linear cross-attention (LCA) and LSA, and show that it achieves Bayes-optimal performance for multi-modal ICL (Theorems~\ref{thm:one_param_optimal} and~\ref{thm:two_param_optimal}). Our analyses deviate substantially from prior work, requiring new ideas to analyze the loss landscape in multi-modal problems. Our architecture differs from standard ones in the empirical literature in critical ways, namely in how we employ the CA mechanism and additional skip connections (Section~\ref{sec:CA_model}). Our work is the first to provide theoretical guarantees on versions of attention-based models that achieve Bayes-optimal performance in multi-modal ICL.

\section{Prior work} \label{sec:prior_works}

The first observation of ICL in large language models \citep{brown2020languagemodelsfewshotlearners} motivated significant research on its foundations. Early work studied the expressivity of transformers
\citep{bai2023,akyürek2023learningalgorithmincontextlearning,garg2023transformerslearnincontextcase}, showing that they can implement common statistical algorithms in context.
Connections with meta-learning and variants of gradient descent were noted in \citep{vonoswald2023transformerslearnincontextgradient,ahn2023transformerslearnimplementpreconditioned,zhang2024icl_transformer_MLP}, generalization and stability were studied in \citep{li2023transformers}, and a Bayesian interpretation was offered in \citep{xie2022explanationincontextlearningimplicit}.

Training dynamics of attention-based models and ICL performance of the trained models have also been investigated.
Perhaps the closest to our work is \citep{train_transformers_bartlett}, which studies ICL for regression and shows that a single-layer LSA model trained by gradient flow on the population loss learns the Bayes-optimal predictor. Gradient flow or descent dynamics for multi-head self-attention were later analyzed in \citep{chen2024training,zhang2025trainingdynamicsincontextlearning}, and \citep{huang2023incontextconvergencetransformers} goes beyond linearized attention to stylized one-layer softmax transformers trained by gradient descent.
\rev{\citet{bordelon2025theory} analyze gradient flow for deep linear attention when the covariance is a random rotation of a fixed spectrum, and show that depth then improves ICL. In our model, the spectrum itself is random (Section~\ref{sec:ICL_LSA}). \citet{gatmiry2024looped} show that gradient flow trains looped linear transformers to multi-step preconditioned gradient descent for a fixed covariance. \citet{gatmiry2024depth} prove that depth must grow with the condition number when the spectrum varies, without analyzing training.}

More recently, ICL has been studied for problems ranging from Gaussian mixture classification and clustering to non-parametric regression \citep{shen2025_gaussian_mixtures,maulensoto2025attentionbasedclustering,ma2025provable,ching2026efficient}. Attention-based models have also been studied as Gaussian sequence multi-index models \citep{cui2024phase,arnaboldi2025asymptotics,troiani2025fundamental} and on sparse-token classification tasks \citep{oymak2023, barnfield2025}, though not with a focus on ICL. ICL at infinite token dimension was studied in \citep{Lu_2025, letey2025pretraintesttaskalignmentgoverns};
see also \citep{wu2024pretrainingtasksneededincontext} for the effect of the number of pre-training tasks at finite token dimension. All of these works, however, consider unimodal data, whereas we focus on ICL for multi-modal data.

In diverse artificial intelligence problems, multi-modal learning---combining information from varied modes, e.g., text, image, video---is known to improve prediction, reasoning and learning capabilities. Although examined extensively in empirical studies \citep{radford2021learningtransferablevisualmodels,alayrac2022flamingovisuallanguagemodel,jaegle2021perceiver,wang2024charxiv}, this setting is far from rigorously understood, and bridging this theoretical gap is a primary aim of our manuscript. Multi-modal learning is useful when the modes share common information. To capture this, we use latent variable models, widely used in statistical estimation for multi-modal data \citep{nandy2024multimodal,ding2022cooperative,mergny2025spectral,keup2025optimal,deshpande2018contextual,yang2025fundamental,sergazinov2025spectral}, but these works do not study ICL or attention-based models. On the theoretical end, \citep{liu2025continualmultimodalcontrastivelearning,gui2025multi,cai2024contrastive} study multi-modal contrastive learning, but do not focus on ICL.

\section{Problem setup} \label{sec:data}

We study  a multi-modal framework where each prompt (task) provides a context of input–output examples; the inputs contain features from two distinct modalities. To model a natural multi-modal setting, we assume the data is drawn from a latent variable model. Unlike prior work on ICL, we allow the distribution of covariates to differ across tasks.

Formally, we observe $N$ prompts (tasks) with (training) context length $L$. For each prompt $j \in [N]$, we observe a sequence $(\bar \bz_i^{(j)}, \tilde \bz_i^{(j)}, y_i^{(j)})_{i=1}^L$ where $y^{(j)}_i \in \rr$ is a response and $\bar \bz_i^{(j)} \in \rr^{d_1}, \tilde \bz_i^{(j)} \in \rr^{d_2}$ are two modalities of covariates, e.g., the modalities could represent an image and associated text annotation, respectively.
In each prompt, we additionally observe a tuple $(\bar{\bz}_q^{(j)}, \tilde{\bz}_q^{(j)}, y_q^{(j)})$. We input the prompt and covariate pair $(\bar \bz_q^{(j)}, \tilde \bz_q^{(j)})$ into the model, and it outputs a prediction $\hat{y}_q^{(j)}$ for  $y_q^{(j)}$. While training, the model weights are tuned to minimize the error in predicting the $y_q$ values on the training prompts. To assess the ICL performance of the fitted model, we consider a fresh prompt, and study the prediction error of the fitted model on the $y_q$ associated with the test prompt.

We now turn to a statistical model for the data. The covariates and response are generated via
\begin{align}\label{eq:data_dist}
    \bar \bz_{i}^{(j)} = u_i^{(j)} \bv^{(j)} + \bxi_i^{(j)}, \,\,\,\,
    \tilde \bz_{i}^{(j)} = u_i^{(j)} \br^{(j)} + \bta_i^{(j)}, \,\,\,\,
    y_i^{(j)} = \rc^{(j)} u_i^{(j)}.
\end{align}
The latent variable $u_i^{(j)} \iidsim \mathcal{N}(0,1)$ is shared across both modalities and the response. The prompt-specific regression coefficient, which relates $u_i^{(j)}$ to $y_i^{(j)}$, is given by $\rc^{(j)} \in \rr$. We assume that across prompts, $\rc^{(j)}$ are sampled i.i.d. from a distribution with mean $0$ and variance $1$. On the other hand, $\bv^{(j)} \in \mathbb{R}^{d_1}$ and $\br^{(j)} \in \mathbb{R}^{d_2}$ are modality-specific vectors sampled i.i.d.~across prompts. We defer precise assumptions about their distributions (Assumption \ref{assump:support_m}). Finally, $\bxi_i^{(j)} \iidsim \mathcal{N}_{d_1}(0, \bI)$ and $\bta_i^{(j)} \iidsim \mathcal{N}_{d_2}(0,\bI) $ are i.i.d.~noise variables, and all unique variables on the right-hand side of \eqref{eq:data_dist} are mutually independent. The setup in \eqref{eq:data_dist} formalizes  a fundamental idea underlying multi-modal machine learning: each modality provides a different (noisy) view of a common latent feature $u$ which drives the response $y$.
\rev{Gaussianity in \eqref{eq:data_dist} is used for linearity of the Bayes predictor in the query (Appendix~\ref{app:bayes}) and for the reduction of softmax to linear attention (Lemma~\ref{lem:tilt}). The remaining arguments use only moment bounds on $u$, $\bxi$ and $\bta$.}

Focusing on a single task,  omitting the prompt superscript, stacking the covariates, we denote $\bx_i := \begin{bmatrix}
    \bar \bz_i &
    \tilde \bz_i
\end{bmatrix}^\top$
and let $d = d_1 + d_2$. Then, in each prompt, the covariates and response are jointly Gaussian:
\begin{equation}\label{eq:m_def_and_gaussianity}
\begin{bmatrix}
    \bx_i \\
    y_i
\end{bmatrix} \iidsim \cN_{d + 1} \left( 0 , \bSig \right) \,\,\,\, \text{with} \,\,\,\, \bSig = \begin{bmatrix}
    \bI + \bfm \bfm^\top & \rc \bfm \\
    \rc \bfm^\top & \rc^2
\end{bmatrix}  \,\,\,\, {\rm and} \,\,\,\, \bfm := \begin{bmatrix}
    \bv \\
    \br
\end{bmatrix}.
\end{equation}
Note, the same construction extends directly to three or more modalities by stacking all views into a single vector $\bx_i$. The joint Gaussianity of the covariates and response allows us to write
\begin{equation}\label{eq:regression_view}
y_i = \ip{\bw}{\bx_i} + \varepsilon_i \quad {\rm where} \quad \bw = \frac{\rc}{1 + \|\bfm\|^2} \bfm \quad {\rm and} \quad \varepsilon_i \iidsim \cN \left(0,\; \rc^2/(1 + \|\bfm\|^2 )\right).
\end{equation}
Our goal is to learn the task-specific Bayes-optimal predictor $\langle w, x \rangle$.
The setup \eqref{eq:regression_view} is reminiscent of the regression setting considered in \citep{train_transformers_bartlett, Lu_2025, zhang2025trainingdynamicsincontextlearning}, yet presents several differences that make the learning problem significantly more difficult. Upon setting
\begin{equation}\label{eq:covspike}
\bLam := \cov(\bx \mid \bfm) = \bI + \bfm\bfm^\top
\end{equation}

as the covariance of the covariates $\{\bx_i\}_{i=1}^L \cup \{\bx_q\}$ conditioned on $\bfm$, we note that the distribution of covariates varies across prompts, since $\bfm$ is  sampled anew for each task. Furthermore, the task vector $\bw$ and covariate distribution are coupled through $\bfm$. Similarly, the variance of the additive noise and the task vector are coupled through the task-specific parameter $\rc$. These dependencies invalidate prior ICL theory in our setup. As we will see below, this leads to the failure of single-layer attention, including the linearized models previously shown to be successful for ICL when the covariate distribution does not change across prompts \citep{train_transformers_bartlett, zhang2025trainingdynamicsincontextlearning}.

\section{ICL and failure of single-layer attention} \label{sec:ICL_LSA}

In this section, we introduce ICL formally and show that single-layer attention fails at the ICL task. To this end,
we represent a single prompt by its covariate matrix $\bX := [\bx_1,\dots,\bx_L]\in\mathbb{R}^{d\times L}$ and target vector $\by := [y_1,\dots,y_L]^\top\in\mathbb{R}^{L}$.
We define the (non-learned) joint embedding matrix
\[
\bE_{\bX} := \begin{bmatrix}\bX & \bx_q\\ \by^\top & 0\end{bmatrix} = \begin{bmatrix}
    \bx_1 & \bx_2 & \cdots & \bx_L & \bx_q \\
    y_1 & y_2 & \cdots & y_L & 0
\end{bmatrix}\in\mathbb{R}^{(d+1)\times(L+1)}
\]

A model maps $\bE_{\bX}$ to a scalar prediction $\hat y_q$ for $y_q$. So far, the context length $L$ is a free variable. In the subsequent discussion, it will be helpful to assume that the context length $L=L_{\rm tr}$ for prompts in the training data and $L=L_{\rm te}$ at test time. Our model architecture is independent of the context length, and thus this distinction is mostly for mathematical convenience.

\paragraph{Formal ICL criterion.} For a new prompt with test context length $L= L_{\mathrm{te}}$, let $\hat y_q^{(\mathrm{te})}$ denote the model prediction. We say a model \emph{in-context learns} if
\begin{equation} \label{eq:icl-def}
\hat y_q^{(\mathrm{te})} \longrightarrow
y_q^{\mathrm{Bayes},\mathrm{(te)}} := \ee[y_q^{(\mathrm{te})}|\bfm^{(\mathrm{te})}, \rc^{(\mathrm{te})}, \bx_q^{(\mathrm{te})}] = \ip{\bw^{(\mathrm{te})}}{\bx_q^{(\mathrm{te})}}
\end{equation}
almost surely (a.s.) as $L_{\mathrm{te}}\to\infty$,
where the a.s. statement is with respect to the randomness of the new prompt.  In words, ICL in our framework is synonymous with the model recovering the \emph{Bayes-optimal predictor} $\ip{\bw^{(\mathrm{te})}}{\bx_q^{(\mathrm{te})}}$
a.s.~in the asymptotic  $L_{\mathrm{te}} \to \infty$ limit.
The asymptotic context-length criterion \eqref{eq:icl-def} matches the mathematical interpretation of ICL in prior work \citep{train_transformers_bartlett}, yet unlike in the linear regression setup \citep{garg2023transformerslearnincontextcase, train_transformers_bartlett, Lu_2025}, the task vector $\bw$, covariate distribution, and error variances are all coupled via latent variables $\bfm$, $\rc$. These dependencies make both the learning problem---and its analysis---substantially more difficult.

\paragraph{Single-layer attention baselines.} As a starting point, we study two simple baselines and show that neither can in-context learn in our setting. The first is a single LSA layer acting on $\bE_{\bX}$ as seen in \citep{vonoswald2023transformerslearnincontextgradient, ahn2023transformerslearnimplementpreconditioned, train_transformers_bartlett, Lu_2025}. We write
\begin{equation}\label{eq:LSA_def}
\hat y_q = (\mathsf{LSA}(\bE_{\bX};\theta))_{d+1, L+1} \,\,\,\,\, {\rm where}\,\,\,\, \mathsf{LSA}(\bE_{\bX};\theta) = \bE_{\bX} + \bW^{PV} \bE_{\bX} \cdot (\bE_{\bX}^\top \bW^{KQ}\bE_{\bX}/L)
\end{equation}
and $\theta = \{\bW^{PV}, \bW^{KQ}\}$ are learnable weight matrices. This model is a linearized analogue of the second baseline, the SA block
\begin{equation}\label{eq:SA-def}
\mathsf{SA}(\bE_{\bX};\theta) =\bE_{\bX} + \bW^P \bW^V\bE_{\bX} \cdot \softmax\left((\bW^K\bE_{\bX})^\top \bW^Q\bE_{\bX}\right)
\end{equation}
---where now $\theta = \{\bW^{P},\bW^V, \bW^{K}, \bW^Q\}$---without the MLP layer and layer normalization, and with tied parameters  $\bW^{KQ} = (\bW^K)^\top \bW^Q$ and $\bW^{PV} = \bW^P \bW^V$. Collapsing the key and query, and the projection and value weight matrices into single matrices differs from a practical transformer implementation, but is mathematically identical with regard to expressibility.

In demonstrating the success of the LSA model in unimodal ICL tasks, prior work \citep{train_transformers_bartlett} relies on proving that the parameters $\theta$ can emulate a global, inverse, covariance $\bLam^{-1}$. However, in the multi-modal setting, as the covariance $\bLam$ is random, no fixed parameter $\theta$ can perform this inversion role and hence single-layer models fail at the ICL task as shown in Theorem~\ref{thm:failure_LSA}.

\begin{theorem}[Single-layer attention fails at ICL] \label{thm:failure_LSA}
    In the setting of Section~\ref{sec:data}, and assuming $\|\bfm\|$ and $\rc$ are atomless \footnote{For instance, any distribution on $\bfm$ that is absolutely continuous with respect to the Lebesgue measure in $\rr^d$. }, neither the LSA model \eqref{eq:LSA_def}, with $\hat y_q=\left(\mathsf{LSA}(\bE_{\bX};\theta)\right)_{d+1,L_{\rm te}+1}$ and fixed $\theta = \{\bW^{PV}, \bW^{KQ}\}$, nor the SA model \eqref{eq:SA-def}, with any finite number of heads, arbitrary linear key, query, value and projection maps, a residual connection and an affine readout (see Appendix~\ref{app:softmax-failure}), can in-context learn in the sense of \eqref{eq:icl-def}. In each case, for every fixed parameter configuration, $\mathbb{P}\left( \lim_{L_{\rm te} \to \infty}\hat y_q^{\mathrm{(te)}} =y_q^{\mathrm{Bayes},\mathrm{(te)}}  \right) = 0$.
\end{theorem}

The above negative result (see proof in Appendix \ref{app:failure_LSA_proof}) motivates examining more complex and deeper architectures for multimodal problems. \citet{train_transformers_bartlett} demonstrates, for a particular shifting-covariate distribution, that the LSA model with weights obtained via gradient flow fails to in-context learn. In contrast, Theorem \ref{thm:failure_LSA} establishes a negative result in full generality, showing that no parameter configuration will yield success at the ICL task under the multi-modal setting. Under the Gaussian latent factor model of Section \ref{sec:data}, a softmax head has exactly the same infinite-context predictor as its linear counterpart (Proposition \ref{prop:softmax-equiv}), so replacing linear attention by softmax does not enlarge the attainable class at a single layer.

\rev{\paragraph{Random versus fixed spectrum.} The randomness of the covariance spectrum is essential to Theorem~\ref{thm:failure_LSA}. If the covariates shared the same covariance spectrum across prompts, a single LSA layer tuned to it would be Bayes optimal (Appendix~\ref{app:fixed-spectrum}). No single attention layer can accommodate a prompt-varying spectrum coupled to the task vector and the noise, which motivates the next section.}

\section{A multi-layer CA model} \label{sec:CA_model}

Motivated by Theorem \ref{thm:failure_LSA}, we propose a model where the embedding matrix is now learned and obtained from a deep network of layers akin to CA. Consider the model
\begin{equation}
\hat y_q =\mathsf{SA}(\bE_{\bF};\theta)_{d+1, L+1} \quad {\rm where} \quad \bE_{\bF} := \begin{bmatrix}
    \bF & \bx_q\\ \label{eq:CA_model}
    \by^\top & 0
\end{bmatrix} \quad
 {\rm and} \quad \bF := \mathsf{CA}_T(\bX;\gamma).
\end{equation}
The above model \eqref{eq:CA_model} is a composition of a self-attention layer applied to an embedding $\bE_{\bF}$  that is in turn obtained by applying multiple layers of CA between the covariates $\bX$ and an evolving distribution $\{\bF_t\}_{t \geq 0}$ defined below. The final SA layer is defined analogously to \eqref{eq:SA-def} with parameters $\theta$ and where $\bE_{\bF}$ now replaces $\bE_{\bX}$. To define the CA model $\mathsf{CA}_T(\bX;\gamma)$, consider the recurrence
\begin{equation}\label{eq:recurrence}
    \bF_t = \bF_{t-1} + \bS_{t-1} + \bA_{t-1 }, \,\, \text{for} \,\, t=1, \dots, T. \quad \text{Define} \quad \bF := \bF_T := \mathsf{CA}_T(\bX;\gamma).
\end{equation}
Above, $\bA_{t-1}$ is the output of an attention head,
\begin{align} \label{eq:def_attn}
\mathsf{A}_{t-1} &= \mathsf{A}(\bQ_{t-1}, \bK_{t-1}, \bV_{t-1})
:= \bV_{t-1} \cdot \softmax \left( \bK_{t-1}^\top \bQ_{t-1} \right), \\
\bV_{t-1} &= \bW^V_{t-1}\bX,\,\,\, \bK_{t-1} = \bW^K_{t-1} \bX, \,\,\, \bQ_{t-1} = \bW^Q_{t-1} \bF_{t-1}.
\end{align}
\begin{wrapfigure}{r}{0.42\textwidth}
    \centering
    \includegraphics[width=\linewidth]{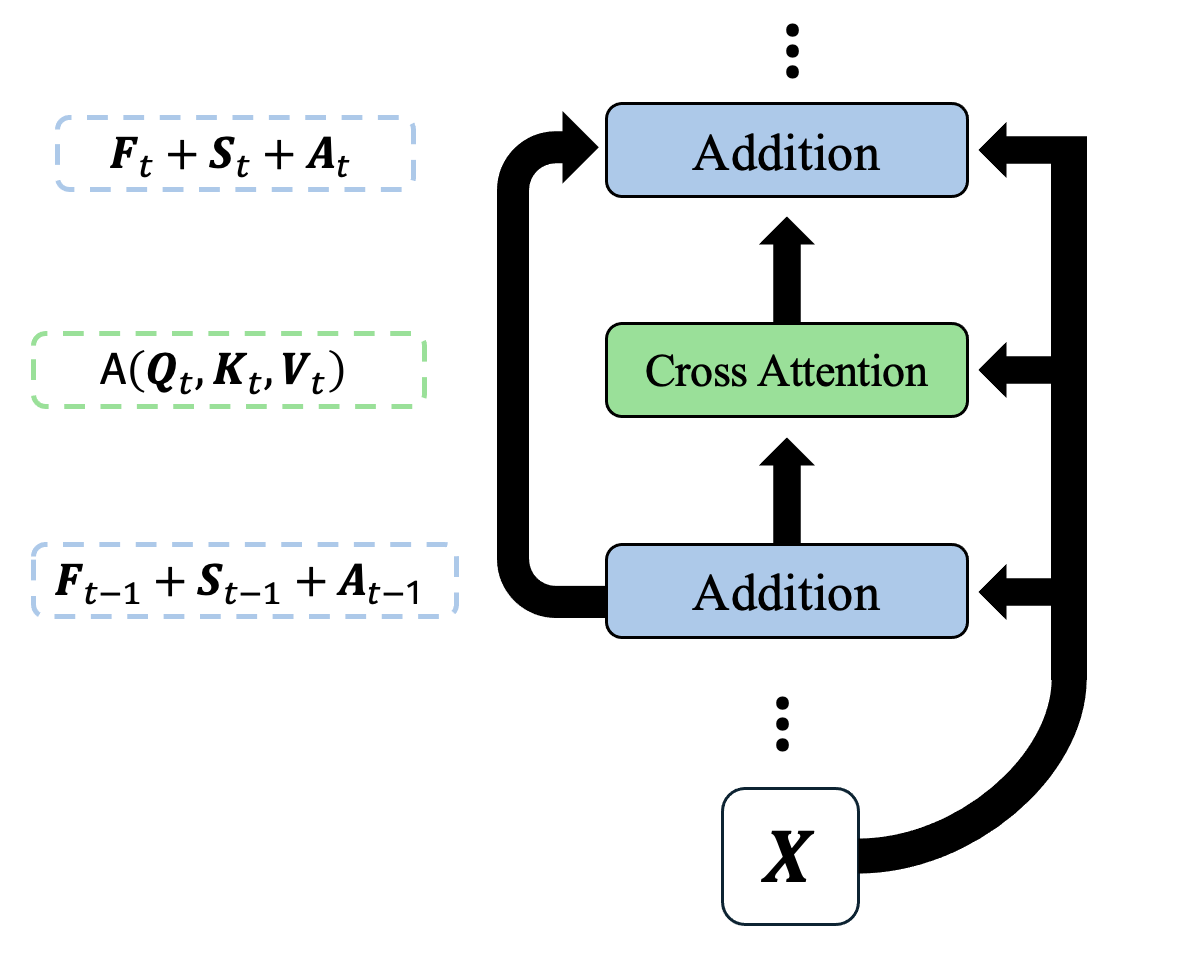}
    \caption{Re-injection of covariates $\bX$ throughout the $\mathsf{CA}_T$ ($\mathsf{LCA}_T$) embeddings. The propagation of the raw data $\bX$ occurs both through $\bS_t$ as well as through the CA  block $\bA_t$. }
    \label{fig:model}
    \vspace{-0.5\baselineskip}
\end{wrapfigure}
Here, $ \bW^V_{t-1},  \bW^K_{t-1}, \bW^Q_{t-1}$ are learnable weight matrices, and $\bS_{t-1} = \bW^S_{t-1}\bX$ with learnable weight matrix $\bW^S_{t-1}$. We take $\bF_0 = 0_{d \times L}$ to begin the recurrence and set $\bF = \bF_T$ with $T$ denoting the depth of the CA embedding. The weights of the embedding are $\gamma = \{\bW_t^V, \bW_t^K, \bW^Q_t, \bW^S_t\}_{t=0}^{T-1}$.

Notice that since the query matrix $\bQ_{t-1}$ depends on $\bF_{t-1}$, $\bA_{t-1}$ is an application of CA between the original data matrix $\bX$ and an evolving data distribution $\bF_{t-1}$. For this reason, we view \eqref{eq:CA_model} as a $T$-layer CA model followed by a single layer of SA to obtain the final model output. The model includes the skip-connection of the previous layer input $\bF_{t-1}$ in each layer \eqref{eq:recurrence} resembling standard transformers. Diverging from conventional architectures \citep{vaswani2023attentionneed}, we also include a learnable skip-connection $\bS_{t-1}$ that injects the unaltered raw data $\bX$ into each layer as displayed in Figure \ref{fig:model}. Taking inspiration from unrolled neural networks \citep{Gregor2010LearningFast, monga2020algorithmunrollinginterpretableefficient}, the inclusion of $\bS_{t-1}$ was preceded by theoretical considerations---specifically, to yield embeddings $\bE_{\bF}$ with desirable properties---and later supported by ablation experiments (see Section \ref{sec:numerics} and Appendix \ref{app:ablations}). As earlier, we use the bottom-right corner of the final output as our prediction $\hat y_q$.

\rev{Unlike conventional multimodal transformers, where cross-attention mediates interactions between different modalities \citep{tanbansal2019lxmert,lu2019vilbert}, our CA block takes, at every layer, its queries from the evolving state $\bF_t$ and its keys and values from the fixed multimodal sequence $\bX$ itself.} Together with the layer-wise re-injection term $\bS_t$, this gives the model persistent access to the original input, heuristically enabling an iterative correction of the covariate representation; we examine the empirical relevance of these new architectural proposals through ablations in Section~\ref{sec:numerics}.

To ensure the analysis remains tractable, we replace $\softmax(\cdot)$ with the normalized identity $(\cdot)/L$, along the lines of Section \ref{sec:ICL_LSA} and prior theoretical literature on ICL \citep{vonoswald2023transformerslearnincontextgradient, ahn2023transformerslearnimplementpreconditioned, train_transformers_bartlett,Lu_2025}. This yields the linearization
\begin{equation}
\hat y_q =\mathsf{LSA}(\bE_{\bF};\theta)_{d+1, L+1} \quad {\rm where} \quad \bE_{\bF} = \begin{bmatrix}
    \bF & \bx_q\\
    \by^\top & 0
\end{bmatrix} \label{eq:LSA_LCA_model} \quad   \bF = \mathsf{LCA}_T(\bX;\gamma),
\end{equation}
and $\mathsf{LCA}_T(\bX;\gamma)$  denotes linearized CA where $\softmax(\cdot)$ is replaced by the normalized identity $(\cdot)/L$ in \eqref{eq:def_attn}.

For mathematical simplification, we do not include two key components of conventional transformers---layer normalization \citep{xiong2020layernormalizationtransformerarchitecture} and non-linear activations (MLP layers). In Section~\ref{sec:numerics} (Table \ref{tab:avmnist_main}), we incorporate these components, together with other refinements, into our architecture to empirically validate the mechanisms motivated by our theory. That said, note the model \eqref{eq:LSA_LCA_model} is highly nonlinear in $\bX $ and the final LSA readout is itself a non-linear (cubic) transformation of the learned embedding $\bE_{\bF}$.
{Additionally, we notice qualitative similarities between the minima of the loss landscape both with and without layer normalization (Figure \ref{fig:loss_landscape}, Appendix \ref{app:loss_landscapes}).}

 \paragraph{Weight simplifications.} Thus far, the learnable weights have remained unconstrained. We now restrict their parameter space to ensure analytical tractability. To this end, we consider two simplifications. For both, we freeze $\bW_{t-1}^K = \bW_{t-1}^Q = \bI_d$, and use an initialization strategy similar to \citep{train_transformers_bartlett} for $\bW^{PV}$ and $\bW^{KQ}$; see Appendix \ref{app:readout}.
\paragraph{One-parameter model.} In the first case, we tie weights across layers and set $\bW_{t-1}^S = - \bW_{t-1}^V = \alpha \bI_d$
for all $0 \leq t \leq T-1$ where $\alpha \in \rr$ is a learnable scalar parameter. This gives a one-parameter model $\hat y_q = f(\bX , \bx_q, \by; \alpha)$.
\paragraph{Two-parameter model.} For the second model, we instead set $\bW_{t-1}^S = \alpha \bI_d \quad {\rm and} \quad \bW_{t-1}^V = \beta \bI_d$
for all $0 \leq t \leq T-1$ where $(\alpha, \beta) \in \rr^{2}$ are learnable parameters, yielding a two-parameter model $\hat y_q = f(\bX , \bx_q, \by; \alpha, \beta)$.

The first model is recovered from the second by setting $\beta = - \alpha$. Since the one- and two-parameter models are subfamilies of the fully parameterized architecture \eqref{eq:LSA_LCA_model}, they help witness the expressibility of the unconstrained architecture. Importantly, these simplifications do not preclude optimality: we show in the following that the two restricted models already achieve Bayes-optimal prediction in the ICL setting when trained by gradient flow. Indeed, Figure \ref{fig:less-restricted-diagnostics} shows that less-restricted parametrizations exhibit qualitatively similar behavior as predicted in Section \ref{sec:main_results}.

\section{Training setup and main results} \label{sec:main_results}
\textbf{Training setup.} Herein, we consider the one- and two-parameter models $\hat y_q = f(\bX, \bx_q ,y;\theta)$ of Section \ref{sec:CA_model} where $\theta$ either denotes the parameters $\alpha$ or $(\alpha, \beta)$. Starting with a set of $N$ training prompts $\{(\bx_i^{(j)}, y_i^{(j)})_{i=1}^{L_{\rm tr}} \cup (\bx_q^{(j)}, y_q^{(j)})\}_{j=1}^N$, given real-valued, continuous targets, it is natural to consider the empirical squared loss on the target query: $\ell_{N,L_{\mathrm{tr}}}(\theta)
:= \frac{1}{N}\sum_{j=1}^N(y^{(j)}_q - \hat y_q^{(j)})^2$.
As in prior work \citep{train_transformers_bartlett, ahn2023transformerslearnimplementpreconditioned}, taking the number of training prompts $N \to \infty$, we operate on the population loss $\ell_{L_{\mathrm{tr}}}(\theta)
:= \lim_{N\to\infty} \ell_{N,L_{\mathrm{tr}}}(\theta)
= \ee[\left(y_q - \hat y_q\right)^2]$,
where the expectation is taken with respect to the joint distribution detailed in Section \ref{sec:data}. As a final reduction, we consider the asymptotic, training context-length regime and define the loss $ \ell(\theta) := \lim_{L_{\rm tr} \to \infty} \ell_{L_{\mathrm{tr}}}(\theta)$.
This regime is qualitatively similar to prior work \citep{train_transformers_bartlett} whereby one establishes the ICL capacity of the model at large context length, however our presentation diverges by taking the $L_{\mathrm{tr}} \to \infty$ limit before training on the loss function.

Having established the loss of interest, the training procedure for both variants of the CA model will simply be to run gradient flow on the limiting objective $\ell$.  That is, the model parameters evolve according to the ordinary differential equation $\frac{\d}{\d t}\theta_t = -\nabla \ell(\theta_t)$,
where $(\theta_t)_{t \geq 0}$ is the trajectory of the parameters over the training horizon. Gradient flow can be seen as the limit of gradient descent in vanishing step size. Analyzing optimization in the gradient flow limit has enabled precise characterization of deep neural network training dynamics \citep{saxe2014exactsolutionsnonlineardynamics, chizat2018globalconvergencegradientdescent}; this approach has been leveraged in previous studies on ICL as well \citep{train_transformers_bartlett, zhang2025trainingdynamicsincontextlearning}.

\textbf{Main results.} We assume the following on the stacked modality-dependent vector $\bfm$ from \eqref{eq:m_def_and_gaussianity}.

\begin{assumption}[Support of $\|\bfm\|$] \label{assump:support_m}
      Assume $\bfm$ has a continuous distribution and
    $
      \underline{m} := {\rm ess}\inf \|\bfm\|^2 <  {\rm ess} \sup \|\bfm\|^2 =: \overline{m} < \infty
$
    so that $\|\bfm\|$ is a.s.~bounded and non-degenerate.
\end{assumption}

The boundedness assumption on $\bfm$ is mild, e.g., it holds as soon as the covariates are normalized, for instance $\bx_i \in [0,1]^d$. The continuity assumption enters only through $\P(\bfm = \bzero_d) = 0$, which ensures that $Z := 1 + \|\bfm\|^2 > 1$ almost surely, and could be replaced by that condition, so in particular the results below require no density for $\|\bfm\|$, whose law may have atoms.
We now present our main results. We show that a deep model of CA layers learns the Bayes-optimal prediction rule in-context for this multi-modal problem.

\begin{theorem}[One-parameter model optimality]  \label{thm:one_param_optimal}
Assume the data generating process in Section~\ref{sec:data} and Assumption \ref{assump:support_m}. Run gradient flow on the loss $\ell(\alpha)$ for the one-parameter model. Then, we have:
\begin{enumerate}
    \item Gradient flow converges. Formally, letting $(\alpha_{t,T})_{t \geq 0}$ denote the gradient flow trajectory for a CA stack of depth $T$, the limit
     $\alpha^\ast_T := \lim_{t \to \infty} \alpha_{t,T}$ exists.
    \item As $T \to \infty$, $\alpha^\ast_T \to \frac{2}{2 + \underline{m} + \overline{m} } =: \alpha^\ast$ and the resulting predictor is Bayes optimal as the depth diverges. Formally, $\lim_{T \to \infty} \lim_{L_{\rm te} \to \infty} \hat y_q^\mathrm{(te)} = y_q^{\mathrm{Bayes},\mathrm{(te)}}$
    a.s.
\end{enumerate}
\end{theorem}

We prove this result in Appendix \ref{app:one-param}. Since prior linear-attention ICL analyses study fixed-depth, often single-layer, architectures, our infinite depth analysis necessitates a different strategy: we use novel convex-analytic, epi-convergence-style arguments to show that the one-parameter loss converges to a limiting objective whose minimizer can be analyzed. We now turn to the two-parameter setting.

\begin{theorem}[Two-parameter model optimality] \label{thm:two_param_optimal}
Assume the data generating process in Section~\ref{sec:data} and Assumption \ref{assump:support_m}. Run gradient flow on the loss $\ell(\alpha, \beta)$ for the two-parameter model with initialization $\beta_0 \in (-2/(\overline{m} + 1), 0)$ and $\alpha_0 = \alpha^\ast(\beta_0)$
where $\beta \mapsto \alpha^\ast(\beta)$ is the profiled minimizer \eqref{eq:unique_beta_min} of Appendix~\ref{app:two_param}. Then, we have:
\begin{enumerate}
    \item For large depth $T$, gradient flow converges. Formally, letting $(\alpha_{t,T}, \beta_{t,T})_{t \geq 0}$ denote the gradient flow trajectory for a CA stack, the limit
     $(\alpha^\ast_T, \beta^\ast_T) := \lim_{t \to \infty} (\alpha_{t,T}, \beta_{t,T})$ exists.
    \item With $\alpha^\ast$ as defined in Theorem \ref{thm:one_param_optimal}, as $T \to \infty$, $\ \alpha^\ast_T \to \alpha^\ast$ and $ \beta^\ast_T \to - \alpha^\ast$, and the resulting predictor is Bayes optimal. Formally, $   \lim_{T \to \infty} \lim_{L_{\rm te} \to \infty} \hat y_q^\mathrm{(te)} = y_q^{\mathrm{Bayes},\mathrm{(te)}}$ a.s.
\end{enumerate}
\end{theorem}

The proof can be found in Appendix \ref{app:two_param}. This theorem is  considerably more involved and  requires new proof techniques: since the loss is non-convex and non-coercive, the fixed-depth, gradient-flow limit $(\alpha^\ast_T, \beta^\ast_T)$ can only be identified as a stationary point rather than a global minimizer. We handle this by reducing the problem to the evolution of $\beta^\ast_T$ and showing, via a uniform derivative control, that, as $T\to \infty$, admissible stationary points collapse to the same limit. Theorem \ref{thm:two_param_optimal} shows that the limiting behavior of the two-parameter model mimics that of the one-parameter setting, whereby in both cases
gradient flow achieves the Bayes-optimal predictor.
\rev{The one- and two-parameter models already contain the Bayes-optimal solution, and Theorems~\ref{thm:one_param_optimal} and~\ref{thm:two_param_optimal} show that gradient flow finds it despite the random covariance spectrum. In a fully parameterized model, we expect the weights to learn preconditioners in the spirit of LISTA \citep{Gregor2010LearningFast} rather than scaled identities, and Figure~\ref{fig:less-restricted-diagnostics} shows that the predicted whitening persists when $(\bW_t^S,\bW_t^V)$ are untied, diagonal or full.}

A careful reading of the proof of Theorem \ref{thm:one_param_optimal} informs the choice of initialization in Theorem \ref{thm:two_param_optimal}. Solving the recurrence \eqref{eq:recurrence} in the one-parameter model reveals
$\lim_{L_{\rm te}\to \infty} \hat y_q^{\mathrm{te}} = y_q^{\mathrm{Bayes},\mathrm{(te)}} +e(\alpha)$ with error $e(\alpha) =O(\|\bI- \alpha \bLam\|^T)$. A short calculation shows that $e(\alpha) \to 0$ as $T \to \infty$ when $\alpha \in (0, 2/(\overline{m} + 1))$. Given the correspondence between the two models, this lends to $\beta_0 \in (-2/(\overline{m} + 1), 0)$ as a suitable initialization in the two-parameter setup. Moreover, the two-parameter loss $\ell(\alpha, \beta)$ turns out to be quadratic in $\alpha$ with explicit minimizer $\alpha^\ast(\beta)$, and hence starting at $\alpha_0 = \alpha^\ast(\beta_0)$ becomes an obvious choice. This intuition is supported in a visualization of the loss landscape provided in Figure \ref{fig:loss_landscape}, Appendix \ref{app:loss_landscapes}.

\rev{\textbf{The limiting parameter.} Along the spike direction of $\bLam=\bI+\bfm\bfm^\top$, which contains the task vector, the error of the one-parameter recurrence contracts by a factor of at most
\[
\phi(\alpha):=\max\bigl\{|1-\alpha(1+\underline{m})|,\ |1-\alpha(1+\overline{m})|\bigr\}
\]
per layer, and the limit $\alpha^\ast=2/(2+\underline{m}+\overline{m})$ of Theorems~\ref{thm:one_param_optimal} and~\ref{thm:two_param_optimal} is exactly the minimizer of $\phi$. Training thus selects the parametrization that minimizes the worst-case error rate over the support of the covariance spectrum.}

\textbf{The role of depth.}
Recall the embedding $\bF=\mathsf{LCA}_T(\bX;\gamma)$ \eqref{eq:CA_model} produced by the $T$-layer CA stack. For the one- and two-parameter models at fixed depth and parameters, the final prediction has the form $\hat y_q^{(\mathrm{te})}=\bw^{(\mathrm{te})\top}(\frac{1}{L_{\rm te}}\bX^{({\rm te})}\bF^\top)\bx_q^{(\mathrm{te})}+o(1)$ almost surely as $L_{\mathrm{te}}\to\infty$,
as shown in Appendix \ref{app:ca-algebra} (Lemma \ref{lem:readout}). Thus, if the CA stack can enforce
\begin{equation}
\label{eq:whiten}
\frac{1}{L_{\rm te}}\bX^{({\rm te})}\bF^\top \longrightarrow \bI_d
\quad\text{as }T\to\infty,
\end{equation}
then $\hat y_q^{\mathrm{(te)}}\to y_q^{\mathrm{Bayes},\mathrm{(te)}}$ as both $T\to\infty$ and $L_{\mathrm{te}}\to\infty$. An instance where \eqref{eq:whiten} can occur is if $\bF \approx \bLam^{-1} \bX$
---that is, when $\bF$ whitens the covariates. Crucially, since the embedding $\bF$ is prompt-dependent, \eqref{eq:whiten} is achievable despite $\bLam$ being random: the model is not constrained to learn a single global inverse covariance, rather it can compute $\bLam^{-1}$ implicitly from the prompt itself. Here lies the significance of the CA embedding as a denoising mechanism for the covariates, where, in contrast, a single-layer attention model can only learn an average covariance, and thus fails to be  Bayes-optimal.

\textbf{Finite context and depth.} The optimality results  of Theorems~\ref{thm:one_param_optimal} and~\ref{thm:two_param_optimal} are asymptotic in both the test-context length and the depth. To supplement our main results, we offer non-asymptotic error bounds in Appendix~\ref{app:finite-context} for the one-parameter model at a fixed parameter. The excess risk splits into a bias that contracts geometrically in the depth and a noise term of order $d/L_{\rm te}$ from the finite context. The geometric error decay rate is also evidenced numerically in Figure  \ref{fig:numerics_combined} whereby the model exhibits high accuracy at moderate depth.

\section{Numerical experiments}\label{sec:numerics}
We verify the qualitative predictions of Theorems~\ref{thm:failure_LSA}, \ref{thm:one_param_optimal}, and \ref{thm:two_param_optimal} on synthetic data (Appendix~\ref{app:more-less-restricted-experiments}), test whether they persist for less-restricted parametrizations, and ablate our architectural components in the presence of standard transformer mechanisms on the multi-modal avMNIST dataset.

\textbf{Validation of theoretical predictions.}
In Figure  \ref{fig:numerics_combined} (\textbf{Left}), we observe the failure of the single-layer LSA model proved in Theorem \ref{thm:failure_LSA}. In contrast, LCA-based models achieve error rates which are several orders of magnitude smaller as ${L_{\rm te}}$ grows.
Figure \ref{fig:numerics_combined} (\textbf{Right}) shows the effect of depth in the $\mathsf{LCA}_T$ embedding, \eqref{eq:LSA_LCA_model},  in aiding model performance on the multi-modal ICL task. Despite the asymptotic frameworks in Theorems \ref{thm:one_param_optimal} and \ref{thm:two_param_optimal}, we observe exceptional performance at moderate depth as anticipated by the geometric-rate error decay $e(\alpha)$ discussed at the end of Section~\ref{sec:main_results}.

\begin{figure}[ht]
\centering
\includegraphics[width=0.92\textwidth]{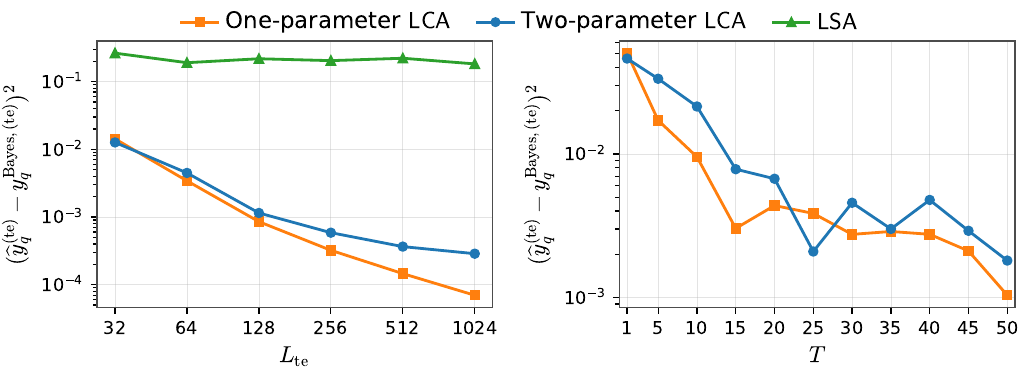}
\caption{
Prediction error for the one- and two-parameter LCA-based models and the LSA model.
Left: ICL error at fixed depth for varying $L_{\mathrm{te}}$.
Right: ICL error at fixed $L_{\mathrm{te}}$ for varying depth $T$. }
\label{fig:numerics_combined}
\end{figure}
\paragraph{Less-restricted parametrizations.}
Here we test whether the predictions of Section~\ref{sec:main_results} persist when the weights $(\bW_t^S, \bW_t^V)$ are untied across layers and taken to be scalar multiples of the identity, diagonal, or full matrices (`untied scalar', `diagonal', and `full-matrix' in Figure~\ref{fig:less-restricted-diagnostics}, see Appendix~\ref{app:more-less-restricted-experiments}).
Figure~\ref{fig:less-restricted-diagnostics} (\textbf{Left}) displays the normalized traces $\widehat\alpha_t=d^{-1}\operatorname{tr}(\bW_t^S)$ and $\widehat\beta_t=d^{-1}\operatorname{tr}(\bW_t^V)$ across layers, showing that they cluster around the one- and two-parameter $T\to\infty$ gradient-flow limit $(\alpha^\ast, -\alpha^\ast)$ (Theorems \ref{thm:one_param_optimal}--\ref{thm:two_param_optimal}).
The normalized traces $(\widehat\alpha_t, \widehat\beta_t)$ on their own do not confirm a full collapse of the form $(\bW_t^S, \bW_t^V) \approx (\alpha^\ast \bI, -\alpha^\ast \bI)$, so Figure~\ref{fig:less-restricted-diagnostics} (\textbf{Right}) displays the whitening diagnostic $\|\bX\bF_t^\top/L-\bI\|_F/\sqrt d$ across layers. In line with \eqref{eq:whiten}, even these less-restricted models exhibit the same prompt-dependent covariance-whitening as proven earlier for our simpler models.

\begin{figure}[ht]
\centering
\includegraphics[width=0.92\textwidth]{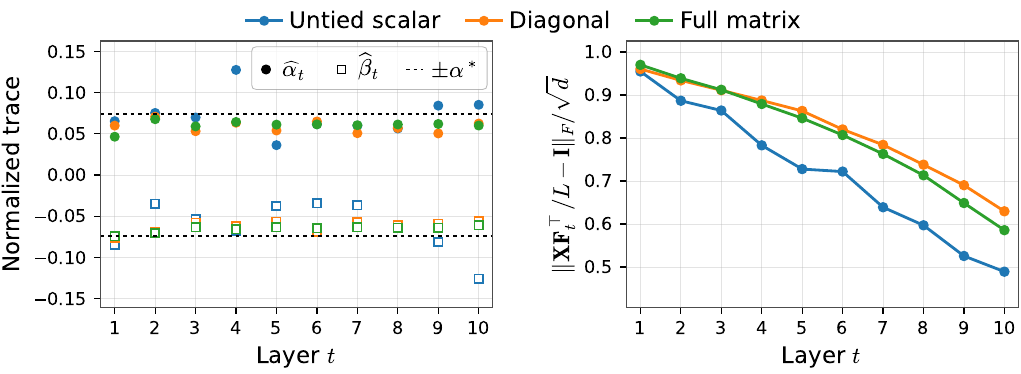}
\caption{Diagnostics for less-restricted LCA parametrizations (setup same as in Figure~\ref{fig:numerics_combined}; see Appendix \ref{app:more-less-restricted-experiments}).
\textbf{Left}: layerwise scalar projections;
\textbf{Right}: whitening error
across layers.}
\label{fig:less-restricted-diagnostics}
\end{figure}

\paragraph{Performance of theory-motivated components in standard multi-modal transformers.}
We next ask whether the novel theory-motivated components of our CA architecture (Section~\ref{sec:CA_model}) remain useful \rev{and learnable in a supervised multimodal learning task, rather than in the prompt-based setting of our theory,} after replacing the linearized model with standard transformer blocks---softmax multi-head attention, layer normalization, MLP blocks. We call this architecture---with all the standard transformer blocks on top of our innovations in Section \ref{sec:CA_model}---the ``full paper-inspired model'' (see Appendix~\ref{app:avmnist_details} for details). Recall from \eqref{eq:recurrence}--\eqref{eq:def_attn} that each layer updates an evolving state $\bF_t$ using two nonstandard ingredients: a CA term $\bA_t$ between $\bF_t$ and the data $\bX$, and an additional skip term $\bS_t$ which re-injects $\bX$ at every layer.   We evaluate on corrupted\footnote{We add substantial noise to covariates so the task is not saturated; this explains accuracies in the 60\% range in Table \ref{tab:avmnist_main}.}
avMNIST~\citep{avmnist}, a paired audio-visual digit classification benchmark. We compare
    the full paper-inspired model against a standard bidirectional CA transformer \citep{lu2019vilbert} and three ablations: (i) A no-injection ablation removes the $\bS_t$ term from \eqref{eq:recurrence}, while retaining CA between $\bF_t$ and $\bX$. (ii) A no-CA ablation retains the re-injection term but replaces the $\bF_t$--$\bX$ CA update by self-attention of $\bF_t$ with itself. (iii) The final ablation removes both mechanisms, yielding a conventional self-attention model. Architecture and training details are deferred to Appendix~\ref{app:avmnist_details}.

\begin{table}[ht]
\centering
\small
\setlength{\tabcolsep}{6pt}
\begin{tabular}{lcccc}
\toprule
Model & Params & Best test acc. (\%) & Best epoch \\
\midrule
\textbf{Full paper-inspired model (Inj. \& cross-attn)}            & $4.36 \times 10^4$ & \textbf{63.85} & 64\\
No injection              & $3.82 \times 10^4$ & 62.72 & 65\\
No cross-attn             & $4.36 \times 10^4$ & 62.10 & 79\\
No inj. \&  no cross-attn         & $3.82 \times 10^4$ & 60.79 & 90 \\
Standard bi-CA transformer& $6.07 \times 10^4$ & 62.72	 & 62 \\
\bottomrule
\end{tabular}
\vspace{1em}
\caption{
Classification results on avMNIST for the full paper-inspired model, three ablations thereof, and a standard bi-directional cross-attention transformer (architecture and training details in Appendix~\ref{app:avmnist_details}). Model parameter counts and best test accuracy with corresponding epoch are shown. }
\label{tab:avmnist_main}
\end{table}

Table~\ref{tab:avmnist_main} shows that our full paper-inspired model achieves the best test accuracy among all compared architectures. In particular, it outperforms each ablation. Additionally, it improves over the standard bidirectional cross-attention transformer, a standard fusion architecture for multimodal transformers~\citep{lu2019vilbert}, despite the latter having more trainable parameters. These results suggest that both mechanisms identified by our theory---cross-attention between the evolving state $\bF_t$ and the original data $\bX$, and recurrent re-injection of the covariates through $\bS_t$---remain useful in more practical deep-learning implementations.
\rev{While preliminary, and not intended as a state-of-the-art comparison, the results are promising and motivate further investigation of our architecture in larger-scale multimodal settings.}

\paragraph{Performance under context-dependent covariates beyond multi-modal data.} Our architecture and theoretical results need not be restricted to multi-modal data. More generally, we cover covariate distributions which vary across contexts and are coupled to the prediction rule. The velocity task of the cosmological benchmark CosmoBench \citep{huang2025cosmobench} has exactly this structure. Galaxy velocities are hard to measure, while galaxy positions are recorded routinely by large surveys. Predicting velocities from positions is a standard problem in cosmology, used to measure the growth of cosmic structure, that is, how fast galaxies and clusters of galaxies form under gravity, and the expansion rate of the Universe \citep[see][and references therein]{huang2025cosmobench}.

Each point cloud in CosmoBench contains $N$ dark-matter halos (or galaxies) of one simulated universe, which trace its matter distribution. Row $i$ of a cloud $(\bX,\bV)$ holds the position $\bx_i\in[0,B)^3$ of halo $i$ in a periodic box of side $B$ and its peculiar velocity $\bv_i\in\rr^3$, and the task is to predict $\bV$ from $\bX$. We use CAMELS, one of the benchmark's three simulation suites, and report the $R^2$ pooled over all halos of the test set.

Each universe $c$ has its own cosmological parameters $\theta_c$ (the matter density $\Omega_m$, the clustering amplitude $\sigma_8$ and four parameters controlling supernova and black-hole feedback), drawn independently from a prior distribution. The simulation generates the positions and the velocities jointly given $\theta_c$, so both the distribution of $\bX_c$ and the conditional law of $\bV_c$ given $\bX_c$ depend on $\theta_c$, which is never observed. The Bayes predictor therefore averages over the posterior of $\theta$ given the cloud,
\[
\ee[\bv_i\mid\bX] \;=\; \int \ee[\bv_i\mid\bX,\theta]\;p(\theta\mid\bX)\,d\theta .
\]
This is the hierarchical structure of Section~\ref{sec:data} with a single modality, in which $\theta_c$ plays the role of $(\bfm,\rc)$, a context-level latent variable that is coupled to both the covariate distribution and the regression function and must be inferred from the context, here the cloud, itself.

The full paper-inspired model of Table~\ref{tab:avmnist_main} carries over with the halos as tokens and a per-halo readout. Appendix~\ref{app:cosmobench} describes the data splits, the inputs, the architecture and the training and evaluation protocol in detail. Table~\ref{tab:cosmobench} compares our model with the benchmark's baselines where we see it improving substantially ($20$--$25$\% relative improvement in $R^2$), evidencing its practicality on more general hierarchical distributions.

\rev{\begin{table}[ht]
\centering
\small
\begin{tabular}{@{}lcc@{}}
\toprule
Model & Params & $R^2$ \\
\midrule
\textbf{Full paper-inspired model} & $5.94 \times 10^4$ & $\textbf{0.2987}$ \\
Graph network \citep{huang2025cosmobench} & $1.27 \times 10^5$ & $0.2392$ \\
Least squares \citep{huang2025cosmobench} & 30 & $0.2494$ \\
\bottomrule
\end{tabular}
\caption{Test $R^2$ on the CAMELS simulations of CosmoBench's velocity task (mean over three seeds). The graph network is the benchmark's released network, re-evaluated in our pooled convention (the benchmark reports $0.2527$ per simulation), and least squares is the benchmark's published value. Details in Appendix~\ref{app:cosmobench}.}
\label{tab:cosmobench}
\end{table}}

\section{Discussion} \label{sec:discussion}
We studied a multi-modal  task in which each prompt is generated by a latent factor model, inducing prompt-dependent covariate statistics which place the setting outside of standard unimodal ICL analyses. To demonstrate this added complexity, we first established the failure of single-layer attention, both the LSA models standard in the literature and their softmax counterparts. Motivated by this limitation, we introduced a novel transformer-like architecture, incorporating both depth and CA, proving not only its expressibility for the ICL task, but also demonstrating its capacity to learn the Bayes-optimal predictor via gradient flow.
\rev{To our knowledge, our work is the first to tackle a multi-modal ICL task in a mathematically tractable setting and to prove that training by gradient flow overcomes a covariance spectrum that varies across prompts.}  While preliminary, our empirical findings (Table \ref{tab:avmnist_main} and \rev{\ref{tab:cosmobench}}) demonstrate the promise of our new architectural innovations (Section \ref{sec:CA_model}) and motivate, as an important future direction, further empirical studies benchmarking these approaches on large-scale multi-modal problems.
Another important future direction would be to  extend our population loss level analysis to sample-level results, as done earlier for unimodal problems \citep{Lu_2025}. We expect this direction to pose intriguing mathematical challenges---requiring non-trivial developments in random matrix theory.
\rev{Finally, whether other transformer-like architectures can provably learn the prompt-specific whitening that our CA layers implement, and whether the guarantees extend to fully parameterized models, are natural questions for future work.}

\subsection*{AI use statement}

The research questions, the multi-modal ICL framework and the novel
architecture were devised entirely by the authors. Prior versions of
Theorem~\ref{thm:failure_LSA}, focusing on linear attention, and of
Theorem~\ref{thm:two_param_optimal}, with an added technical assumption on the
limit points of $(\beta_T)_{T \geq 1}$, were proved entirely by the authors. We
used generative AI assistants to extend the proof of
Theorem~\ref{thm:failure_LSA} from linear to softmax attention and to remove the
aforementioned technical assumption from Theorem~\ref{thm:two_param_optimal}.
The softmax extension is a reduction to the linear case, since Gaussian
exponential tilting (Lemma~\ref{lem:tilt}, Proposition~\ref{prop:softmax-equiv})
gives a softmax head the same infinite-context predictor as a linear-attention
head. The proof is then completed by the argument of the authors' original
linear-attention proof, isolated as Lemma~\ref{lem:obstruction}.
Additionally, generative AI was used to help derive the finite-context bounds of
Appendix~\ref{app:finite-context}. Generative AI also helped implement
experiment code and produce figures, and was used in polishing the prose of the paper such as in locating grammatical errors of our draft. All AI-assisted content was checked
thoroughly by the authors, and we take full responsibility for the content of
this paper.

\subsection*{Acknowledgments}
The authors thank Daniel Hsu for helpful feedback on an earlier version of this manuscript. N.B.~acknowledges support from the \textit{Fonds de recherche du Qu\'ebec---Nature et technologies} (FRQNT). S.S. gratefully acknowledges support from the National Science Foundation (DMS CAREER 2239234), Office of Naval Research (N00014-23-1-2489) and Air Force Office of Scientific Research (FA9950-23-1-0429). P.S.~gratefully acknowledges support from the National Science Foundation (DMS CAREER 2440824) and the Office of Naval Research (N00014-26-1-2144).

\bibliography{refs}
\bibliographystyle{plainnat}

\newpage

\appendix
\section{Visualization of the loss landscape}\label{app:loss_landscapes}

The rationale behind the initializations given in Theorem \ref{thm:two_param_optimal} and discussed in Section \ref{sec:main_results} is supported in Figure \ref{fig:loss_landscape} which illustrates a steep valley along the curve $\{(\alpha^\ast(\beta), \beta): \beta \in (-2/(\overline{m} + 1), 0) \}$ as $\alpha^\ast(\beta) \approx -\beta$ ($\alpha^\ast(\beta) \to - \beta$ as $T \to \infty$) in this window. To contrast our setting with a more sophisticated model, we also plot the loss landscape of the two-parameter model where pre-Layer Normalization \citep{xiong2020layernormalizationtransformerarchitecture} is added between the LCA layers. Interestingly, this model reveals a similarly located and sized ravine in the loss landscape, hinting at the insights from our ablated setting to more complex models.

\begin{figure}[ht]
  \begin{center}
    \centerline{\includegraphics[width=1.0\textwidth]{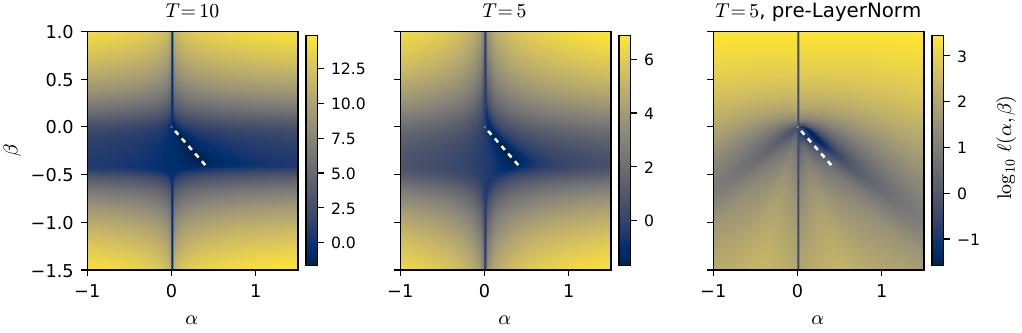}}
    \caption{Log-scale heat maps of the two-parameter excess loss $\ell(\alpha, \beta)$ with $d_1=d_2=16$ and $\|\bfm\| \sim {\rm Unif}(0,2)$. (\textbf{Left}) depth $T=10$. (\textbf{Middle}) depth $T=5$. (\textbf{Right}) depth $T=5$ with pre-LayerNorm added between the CA layers. The dashed line marks $\alpha=-\beta$ over $\beta \in (-2/(\overline{m}+1),0)$. Each panel carries its own colour scale, as the three surfaces span different magnitudes.}
    \label{fig:loss_landscape}
  \end{center}
    \vskip -0.2in
\end{figure}

\section{Ablation experiments on synthetic data} \label{app:ablations}

In this appendix, we examine the significance of the $S_t$ skip-connection and the CA between the layer outputs $(F_t)_{t \geq 0}$ and the original raw data $X$---displayed in Figure \ref{fig:model}---on the data generating distribution compatible with Section \ref{sec:data} and Assumption \ref{assump:support_m}. To this end, we compare, empirically, the trained performance of the one- and two-parameter LCA models defined in Section \ref{sec:CA_model} against ablated models. We introduce three new architectures in the one- and two-parameter settings. First, consider the one-parameter LCA model without $S_t$ (equivalently with $\alpha \equiv 0)$---that is, we modify \eqref{eq:recurrence} to be
\begin{equation}
    \bF_t = \bF_{t-1} + \bA_{t-1}
\end{equation}
and operate with one learnable parameter as in Section \ref{sec:CA_model}. Second, building on the preceding model, we now set $\bA_{t} = \mathsf{A}(\bQ_{t}, \bK_{t}, \bV_t)$ with
\begin{equation} \label{eq:LSA_replacing_LCA}
\bV_{t} = \alpha \bF_{t}, \quad \bK_{t} = \bQ_{t}= \bF_t,
\end{equation}
so that $\bA_t$ is no longer a LCA implementation, but rather is LSA applied to $\bF_t$ with a single learnable parameter. We refer to this architecture as a ``one-parameter deep LSA without $S_t$'' to delineate it from the single-layer LSA model of Section \ref{sec:ICL_LSA}. Lastly, as a final trainable comparison, we consider the LSA modification given in \eqref{eq:LSA_replacing_LCA} and re-introduce $S_t$ (again producing two learnable parameters) in the architecture, and refer to this ablation as a ``two-parameter deep LSA.'' As a naive baseline, we also consider the in-context sample mean:
\[
\bar y_{L_{\rm te}} = \frac{1}{L_{\rm te}} \sum_{i=1}^{L_{\rm te}} y_i.
\]

Both no-$S_t$ ablations (the one-parameter LCA model and the one-parameter deep LSA model) are initialized with $\bF_0=\bX$.

\begin{figure}[ht]
  \begin{center}
     \centerline{\includegraphics[width=1.0\textwidth]{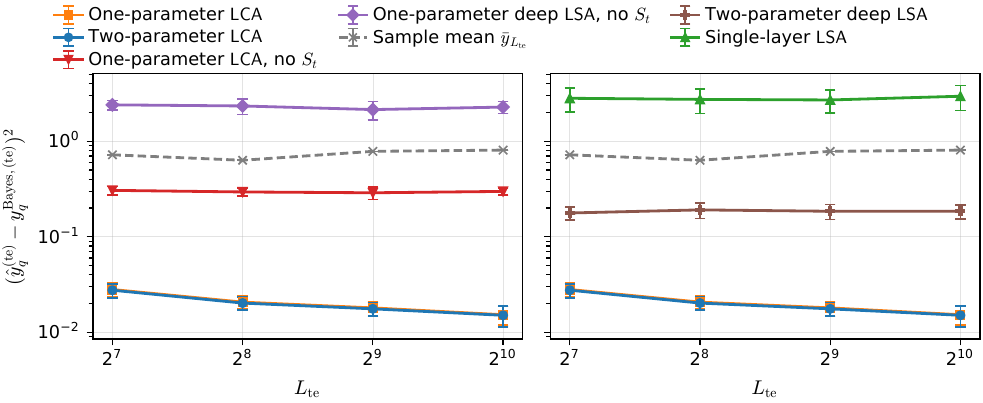}}
        \caption{In-context performance at various $L_{\rm te}$ ($T=10$, $d_1=d_2=16$). (\textbf{Left}) one- and two-parameter LCA-based models, the two ablations without $S_t$, and the sample mean $\bar y_{L_{\rm te}}$. (\textbf{Right}) one- and two-parameter LCA-based models, the two-parameter deep LSA model, the LSA model from \eqref{eq:LSA_def}, and the sample mean. Models are optimized on $\ell_{N,L_{\mathrm{tr}}}$ ($L_{\rm tr} = 100, N=2000$) using gradient descent. Performance is averaged over $1000$ test-prompts where error bars represent standard deviation over $10$ separate training runs. The sweep starts at $L_{\rm te}=128$, since for $L_{\rm te}$ comparable to $d$ the sample covariance fluctuates enough that the depth-$T$ recurrence leaves its stable region on a small fraction of prompts.}
    \label{fig:no_St_comparisons}
  \end{center}
    \vskip -0.2in
\end{figure}

In Figure \ref{fig:no_St_comparisons} (\textbf{Left}), we see that the removal of the $S_t$ skip-connection in the embedding architecture compromises model performance. Notably, the ablated models cease to in-context learn, and their error plateaus as $L_{\rm te}$ grows, whereas the full models continue to improve. The one-parameter deep LSA ablation is in addition worse than the naive in-context sample mean at every $L_{\rm te}$. Moreover, in Figure \ref{fig:no_St_comparisons} (\textbf{Right}), we see that swapping the LCA mechanism for LSA on $F_t$ while keeping the $S_t$ skip-connection deteriorates model performance, albeit improving upon the sample mean baseline and the single-layer LSA model. This further emphasizes the utility of the cross-attention and particularly of our novel $S_t$ component.

\section{Notation, basic identities, and the Bayes-predictor}\label{app:notation}

\subsection{Notation and conventions}\label{app:notation_table}

We work with a single prompt (task) unless otherwise stated. For two sequences $(a_n)_{n \ge 1 }$ and $(b_n)_{n \ge 1}$, we write
\[
a_n = O(b_n)
\]
to mean $a_n \leq C \cdot b_n$ for some constant $C > 0$ for $n$ sufficiently large. We also use $o(1)$ to denote a sequence of terms which vanish as the ambient index grows. For example, writing $a_n = o(1)$ implies $a_n \to 0$ as $n \to \infty$.

\begin{table}[h]
\centering
\begin{tabular}{ll}
\hline
Symbol & Meaning \\
\hline
$d=d_1+d_2$ & total covariate dimension \\
$L$ & context length (we use $L_{\mathrm{tr}}$ or $L_{\mathrm{te}}$ when needed) \\
$\bx_i\in\mathbb{R}^d$ & stacked covariate for context token $i\in[L]$ \\
$\bx_q\in\mathbb{R}^d$ & query covariate \\
$y_i\in\mathbb{R}$ & context response for token $i\in[L]$ \\
$y_q\in\mathbb{R}$ & query response \\
$\bfm\in\mathbb{R}^d$ & stacked modality vector (prompt-specific), $\bfm=[\bv;\br]$ \\
$\rc\in\mathbb{R}$ & prompt-specific scalar controlling the response \\
$u_i\sim\mathcal{N}(0,1)$ & latent scalar shared across modalities and response \\
$\bfmu_i\sim\mathcal{N}(0,\bI_d)$ & stacked noise in covariates \\
$\bLam=\bI+\bfm\bfm^\top$ & covariate covariance conditional on $\bfm$ \\
$\bSig$ & joint covariance of $(x,y)$ conditioned on $(m,\rc)$ \\
$Z := 1+\|\bfm\|^2$ & top eigenvalue of $\bLam$ (rank-one spike) \\
$\bw\in\mathbb{R}^d$ & Bayes-optimal regression coefficient: $\bw=\frac{\rc}{1+\|\bfm\|^2}\bfm$ \\
$\underline m,\overline m$ & $\underline m=\mathrm{ess\,inf}\,\|\bfm\|^2$, $\overline m=\mathrm{ess\,sup}\,\|\bfm\|^2$ \\
$\underline{Z},\overline{Z}$ & $\underline{Z}=1+\underline m=\mathrm{ess\,inf}\,Z$, $\overline{Z}=1+\overline m=\mathrm{ess\,sup}\,Z$ \\
\hline
\end{tabular}
\end{table}

\subsection{Joint Gaussian form and Bayes-predictor}\label{app:bayes}

Recall the data generating process (Section~\ref{sec:data}): for $i\in[L]\cup\{q\}$,
\[
\bx_i = u_i \bfm + \bfmu_i,\qquad y_i=\rc u_i,
\]
with $u_i\sim\mathcal{N}(0,1)$ and $\bfmu_i = [\bxi_i; \bta_i]\sim\mathcal{N}(0,I_d)$ independent across tokens $i$ and independent of $(\bfm,\rc)$.

\begin{lemma}[Joint covariance and Bayes coefficient]\label{lem:bayes}
Conditional on $(\bfm,\rc)$, the pair $(\bx_i,y_i)\in\mathbb{R}^{d+1}$ is jointly Gaussian with mean $0$ and covariance
\[
\bSig
=
\begin{bmatrix}
\bLam & \rc \bfm\\
\rc \bfm^\top & \rc^2
\end{bmatrix},
\qquad
\bLam = \bI+\bfm\bfm^\top.
\]
Moreover,
\[
\ee[y_i\mid \bfm,\rc,\bx_i] = \bw^\top \bx_i,\qquad
\bw =\frac{\rc}{1+\|\bfm\|^2}\bfm,
\]
and we may write
\[
y_i = \ip{\bw}{\bx_i} + \varepsilon_i,
\quad \varepsilon_i \,\indep  \bx_i \mid(\bfm,\rc),
\quad
\varepsilon_i \sim \cN \left( 0,\frac{\rc^2}{1+\|\bfm\|^2} \right).
\]
\end{lemma}

\begin{proof}
We compute conditional covariances given $(m,\rc)$:
\[
\bSig_{x,x} := \ee[\bx_i \bx_i^\top\mid \bfm,\rc] = \ee[(u_i \bfm + \bfmu_i)(u_i \bfm + \bfmu_i)^\top] = \bfm\bfm^\top + \bI = \bLam,
\]
\[
\bSig_{x,y} := \ee[\bx_i y_i\mid \bfm,\rc] = \ee[(u_i \bfm + \bfmu_i)\cdot \rc u_i] = \rc\, \ee[u_i^2]\,\bfm = \rc \bfm,
\]
\[
\bSig_{y,y} := \ee[y_i^2\mid \bfm,\rc]=\rc^2\ee[u_i^2]=\rc^2,
\]
yielding the stated $\bSig$. For a jointly Gaussian vector, the conditional mean is
\[
\ee[y_i\mid \bx_i,\bfm,\rc] = \bSig_{x,y}^\top \bSig_{x,x} ^{-1} \bx_i
= (\rc \bfm)^\top \bLam^{-1} \bx_i
\]
and so $\bw=\bLam^{-1}(\rc \bfm)$. Using the Sherman--Morrison formula,
\[
\bLam^{-1}=(\bI+\bfm\bfm^\top)^{-1}=\bI-\frac{\bfm\bfm^\top}{1+\|\bfm\|^2}
\]
and hence,
\[
\bw=\left(\bI-\frac{\bfm\bfm^\top}{1+\|\bfm\|^2}\right)\rc \bfm =\rc \bfm - \rc\frac{\|\bfm\|^2}{1+\|\bfm\|^2}\bfm =\frac{\rc}{1+\|\bfm\|^2}\bfm.
\]

Finally, the conditional variance of $y_i$ given $\bx_i$ is
\begin{align*}
\mathrm{Var}(y_i \mid \bx_i,\bfm,\rc) &= \bSig_{y,y} - \bSig_{x,y}^\top \bSig^{-1}_{x,x} \bSig_{x,y}\\
&= \rc^2-(\rc \bfm)^\top\bLam^{-1}(\rc \bfm) \\
&=\rc^2\left(1- \|\bfm\|^2/(1+\|\bfm\|^2)\right) \\
&= \rc^2/(1+\|\bfm\|^2).
\end{align*}
The residual $\varepsilon_i:=y_i-\ip{\bw}{\bx_i}$ is independent of $\bx_i$ conditional on $(\bfm,\rc)$ by Gaussianity.
\end{proof}

\subsection{Law of large numbers for prompt statistics}\label{app:lln}

Let $\bX=[\bx_1,\dots,\bx_L]\in\mathbb{R}^{d\times L}$ and $\by=[y_1,\dots,y_L]^\top\in\mathbb{R}^{L}$ denote the covariate matrix and target vector, and define the context sample covariance
\[
\widehat\bLam := \frac{1}{L}\bX\bX^\top.
\]

\begin{lemma}[Matrix LLN at fixed dimension]\label{lem:matrix-lln}
Fix $(\bfm,\rc)$ and thus $\bLam=\bI+\bfm\bfm^\top$. Then, almost surely as $L\to\infty$,
\[
\widehat\bLam \to \bLam
\]
entry-wise, and hence also in operator norm. Moreover,
\[
\frac{1}{L}\bX \by \to \ee[\bx_i y_i\mid \bfm,\rc]=\rc \bfm,
\qquad
\frac{1}{L}\sum_{i=1}^L y_i^2 \to \ee[y_i^2\mid \bfm,\rc]=\rc^2.
\]
Finally, writing $y_i=\ip{\bw}{\bx_i}+\varepsilon_i$ as in Lemma~\ref{lem:bayes} and letting $\beps=(\varepsilon_1,\dots,\varepsilon_L)^\top$,
\[
\frac{1}{L}\bX \beps \to \bzero\quad\text{almost surely.}
\]
\end{lemma}

\begin{proof}
Conditional on $(\bfm, \rc)$, the pairs $(\bx_i, y_i)$ are i.i.d. Gaussian with finite second moments, so the stated convergences follow from the law of large numbers entrywise.
\end{proof}

\section{Proof of Theorem \ref{thm:failure_LSA}} \label{app:failure_LSA_proof}

Theorem \ref{thm:failure_LSA} concerns two classes of single-layer attention models, namely the linear self-attention model \eqref{eq:LSA_def} and the softmax self-attention block \eqref{eq:SA-def} with finitely many heads. Sections~\ref{app:obstruction} and~\ref{app:lsa-failure} give the argument for linear attention, demonstrating that its infinite-context predictor is affine in the query, with a slope that is a quadratic polynomial in the task scalar $\rc$ whose linear coefficient is a \emph{fixed} matrix applied to $\bfm$, and Lemma \ref{lem:obstruction} shows that no such slope can match the oracle coefficient $\bw = \rc\bfm/(1+\|\bfm\|^2)$. Section~\ref{app:softmax-failure} extends the conclusion to softmax attention by reduction to this linear case, as Proposition \ref{prop:softmax-equiv} demonstrates that a softmax head has the same infinite-context predictor as a linear-attention head.

Throughout this appendix we write the context tokens as
\[
\bz_j := \begin{bmatrix}\bx_j\\ y_j\end{bmatrix} \in \rr^{d+1},
\qquad j \in [L_{\rm te}],
\]
which, conditionally on the task $(\bfm,\rc)$, are i.i.d.\ $\cN_{d+1}(\bzero,\bSig)$ with $\bSig$ as in \eqref{eq:m_def_and_gaussianity} and independent of the query $\bx_q$.

\subsection{The linear case I}\label{app:obstruction}

This section and the next constitute the proof for linear attention. The lemma below isolates its concluding step, which is stated separately because Section~\ref{app:softmax-failure} reuses it.

\begin{lemma}\label{lem:obstruction}
Suppose $\|\bfm\|$ and $\rc$ are atomless and $\rc$ is independent of $\bfm$. Let $\bB \in \rr^{d \times d}$ be a fixed matrix, let $\bw_0, \bw_2 : \rr^d \to \rr^d$ be measurable, and suppose a predictor has the infinite-context limit $\ip{\hat\bw(\bfm,\rc)}{\bx_q} + k(\bfm,\rc)$, where $k$ is measurable and
\begin{equation}\label{eq:poly-slope}
\hat\bw(\bfm,s) = \bw_0(\bfm) + s\,\bB^\top\bfm + s^2\,\bw_2(\bfm),
\qquad s \in \rr .
\end{equation}
Then
\begin{equation}\label{eq:obstruction-conclusion}
\P\Bigl(\ip{\hat\bw(\bfm,\rc)}{\bx_q} + k(\bfm,\rc) = \ip{\bw}{\bx_q}\Bigr) = 0 .
\end{equation}
\end{lemma}

\begin{proof}
Fix $\bfm$ and subtract the oracle coefficient $\bw(\bfm,s) = s\bfm/(1+\|\bfm\|^2)$ of \eqref{eq:regression_view}, giving the vector-valued polynomial
\begin{equation}\label{eq:qpoly}
q_{\bfm}(s) := \hat\bw(\bfm,s) - \frac{s\bfm}{1+\|\bfm\|^2}
= \bw_0(\bfm) + \left(\bB^\top\bfm - \frac{\bfm}{1+\|\bfm\|^2}\right)s + \bw_2(\bfm)s^2 .
\end{equation}
If $q_{\bfm}$ is not identically zero, at least one of its scalar components is a nonzero polynomial of degree at most two, so its zero set is finite. Since $\rc$ is atomless and independent of $\bfm$,
\[
\P\bigl(q_{\bfm}(\rc)=\bzero_d\mid\bfm\bigr)=0
\quad\text{whenever }q_{\bfm}\not\equiv\bzero_d.
\]
If instead $q_{\bfm}\equiv\bzero_d$, all of its coefficients vanish. In particular, the linear coefficient gives
\begin{equation} \label{eq:non_linear_in_m}
\bB^\top\bfm=\frac{\bfm}{1+\|\bfm\|^2}.
\end{equation}
For $\bfm\ne\bzero_d$, equation \eqref{eq:non_linear_in_m} implies that $(1+\|\bfm\|^2)^{-1}$ belongs to the finite spectrum $\sigma(\bB^\top)=\sigma(\bB)$. The atomlessness of $\|\bfm\|$ gives both $\P(\bfm=\bzero_d)=0$ and $\P((1+\|\bfm\|^2)^{-1}\in\sigma(\bB))=0$, since $r\mapsto(1+r^2)^{-1}$ is injective on $[0,\infty)$. Thus $\P(q_{\bfm}\equiv\bzero_d)=0$. Combining the two cases and averaging over $\bfm$,
\begin{equation} \label{eq:coef_mismatch_posprob_appB}
\P \left(\hat \bw(\bfm,\rc) = \bw(\bfm,\rc)\right) = 0.
\end{equation}
Finally, condition on $(\bfm,\rc)$ and write $\boldsymbol\delta:=\hat\bw(\bfm,\rc)-\bw(\bfm,\rc)$. Whenever $\boldsymbol\delta\ne\bzero_d$, the scalar $\boldsymbol\delta^\top\bx_q$ is Gaussian with mean zero and variance $\boldsymbol\delta^\top\bLam\boldsymbol\delta\geq\|\boldsymbol\delta\|^2>0$, because $\bLam\succeq\bI$. It therefore equals the fixed scalar $-k(\bfm,\rc)$ with conditional probability zero. Averaging over the task proves \eqref{eq:obstruction-conclusion}.
\end{proof}

Only two features of the slope are used, namely that it is a polynomial of degree at most two in $\rc$, and that the coefficient of $\rc$ is a fixed matrix acting on $\bfm$. Lemma~\ref{lem:obstruction} therefore applies for any number of heads and for both linear and softmax attention.

\subsection{The linear case II, failure of single-layer LSA}\label{app:lsa-failure}

Recall the single-layer LSA predictor
\begin{equation} \label{eq:lsa_def_appB}
\hat y_q = \mathsf{LSA}(\bE_{\bX};\theta)_{d+1,L+1}, \quad \mathsf{LSA}(\bE_{\bX};\theta) =
\bE_{\bX} + \bW^{PV}\bE_{\bX}\cdot \frac{\bE_{\bX}^\top \bW^{KQ}\bE_{\bX}}{L},
\end{equation}
where
\[
\bE_{\bX}=\begin{bmatrix}\bX & \bx_q\\ \by^\top & 0\end{bmatrix}\in\mathbb{R}^{(d+1)\times(L+1)}.
\]

\begin{proposition}[The linear case]\label{prop:lsa-failure}
Suppose $\|\bfm\|$ and $\rc$ are atomless and $\rc$ is independent of $\bfm$. For every fixed $\theta=\{\bW^{PV},\bW^{KQ}\}$, the predictor \eqref{eq:lsa_def_appB} satisfies $\P(\lim_{L_{\rm te}\to\infty}\hat y_q = \ip{\bw}{\bx_q}) = 0$.
\end{proposition}

\begin{proof}
Let us write the weight matrix $\bW^{KQ}$ in block form
\[
\bW^{KQ} = \begin{bmatrix}
\bA & \bb'\\
\bb^\top & c
\end{bmatrix}, \qquad
\bA \in \rr^{d\times d},\; \bb,\bb'\in \rr^d,\;
c\in\rr,
\]
and denote the last row of $\bW^{PV}$ by
\[
 \begin{bmatrix}
\bu^\top & v
\end{bmatrix} := \be_{d+1}^\top \bW^{PV},
\qquad \bu\in\rr ^d,\; v\in\rr.
\]
Since $(\bE_{\bX})_{d+1,{L_{\rm te}}+1}=0$, the residual term (skip-connection) in $\mathsf{LSA}(\bE_{\bX};\theta)$ does not contribute to the predictor $\hat y_q$ and so
\begin{equation} \label{eq:yhat_reduce1_appB}
\hat y_q =  \be_{d+1}^\top \bW^{PV} \bE_{\bX} \cdot \frac{\bE_{\bX}^\top \bW^{KQ}\bE_{\bX}}{{L_{\rm te}}} \be_{{L_{\rm te}}+1}.
\end{equation}
For $i \in [{L_{\rm te}}]$, the $i$-th entry of the vector $\bE_{\bX}^\top \bW^{KQ}\bE_{\bX} \be_{{L_{\rm te}}+1}$ is given by
\[
[\bE_{\bX}^\top \bW^{KQ}\bE_{\bX}  \be_{{L_{\rm te}}+1}]_i
= \begin{bmatrix}\bx_i\\ y_i\end{bmatrix}^\top
\begin{bmatrix}\bA & \bb'\\ \bb^\top & c\end{bmatrix}
\begin{bmatrix}\bx_q\\ 0\end{bmatrix}
= \left( \bx_i^\top \bA \bx_q + y_i\bb^\top \bx_q \right).
\]
Moreover, the $({L_{\rm te}}+1)$-th entry of $\tfrac{1}{{L_{\rm te}}}\bE_{\bX}^\top \bW^{KQ}\bE_{\bX} \be_{{L_{\rm te}}+1}$ is of order $O(\tfrac{1}{{L_{\rm te}}})$ and thus negligible as ${L_{\rm te}} \to \infty$. Hence, we obtain
\begin{equation} \label{eq:yhat_sum_appB}
\hat y_q = \frac{1}{{L_{\rm te}}}\sum_{i=1}^{L_{\rm te}} \left(\bu^\top \bx_i+v y_i\right) \left(\bx_i^\top \bA  + y_i \bb^\top \right)\bx_q + o(1).
\end{equation}
Conditioned on $(\bfm,\rc)$, the context tokens $\{(\bx_i,y_i)\}_{i=1}^{L_{\rm te}}$ are i.i.d., and $\bx_q$ is independent of $\{(\bx_i,y_i)\}_{i=1}^{L_{\rm te}}$ with the same conditional law as $\bx_i$. Applying Lemma \ref{lem:matrix-lln} to each empirical moment that appears when expanding \eqref{eq:yhat_sum_appB} yields, almost surely (for fixed $(\bfm,\rc,\bx_q)$),
\[
\hat y_q \longrightarrow \ip{\hat\bw}{\bx_q},
\]
as ${L_{\rm te}} \to \infty$ where the limiting regression vector $\hat\bw =\hat \bw(\bfm,\rc)\in\mathbb{R}^d$ is
\begin{equation} \label{eq:what_appB}
\hat \bw(\bfm,\rc)^\top = \bu^\top \bLam \bA + \left((\bu^\top \bfm)\bb^\top+ v \bfm^\top \bA\right) \rc +  v \bb^\top \rc^2.
\end{equation}
Transposing \eqref{eq:what_appB} exhibits the form \eqref{eq:poly-slope} required by Lemma \ref{lem:obstruction}, with the task-independent matrix
\[
\bB := \bu\bb^\top+v\bA,
\qquad
\bw_0(\bfm) = \bA^\top\bLam\bu,
\qquad
\bw_2 = v\bb ,
\]
and with $k \equiv 0$. Lemma \ref{lem:obstruction} gives the claim.
\end{proof}

\subsection{Extension to softmax attention by reduction to the linear case}\label{app:softmax-failure}

We now extend Proposition \ref{prop:lsa-failure} to softmax attention. The extension is a reduction, as Proposition \ref{prop:softmax-equiv} shows that a softmax head has the same infinite-context predictor as a linear-attention head, after which Lemma~\ref{lem:obstruction} applies verbatim. We work in a class slightly more general than the block \eqref{eq:SA-def}. For each head $h \in [H]$, with $H<\infty$ fixed, let
\begin{equation}\label{eq:softmax-scores}
\ba_h(\bx_q) = \bK_h\bx_q + \boldsymbol{k}_h,
\qquad
s_{hj} = \ba_h(\bx_q)^\top\bz_j,
\qquad
f_{hj} = \boldsymbol{c}_h^\top\bz_j + c_{h,0},
\end{equation}
where $\bK_h \in \rr^{(d+1)\times d}$, $\boldsymbol{k}_h, \boldsymbol{c}_h \in \rr^{d+1}$ and $c_{h,0} \in \rr$. Let $J_h$ be either $[L_{\rm te}]$ or $[L_{\rm te}]\cup\{q\}$, according to whether the query token participates in the attention, and if it does, its score and value may be arbitrary finite measurable functions of $\bx_q$. The predictor is
\begin{equation}\label{eq:softmax-class}
\hat y_q^{(L_{\rm te})}
= \ip{\boldsymbol{\ell}_1}{\bx_q} + \ell_0
+ \sum_{h=1}^{H}
\frac{\sum_{j\in J_h} e^{s_{hj}} f_{hj}}
     {\sum_{j\in J_h} e^{s_{hj}}} .
\end{equation}
All parameters, and $H$, are fixed independently of the task and of $L_{\rm te}$.

This class contains the block \eqref{eq:SA-def} with arbitrary linear key, query, value and projection maps of any head dimension, a residual connection and an affine scalar readout. Indeed, appending a fixed placeholder to the query's response slot and writing the effective key--query form of head $h$ as $\bW^{KQ}_h$, the context score $\bz_j^\top\bW^{KQ}_h(\bx_q,p)$ is affine in $\bz_j$ with coefficient of the form \eqref{eq:softmax-scores}, and linear value and projection maps followed by a scalar readout likewise make $f_{hj}$ affine in $\bz_j$. Terms in a score that are common to all keys cancel in the softmax.

\begin{lemma}[Gaussian exponential tilt]\label{lem:tilt}
Let $\bz\sim\cN_k(\bzero,\bSig)$ with $\bSig\succeq0$. For every $\ba,\boldsymbol{c}\in\rr^k$ and $c_0\in\rr$,
\begin{equation}\label{eq:tilt}
\ee\bigl[e^{\ba^\top\bz}\bigr]=e^{\ba^\top\bSig\ba/2},
\qquad
\frac{\ee\bigl[(\boldsymbol{c}^\top\bz+c_0)e^{\ba^\top\bz}\bigr]}
     {\ee\bigl[e^{\ba^\top\bz}\bigr]}
=\boldsymbol{c}^\top\bSig\ba+c_0 .
\end{equation}
Moreover $\ee\bigl[e^{|\ba^\top\bz|}|\boldsymbol{c}^\top\bz+c_0|\bigr]<\infty$.
\end{lemma}

\begin{proof}
Write $\bz=\bSig^{1/2}\boldsymbol{g}$ with $\boldsymbol{g}\sim\cN_k(\bzero,\bI)$ and set $\boldsymbol\beta=\bSig^{1/2}\ba$. Completing the square in the standard Gaussian density $\varphi_k$ gives $e^{\boldsymbol\beta^\top\boldsymbol{g}}\varphi_k(\boldsymbol{g})=e^{\|\boldsymbol\beta\|^2/2}\varphi_k(\boldsymbol{g}-\boldsymbol\beta)$, and integrating yields
\[
\ee\bigl[e^{\boldsymbol\beta^\top\boldsymbol{g}}\bigr]=e^{\|\boldsymbol\beta\|^2/2},
\qquad
\ee\bigl[\boldsymbol{g}\,e^{\boldsymbol\beta^\top\boldsymbol{g}}\bigr]=\boldsymbol\beta\,e^{\|\boldsymbol\beta\|^2/2}.
\]
Multiplying the second identity by $\bSig^{1/2}$ gives $\ee[\bz e^{\ba^\top\bz}]=\bSig\ba\,e^{\ba^\top\bSig\ba/2}$, which proves \eqref{eq:tilt}, and the argument covers singular $\bSig$. For integrability, use $e^{|t|}\leq e^{t}+e^{-t}$ together with the Cauchy--Schwarz inequality, noting that a Gaussian moment generating function is finite at every real argument and that the affine value has finite second moment.
\end{proof}

The second identity in \eqref{eq:tilt} is the reason softmax collapses onto the linear-attention population class here, as exponential tilting shifts a Gaussian mean by exactly $\bSig\ba$, which is linear in the tilt.

\begin{proposition}[Population predictor and linear-attention equivalence]\label{prop:softmax-equiv}
For every fixed parameter configuration in \eqref{eq:softmax-class}, almost surely
\begin{equation}\label{eq:softmax-limit}
\lim_{L_{\rm te}\to\infty}\hat y_q^{(L_{\rm te})} = \ip{\hat\bw(\bfm,\rc)}{\bx_q}+k(\bfm,\rc),
\end{equation}
where
\begin{equation}\label{eq:softmax-wk}
\hat\bw(\bfm,\rc)^\top=\boldsymbol{\ell}_1^\top+\sum_{h=1}^{H}\boldsymbol{c}_h^\top\bSig\bK_h,
\qquad
k(\bfm,\rc)=\ell_0+\sum_{h=1}^{H}\bigl(\boldsymbol{c}_h^\top\bSig\boldsymbol{k}_h+c_{h,0}\bigr).
\end{equation}
Replacing the softmax ratio of head $h$ by its normalized linear counterpart $L_{\rm te}^{-1}\sum_{j\in J_h}s_{hj}f_{hj}$ yields the same limit, minus $\sum_h c_{h,0}$. Consequently, if all value biases $c_{h,0}$ vanish then the softmax and the normalized linear model with the \emph{same} parameters have the same infinite-context predictor almost surely, and with an unrestricted readout intercept the two families realize the same class of infinite-context predictors.
\end{proposition}

\begin{proof}
Condition on $(\bfm,\rc,\bx_q)$ and fix a head, so that $\ba:=\ba_h(\bx_q)$, $\boldsymbol{c}:=\boldsymbol{c}_h$ and $c_0:=c_{h,0}$ are fixed while the context tokens are i.i.d.\ Gaussian. By Lemma \ref{lem:tilt} the numerator and denominator summands are integrable, so the strong law gives
\[
\frac{1}{L_{\rm te}}\sum_{j=1}^{L_{\rm te}}e^{s_{hj}}f_{hj}
\longrightarrow(\boldsymbol{c}^\top\bSig\ba+c_0)\,e^{\ba^\top\bSig\ba/2},
\qquad
\frac{1}{L_{\rm te}}\sum_{j=1}^{L_{\rm te}}e^{s_{hj}}
\longrightarrow e^{\ba^\top\bSig\ba/2}>0 .
\]
If the query participates, its numerator and denominator contributions are $e^{s_{hq}}f_{hq}/L_{\rm te}$ and $e^{s_{hq}}/L_{\rm te}$, which vanish. Head $h$ therefore converges to $\boldsymbol{c}^\top\bSig\ba+c_0$, and summing the finitely many heads and expanding $\ba_h(\bx_q)=\bK_h\bx_q+\boldsymbol{k}_h$ gives \eqref{eq:softmax-limit}--\eqref{eq:softmax-wk}. For the linear counterpart, the strong law gives
\[
\frac{1}{L_{\rm te}}\sum_{j=1}^{L_{\rm te}}s_{hj}f_{hj}
\longrightarrow\ee\bigl[(\ba^\top\bz)(\boldsymbol{c}^\top\bz+c_0)\bigr]
=\ba^\top\bSig\boldsymbol{c}+c_0\,\ba^\top\ee[\bz]
=\boldsymbol{c}^\top\bSig\ba ,
\]
since $\ee[\bz]=\bzero$, and any query contribution again vanishes after division by $L_{\rm te}$. The only difference is the constant $\sum_hc_{h,0}$, which a readout intercept absorbs. Integrating the conditional probability-one statements over $(\bfm,\rc,\bx_q)$ completes the proof.
\end{proof}

\begin{proposition}[The softmax case]\label{prop:softmax-failure}
Suppose $\|\bfm\|$ and $\rc$ are atomless and $\rc$ is independent of $\bfm$. For every fixed parameter configuration in \eqref{eq:softmax-class}, $\P(\lim_{L_{\rm te}\to\infty}\hat y_q^{(L_{\rm te})}=\ip{\bw}{\bx_q})=0$. The same holds for every fixed finite-head normalized linear-attention counterpart of \eqref{eq:softmax-class}.
\end{proposition}

\begin{proof}
Write
\[
\boldsymbol{c}_h=\begin{bmatrix}\bu_h\\ v_h\end{bmatrix},
\qquad
\bK_h=\begin{bmatrix}\bA_h\\ \bb_h^\top\end{bmatrix},
\qquad
\bu_h\in\rr^d,\ v_h\in\rr,\ \bA_h\in\rr^{d\times d},\ \bb_h\in\rr^d .
\]
Since $\bSig=\begin{bsmallmatrix}\bLam&\rc\bfm\\ \rc\bfm^\top&\rc^2\end{bsmallmatrix}$ by \eqref{eq:m_def_and_gaussianity}, block multiplication gives
\begin{equation}\label{eq:block-expand}
\boldsymbol{c}_h^\top\bSig\bK_h
=\bu_h^\top\bLam\bA_h
+\rc\,\bfm^\top\bigl(v_h\bA_h+\bu_h\bb_h^\top\bigr)
+\rc^2 v_h\bb_h^\top .
\end{equation}
Substituting \eqref{eq:block-expand} into \eqref{eq:softmax-wk} and transposing exhibits the form \eqref{eq:poly-slope}, with
\[
\bB:=\sum_{h=1}^{H}\bigl(v_h\bA_h+\bu_h\bb_h^\top\bigr),
\qquad
\bw_0(\bfm)=\boldsymbol{\ell}_1+\sum_{h=1}^{H}\bA_h^\top\bLam\bu_h,
\qquad
\bw_2=\sum_{h=1}^{H}v_h\bb_h ,
\]
all independent of the task except through $\bLam$, which enters only $\bw_0$. The intercept $k$ in \eqref{eq:softmax-wk} is measurable, and Proposition \ref{prop:softmax-equiv} supplies the almost-sure limit, so Lemma \ref{lem:obstruction} applies and gives the first claim. For the linear counterpart, Proposition \ref{prop:softmax-equiv} shows that the slope is unchanged and only the task-independent intercept differs, so Lemma \ref{lem:obstruction} applies verbatim.
\end{proof}

\begin{proof}[Proof of Theorem~\ref{thm:failure_LSA}]
The statement for the linear model \eqref{eq:LSA_def} is Proposition \ref{prop:lsa-failure}, and the statement for the finite-head softmax class \eqref{eq:softmax-class}, which contains the block \eqref{eq:SA-def}, is Proposition \ref{prop:softmax-failure}. In both cases the conclusion is exactly the failure of the criterion \eqref{eq:icl-def}.
\end{proof}

\rev{\subsection{Single-layer attention and deterministic covariance spectrum}\label{app:fixed-spectrum}

Theorem~\ref{thm:failure_LSA} assumes that $\|\bfm\|$ is atomless. To see why this matters, suppose instead that $\|\bfm\|$ takes a given value $r>0$ with positive probability. On that event, all prompts share the same spectrum, even though the eigenvectors of $\bLam$ still vary with $\bfm$. We show below that a single LSA layer can then be tuned to this spectrum, which justifies the claim made after Theorem~\ref{thm:failure_LSA}.

\begin{proposition}\label{prop:fixed-spectrum}
Suppose that $\P(\|\bfm\|=r)>0$ for some $r>0$, and consider the LSA model \eqref{eq:lsa_def_appB} with
\[
\bW^{KQ}=\begin{bmatrix}(1+r^2)^{-1}\bI_d & \bzero_d\\ \bzero_d^\top & 0\end{bmatrix},
\qquad
\bW^{PV}=\be_{d+1}\be_{d+1}^\top .
\]
Then, on the event $\{\|\bfm\|=r\}$, the prediction $\hat y_q$ converges almost surely to the Bayes-optimal prediction $\ip{\bw}{\bx_q}$ as $L_{\rm te}\to\infty$. In particular, this single layer in-context learns with positive probability.
\end{proposition}

\begin{proof}
In the notation of the proof of Proposition~\ref{prop:lsa-failure}, these weights correspond to
\[
\bA=\frac{1}{1+r^2}\,\bI_d,\qquad \bb=\bb'=\bzero_d,\qquad c=0,\qquad \bu=\bzero_d,\qquad v=1.
\]
Substituting these values into \eqref{eq:what_appB}, the limiting regression vector becomes
\[
\hat\bw(\bfm,\rc)=\frac{\rc\,\bfm}{1+r^2},
\]
which coincides with the oracle coefficient $\bw=\rc\bfm/(1+\|\bfm\|^2)$ whenever $\|\bfm\|=r$.
\end{proof}

A single layer succeeds here because $\bw$ is parallel to the top eigenvector of $\bLam$, so it only has to invert one eigenvalue, and on this event that eigenvalue is the same for every prompt. This is why the atomlessness of $\|\bfm\|$ in Theorem~\ref{thm:failure_LSA} cannot be dropped.}

\section{CA algebra and preliminaries}\label{app:ca-algebra}

This appendix collects algebraic identities for the LCA stack and the final LSA readout.
These identities are used in the proofs of Theorems~\ref{thm:one_param_optimal} and~\ref{thm:two_param_optimal}.

\subsection{LSA readout with frozen parameters}\label{app:readout}

Taking inspiration from the initialization strategy of \cite{train_transformers_bartlett}, we initialize the $\bW^{PV}$ and $\bW^{KQ}$ by
\[
\bW^{PV}= \begin{bmatrix}
\bzero_{d\times d} & \bzero_d\\ \bzero_d^\top & 1
\end{bmatrix}, \qquad \bW^{KQ}= \begin{bmatrix} \bI_d & \bzero_d\\ \bzero_d^\top & 0 \end{bmatrix}.
\]

\begin{lemma}[Closed-form readout]\label{lem:readout}
With the above $\bW^{PV},\bW^{KQ}$, the single-layer LSA readout satisfies, for any embedding $\bF=[\f_1,\ldots,\f_L]$,
\[
\widehat y_q
=
\mathrm{LSA}(\bE_{\bF})_{d+1,L+1}
=
\bw^\top \left[\frac{1}{L} \sum_{i=1}^L \bx_i \f_i^\top \right]  \bx_q +\left[\frac{1}{L} \sum_{i=1}^L \varepsilon_{i} \f_{i}^\top \right] \bx_q.
\]
For the one- and two-scalar CA models of Appendix~\ref{app:ca-solve}, fix $d$, $T\geq 1$, and real parameters $\alpha,\beta$ independent of the prompt and of $L$, with $\beta=-\alpha$ in the one-parameter case. Define
\[
P_T(\boldsymbol H):=\alpha\sum_{k=0}^{T-1}(\bI+\beta\boldsymbol H)^k,
\qquad \boldsymbol H\in\rr^{d\times d}.
\]
Then, as $L\to\infty$,
\[
N_L:=\frac{\beps^\top\bF_T^\top\bx_q}{L}\to 0,
\qquad
\hat y_q\to\hat y_{q,\infty}:=\bw^\top\bLam P_T(\bLam)\bx_q
\quad\text{almost surely}.
\]
Under Assumption~\ref{assump:support_m}, both convergences also hold in mean square, and the corresponding expected squared prediction losses converge.
\end{lemma}

\begin{proof}
By definition,
\[
\mathsf{LSA}(\bE_{\bF})
=
\bE_{\bF} + \bW^{PV}\bE_{\bF}\cdot \frac{\bE_{\bF}^\top \bW^{KQ}\bE_{\bF}}{L}.
\]
Since $(\bE_{\bF})_{d+1,L+1}=0$, given the definitions of $\bW^{PV}$ and $\bW^{KQ}$, we have
\begin{align*}
    \hat y_q
    &= \be_{d+1}^\top \bW^{PV}\bE_{\bF}\cdot \frac{\bE_{\bF}^\top \bW^{KQ}\bE_{\bF}}{L}\be_{L+1} \\
    &= \be_{d+1}^\top \frac{\bE_{\bF} \bE_{\bF}^\top}{L} \begin{bmatrix}
        \bx_q \\
        0
    \end{bmatrix} \\
    &= \be_{d+1}^\top \frac{1}{L} \begin{bmatrix}
        \bF\bF^\top + \bx_q \bx_q^\top & \bF \by \\
        \by^\top \bF^\top & \by^\top \by
    \end{bmatrix} \begin{bmatrix}
        \bx_q \\
        0
    \end{bmatrix} \\
    &= \frac{\by^\top \bF^\top \bx_q}{L} \\
    &= \bw^\top \left[\frac{1}{L} \sum_{i=1}^L \bx_i \f_i^\top \right]  \bx_q +\left[\frac{1}{L} \sum_{i=1}^L \varepsilon_{i} \f_{i}^\top \right] \bx_q.
\end{align*}
For the scalar CA models, the recurrence gives $\bF_T=P_T(\widehat\bLam)\bX$, where $\widehat\bLam=\bX\bX^\top/L$, see the algebraic identity in Lemma~\ref{lem:CA_output}. Conditional on $(\bfm,\rc)$, Lemma~\ref{lem:matrix-lln} gives $\widehat\bLam\to\bLam$ and $\bX\beps/L\to\bzero_d$ almost surely. Polynomial continuity therefore implies
\[
\frac{\bF_T\beps}{L}
=P_T(\widehat\bLam)\frac{\bX\beps}{L}\to\bzero_d,
\qquad
\frac{\bX\bF_T^\top}{L}
=\widehat\bLam P_T(\widehat\bLam)\to\bLam P_T(\bLam).
\]
Multiplying by the fixed query $\bx_q$ and using the exact readout identity proves the almost-sure claims. Averaging over the task and query preserves the probability-one statement. This argument does not require independence among the feature columns or a bound on their norms.

For mean-square convergence, assume Assumption~\ref{assump:support_m}. All norms in the following bounds are Euclidean or operator norms, as appropriate. For every integer $p\geq 1$, Jensen's inequality and the conditional Gaussian law give
\[
\sup_{L\geq 1}\ee\!\left[(\tr\widehat\bLam)^p\mid\bfm\right]
\leq\ee[\|\bx_1\|^{2p}\mid\bfm]\leq C_p,
\]
where $C_p$ depends only on $p,d,\overline m$, since $\|\bLam\|\leq 1+\overline m$. Conditional on $(\bfm,\rc,\bX)$, the vector $\beps$ is centered with covariance $\rc^2(1+\|\bfm\|^2)^{-1}\bI_L$ and is independent of the query. Thus,
\[
\ee[N_L^2\mid\bfm,\rc,\bX]
=\frac{\rc^2}{L(1+\|\bfm\|^2)}
\tr\!\left(\bLam P_T(\widehat\bLam)\widehat\bLam P_T(\widehat\bLam)\right).
\]
Because $P_T$ has degree at most $T-1$, the trace is bounded by $C(1+\tr\widehat\bLam)^{2T-1}$, with $C$ depending only on $d,T,\alpha,\beta,\overline m$. The moment bound above yields $\ee[N_L^2\mid\bfm,\rc]\leq C\rc^2/L$ and hence $\ee[N_L^2]\leq C/L\to 0$, using only $\ee[\rc^2]=1$.

To control the signal term, let
\[
\bD_L:=\widehat\bLam P_T(\widehat\bLam)-\bLam P_T(\bLam).
\]
For each fixed $\bfm$, $\bD_L\to\bzero$ almost surely. The same moment bound with $p=4T$ gives $\sup_L\ee[\|\bD_L\|^4\mid\bfm]\leq C$, uniformly in $\bfm$. Consequently, $\|\bD_L\|^2$ is conditionally uniformly integrable, so $\ee[\|\bD_L\|^2\mid\bfm]\to 0$, with these expectations uniformly bounded. Since $\|\bw\|\leq |\rc|/2$ and $\bx_q$ is independent of $\bX$ conditional on the task,
\[
\ee[(\bw^\top\bD_L\bx_q)^2\mid\bfm,\rc]
\leq\frac{1+\overline m}{4}\rc^2\ee[\|\bD_L\|^2\mid\bfm].
\]
Here the conditional law of $\bX$ given $\bfm$ does not depend on $\rc$. Dominated convergence over $(\bfm,\rc)$ now gives $\ee[(\bw^\top\bD_L\bx_q)^2]\to 0$. Combining this with $\hat y_q-\hat y_{q,\infty}=\bw^\top\bD_L\bx_q+N_L$ proves mean-square convergence. Finally, $y_q$ and $\hat y_{q,\infty}$ have finite second moments, so convergence in $L^2$ implies
\[
\lim_{L\to\infty}\ee[(y_q-\hat y_q)^2]
=\ee[(y_q-\hat y_{q,\infty})^2].
\]
\end{proof}

For the scalar CA models, $\bF=\bF_T$ depends on both $L$ and $T$. Parameters may depend on $T$ but are held fixed as $L\to\infty$. From Lemma~\ref{lem:readout}, if
\[
\lim_{T\to\infty}\lim_{L\to\infty}\frac{1}{L}\bX\bF_T^\top=\bI_d
\]
almost surely, then the asymptotic model predictor satisfies
\[
\lim_{T \to \infty} \lim_{L \to \infty} \hat y_q = \ip{\bw}{\bx_q},
\]
hence yielding the Bayes-optimal prediction. The fixed-depth statements hold on a common probability-one event for the countably many depths. This is the key step in establishing Theorems~\ref{thm:one_param_optimal}(2) and~\ref{thm:two_param_optimal}(2).

\subsection{Solving the CA recurrence}\label{app:ca-solve}

In the one- and two-scalar simplification, we take (for all layers)
\[
\bW^K=\bW^Q=\bI_d,\qquad \bW^S=\alpha \bI_d,\qquad \bW^V=\beta \bI_d,
\]
where $\alpha,\beta\in \rr$ are scalar parameters (with the one-parameter model corresponding to $\beta=-\alpha$). Layer-wise, the model then becomes, for $t=1,\dots, T$,
\[
\bF_t = \bF_{t-1} + \alpha \bX + \frac{\beta}{L}\bX\bX^\top \bF_{t-1},
\]
with $\bF_0 = \bzero_{d \times L}$ and $\bF = \bF_T$. Define the sample covariance $\widehat\bLam:=\frac{1}{L}\bX\bX^\top$. The following result yields a closed-form expression for the final output of the CA stack and establishes the parameter window for which one obtains the Bayes-predictor.

\begin{lemma} \label{lem:CA_output}
    The output of the final layer of the CA network is
    \begin{equation} \label{eq:F_t_simplified}
   \bF =  \bF_T = \alpha \sum_{k=0}^{T-1}\bfM^k \bX\quad {\rm where} \quad \bfM = \bI+ \beta \widehat \bLam
     \end{equation}
     and, for $\beta \neq 0$,
     \begin{equation}
         \frac{1}{L}\bX\bF^\top = \frac{\alpha}{\beta}\left[ (\bI+\beta \widehat \bLam)^T - \bI \right].
     \end{equation}
    If $\alpha = - \beta \in (0, 2/ (1 + \overline{m}) )$, then as context length $L \to \infty$ followed by the number of layers $T \to \infty$, we have
    \[
    \frac{1}{L}\bX \bF^\top \to \bI
    \]
    almost surely.
\end{lemma}
\begin{proof}
    We prove the first statement by induction. For the base case $T=1$, we have trivially $\bF_1 = \alpha \bX$ since $\bF_0 = \bzero_{d \times L}$. Now, suppose $\bF_{T-1} = \alpha \sum_{k=0}^{T-2}\bfM^k \bX$. Then,
    \begin{align*}
        \bF_T &= \bF_{T-1} + \alpha \bX + \beta \widehat \bLam \bF_{T-1} \\
        &= \alpha \bX + (\bI + \beta \widehat \bLam)\bF_{T-1} \\
        &= \alpha \bX +  \alpha \bfM \sum_{k=0}^{T-2}\bfM^k \bX \\
        &= \alpha \bX +  \alpha  \sum_{k=1}^{T-1}\bfM^k \bX  =  \alpha \sum_{k=0}^{T-1}\bfM^k \bX
    \end{align*}
    which completes the induction. As $\bfM$ is symmetric and $\widehat \bLam = \beta^{-1}(\bfM-\bI)$, a simple computation then yields
    \begin{align*}
           \frac{1}{L}\bX\bF^\top &= \alpha \widehat \bLam \sum_{k=0}^{T-1}\bfM^k = \frac{\alpha}{\beta} \sum_{k=0}^{T-1}(\bfM^{k+1} - \bfM^k) \\
           &= \frac{\alpha}{\beta} (\bfM^T - \bI) =\frac{\alpha}{\beta}\left[ (\bI+\beta \widehat \bLam)^T - \bI \right].
    \end{align*}
    Now, if $\alpha = - \beta$, we have
    \[
      \frac{1}{L}\bX\bF^\top = \bI - (\bI - \alpha \widehat \bLam)^T \to \bI - (\bI- \alpha \bLam)^T
    \]
    as $L \to \infty$ by the Law of Large Numbers. As $\bLam = \bI + \bfm\bfm^\top$, the eigenvalues of $\bI - \alpha \bLam$ are $1-\alpha$ and $1 - \alpha(1 + \|\bfm\|^2)$. Hence, for the choice of $\alpha \in (0, 2/(1 + \overline{m}))$, it follows that the operator norm $\|\bI- \alpha \bLam\| < 1$ almost surely and so $(\bI- \alpha \bLam)^T \to \bzero_{d \times d}$ as $T \to \infty$ almost surely.
\end{proof}

\begin{remark}
    From the proof of Lemma \ref{lem:CA_output}, we see that at asymptotic context length the error away from the Bayes-optimal prediction decays geometrically in depth $T$ so long as the model parameters satisfy $\alpha = - \beta \in (0, 2/ (1 + \overline{m}))$. Hence, although the result concerns the behavior at infinite depth, in practice small to moderate depth networks should suffice for accurate in-context predictions. This intuition is supported by numerical experiments in Section \ref{sec:numerics}.
\end{remark}

\section{Proof of Theorem~\ref{thm:one_param_optimal}} \label{app:one-param}

This appendix proves Theorem~\ref{thm:one_param_optimal} for the one-parameter model, i.e. the simplified CA stack with $\beta=-\alpha$ trained by gradient flow on the limiting population loss. Throughout, we assume Assumption~\ref{assump:support_m}: $\|\bfm\|^2$ is almost surely bounded and non-degenerate, with $\P(\bfm=\bzero_d)=0$. Define
\[
Z:=1+\|\bfm\|^2,\qquad \underline{Z}:=\mathrm{ess}\inf\, Z,\qquad \overline{Z}:=\mathrm{ess}\sup\,Z,
\]
so $1\le \underline{Z} < \overline{Z} < \infty$. We begin by determining a simplified form for the loss $\ell(\alpha)$ and collecting some of its properties.

\begin{lemma}[One-parameter loss] \label{lem:one_param_loss_properties}
    For each fixed integer $T\geq 1$, the limiting population loss $\ell(\alpha)$ has the form
    \[
    \ell(\alpha) = \ee\!\left[ \frac{Z - 1}{Z}(1 - \alpha Z)^{2T}\right]
    \]
    up to the additive constant $\ee[1/Z]$, where $Z = 1 + \|\bfm\|^2$. Moreover, the loss has the following properties:
    \begin{enumerate}
        \item  $\ell(\alpha)$ is strictly convex and coercive in the sense that
    \[
    \lim_{|\alpha| \to \infty}   \ell(\alpha) = \infty;
    \]
    \item  $\ell(\alpha)$ has a unique minimizer which lies in $(0,1)$.
    \end{enumerate}
\end{lemma}

\begin{proof}
Fix $d$, an integer $T\geq 1$, and $\alpha\in\rr$, independently of $L$ and the prompt. Write
\[
\widehat\bLam=\frac{\bX\bX^\top}{L},\qquad
\bfM=\bI-\alpha\widehat\bLam,\qquad
\bQ_T=\sum_{k=0}^{T-1}\bfM^k.
\]
Here $\bfM^T$ denotes the $T$th matrix power, and transposes are denoted by $\top$. The recurrence $\bF_t=\bfM\bF_{t-1}+\alpha\bX$, $\bF_0=\bzero$, gives $\bF_T=\alpha\bQ_T\bX$ by Lemma~\ref{lem:CA_output}. Since $\bfM$ and $\bQ_T$ are symmetric polynomials in $\widehat\bLam$, telescoping yields
\[
\alpha\widehat\bLam\bQ_T=\alpha\bQ_T\widehat\bLam
=\sum_{k=0}^{T-1}(\bfM^k-\bfM^{k+1})=\bI-\bfM^T.
\]
Consequently, Lemma~\ref{lem:readout} gives the exact readout and error
\begin{align*}
\hat y_q
&=\frac{\by^\top\bF_T^\top\bx_q}{L}
=\bw^\top(\bI-\bfM^T)\bx_q
  +\frac{\alpha}{L}\beps^\top\bX^\top\bQ_T\bx_q,\\
y_q-\hat y_q
&=\bw^\top\bfM^T\bx_q+\varepsilon_q
  -\frac{\alpha}{L}\beps^\top\bX^\top\bQ_T\bx_q.
\end{align*}
Condition first on the full task $\tau=(\bfm,\rc)$ and on $\bX$. Then $\bw=\rc\bfm/Z$ is fixed, $\bx_q$ has covariance $\bLam=\bI+\bfm\bfm^\top$, and the residuals are centered, independent of each other and of all covariates, with variance $\sigma_\tau^2:=\rc^2/Z$ by Lemma~\ref{lem:bayes}. All three cross terms in the squared error therefore have zero conditional expectation. Keeping the matrix order, we obtain
\begin{align*}
\ee[(y_q-\hat y_q)^2\mid\tau,\bX]
&=\bw^\top\bfM^T\bLam\bfM^T\bw+\sigma_\tau^2
  +\frac{\sigma_\tau^2\alpha^2}{L^2}
    \tr(\bLam\bQ_T\bX\bX^\top\bQ_T)\\
&=\bw^\top\bfM^T\bLam\bfM^T\bw+\sigma_\tau^2
  +\frac{\sigma_\tau^2\alpha^2}{L}
    \tr(\bLam\bQ_T\widehat\bLam\bQ_T).
\end{align*}
The last term can equivalently be written as
\[
\frac{\sigma_\tau^2\alpha^2}{L}
\tr(\bLam\bQ_T\widehat\bLam\bQ_T)
=\frac{\sigma_\tau^2\alpha}{L}
\tr\!\left(\bLam\bQ_T(\bI-\bfM^T)\right).
\]
These identities also hold at $\alpha=0$ and for singular $\widehat\bLam$, as no inverse is used. At finite $L$, $\bfM$ need not commute with $\bLam$, so the bias term retains $\bLam$ between the two powers of $\bfM$.

Let $\ell_L(\alpha):=\ee[(y_q-\hat y_q)^2]$ denote the finite-context population loss. As $L\to\infty$, $\widehat\bLam\to\bLam$ almost surely. The mean-square convergence in Lemma~\ref{lem:readout}, applied with $\beta=-\alpha$ and $P_T(\widehat\bLam)=\alpha\bQ_T$, justifies convergence of the expected loss, not just of the predictor. In particular, its Gaussian moment and uniform-integrability argument uses the boundedness of $\bfm$ and only $\ee[\rc^2]=1$, and it also shows that the expected context-noise contribution above vanishes. Thus, with $\bfM_\infty:=\bI-\alpha\bLam$,
\[
\lim_{L\to\infty}\ell_L(\alpha)
=\ee\!\left[\bw^\top\bfM_\infty^T\bLam\bfM_\infty^T\bw
+\sigma_\tau^2\right].
\]
This limit holds at fixed depth and parameters. Only now do we use that $\bfM_\infty$ commutes with $\bLam$, giving $\bfM_\infty^T\bLam\bfM_\infty^T=\bLam(\bI-\alpha\bLam)^{2T}$.

To average over tasks, recall that $\rc$ is independent of $\bfm$, with mean zero and second moment one. Hence,
\[
\bSig_w:=\operatorname{Cov}(\bw\mid\bLam)
=\ee[\bw\bw^\top\mid\bLam]
=\frac{\bfm\bfm^\top}{Z^2}
=\frac{1}{Z}(\bI-\bLam^{-1}),
\qquad
\ee[\sigma_\tau^2\mid\bLam]=\frac{1}{Z}.
\]
Both $Z=\tr(\bLam)-d+1$ and $\bfm\bfm^\top=\bLam-\bI$ are determined by $\bLam$. Write $\bLam=\bU\bGam\bU^\top$, where $\bU$ is orthogonal and $\bGam=\diag(\gamma_1,\ldots,\gamma_d)$. When $\bfm\ne\bzero_d$, choose the first column of $\bU$ along $\bfm$. Then $\gamma_1=Z$, $\gamma_i=1$ for $i\geq 2$, and $\bU^\top\bSig_w\bU=\diag((Z-1)/Z^2,0,\ldots,0)$. Therefore,
\[
\tr\!\left(\bSig_w\bLam(\bI-\alpha\bLam)^{2T}\right)
=\frac{Z-1}{Z}(1-\alpha Z)^{2T}.
\]
The same identity holds when $\bfm=\bzero_d$, since then $Z=1$ and $\bSig_w=\bzero$. Averaging yields
\[
\lim_{L\to\infty}\ell_L(\alpha)
=\ee[1/Z]+\ee\!\left[\frac{Z-1}{Z}(1-\alpha Z)^{2T}\right].
\]
As in the statement, we henceforth omit the additive constant $\ee[1/Z]$ when writing $\ell(\alpha)$, as its derivatives and minimizers are unchanged.
\begin{enumerate}
    \item For $Z>1$, the function $\alpha\mapsto\frac{Z-1}{Z}(1-\alpha Z)^{2T}$ is strictly convex, and for $Z=1$ it is identically zero. Since $\P(Z>1)>0$ under Assumption~\ref{assump:support_m}, taking expectations preserves strict convexity. Moreover, boundedness of $Z$ makes $\ell$ a polynomial of degree $2T$ with leading coefficient $\ee[(Z-1)Z^{2T-1}]>0$. Hence $\ell$ is coercive.
    \item By (1), $\ell(\alpha)$ has a unique minimizer. Differentiating this polynomial gives
    \[
    \ell'(\alpha)=-2T\ee[(Z-1)(1-\alpha Z)^{2T-1}].
    \]
    On $\{Z>1\}$, the integrand is strictly positive for $\alpha\leq 0$ and strictly negative for $\alpha\geq 1$, and it vanishes on $\{Z=1\}$. Thus $\ell'(\alpha)<0$ for $\alpha\leq 0$ and $\ell'(\alpha)>0$ for $\alpha\geq 1$. The unique minimizer lies in $(0,1)$.
\end{enumerate}
\end{proof}

\subsection{Convergence of gradient flow}

Given the above established properties of $\ell(\alpha)$, by standard dynamical systems theory results, we can derive the global convergence of gradient flow on the one-parameter loss. We include a detailed proof to initiate the unfamiliar reader.

\begin{lemma}[One-parameter gradient flow convergence] \label{lem:one_param_gf_conv}
    Suppose Assumption \ref{assump:support_m} holds. Then, gradient flow on $\ell(\alpha)$ converges globally to the unique minimum of $\ell$ whose existence is guaranteed by Lemma \ref{lem:one_param_loss_properties}.
\end{lemma}

\begin{proof}
    Let $(\alpha_t)_{t \geq 0}$ be the gradient flow trajectory initialized at $\alpha_0 \in \rr$ and let $\alpha^\ast$ denote the unique minimizer of $\ell(\alpha)$.  By chain rule and the definition of gradient flow, we have
    \[
    \frac{\d}{\d t} \ell(\alpha_t) = \nabla_\alpha \ell(\alpha_t)\dot \alpha_t = - \|\nabla_\alpha \ell(\alpha_t)\|^2 \leq 0.
    \]
    Hence, the trajectory of the loss function $\ell(\alpha_t)$ is decreasing and since $\ell(\alpha) \geq 0$, it follows that $\ell(\alpha_t)$ converges as $t \to \infty$. Moreover, by coercivity of $\ell$, $\alpha_t \in \{\alpha: \ell(\alpha) \leq \ell(\alpha_0)\}$ for all $t$ which is compact. Thus, since $\ell'$ is Lipschitz on this compact set and $|\dot\alpha_t|=|\ell'(\alpha_t)|$ is bounded there, $t\mapsto\ell'(\alpha_t)$ is uniformly continuous. Then, by Barbalat's Lemma \citep{slotine1991applied}, $\ell'(\alpha_t) \to 0$ as $t \to \infty$. As  $\alpha_t \in \{\alpha: \ell(\alpha) \leq \ell(\alpha_0)\}$ which is compact, moving to a convergent subsequence, $\alpha_{t_k} \to \bar \alpha$ as $k \to \infty$. By continuity of $\ell'$, $\ell'(\alpha_{t_k}) \to \ell'(\bar \alpha) = 0$. Hence, $\bar \alpha = \alpha^\ast$ which is the unique minimizer of $\ell$. Hence, as every convergent subsequence of $(\alpha_t)_{t \geq 0}$ has limit $\alpha^\ast$ and the trajectory lies in a compact set, it follows that $\alpha_t \to \alpha^\ast$ as $t \to \infty$.
\end{proof}

\subsection{Asymptotic depth behavior}

Having thus far established the convergence of gradient flow at finite depth $T$, we now complete the proof of Theorem \ref{thm:one_param_optimal} with the following result which shows that as $T \to \infty$, we recover the optimal prediction of the model. Formally, let $\alpha^\ast_T$ denote the limit of the gradient flow on the $T$-layer model which, from Lemma \ref{lem:one_param_gf_conv}, exists  and coincides with the global minimizer of $\ell(\alpha)$.

\begin{theorem} \label{thm:one_param_optimal_limit}
        Suppose Assumption \ref{assump:support_m} holds.  Let $\alpha^\ast_T$ denote the global minimizer of $\ell(\alpha)$ for the $T$-layer CA model which exists uniquely by Lemma \ref{lem:one_param_loss_properties}. Then,
    \[
    \lim_{T \to \infty} \alpha^\ast_T =: \alpha^\ast_\infty = \frac{2}{2 + \underline{m} + \overline{m}} = \frac{2}{\underline{Z} + \overline{Z}}.
    \]
    The resulting predictor satisfies $\lim_{T\to\infty}\lim_{L_{\rm te}\to\infty}\hat y_q^{\mathrm{(te)}}=y_q^{\mathrm{Bayes},\mathrm{(te)}}$ almost surely, with the context-length limit taken first.
\end{theorem}

\begin{proof}
By Lemma~\ref{lem:one_param_loss_properties}, the excess loss at depth $T$ is
\[
R_T(\alpha):=\ee\!\left[W(1-\alpha Z)^{2T}\right],
\qquad W:=\frac{Z-1}{Z}.
\]
The total limiting loss is $\ee[1/Z]+R_T(\alpha)$, so $R_T$ and
$f_T(\alpha):=R_T(\alpha)^{1/(2T)}$ have the same unique minimizer $\alpha_T^\ast\in(0,1)$.
Introduce the probability measure
\[
\d\mu:=\frac{W}{\ee[W]}\,\d\P.
\]
The continuous law of $\bfm$ in Assumption~\ref{assump:support_m} implies $\P(\bfm=\bzero_d)=0$, so $Z>1$ and $W>0$ almost surely. Thus $0<\ee[W]<1$ and $\mu$ and $\P$ have the same null sets. In particular, $Z$ has the same essential endpoints $\underline Z,\overline Z$ under both measures. This also covers $\underline Z=1$, since although the weight vanishes at $1$, every neighborhood of this endpoint has positive $\mu$-mass. No positive endpoint-density assumption is needed.

For each fixed $\alpha\in[0,1]$, write
\[
f_T(\alpha)=\ee[W]^{1/(2T)}\|1-\alpha Z\|_{L^{2T}(\mu)}.
\]
For completeness, if $h$ is a bounded nonnegative function with $H:=\|h\|_{L^\infty(\mu)}>0$, then for every $s\in(0,H)$,
\[
s\,\mu(h>s)^{1/p}\leq\|h\|_{L^p(\mu)}\leq H,
\qquad \mu(h>s)>0.
\]
Letting $p\to\infty$ and then $s\uparrow H$ proves $\|h\|_{L^p(\mu)}\to H$, and the case $H=0$ is immediate. Applied to $h=|1-\alpha Z|$, this gives
\begin{equation}\label{eq:limsup_bound}
\lim_{T\to\infty}f_T(\alpha)=\phi(\alpha),
\qquad
\phi(\alpha):=\max\{|1-\alpha\underline Z|,|1-\alpha\overline Z|\}.
\end{equation}
Here convexity of $z\mapsto|1-\alpha z|$ bounds its value by the endpoint maximum, while continuity and positive mass in every endpoint neighborhood give the reverse inequality for the essential supremum.

To pass from fixed parameters to minimizers, the reverse triangle inequality for $L^{2T}(\mu)$ yields, for $\alpha,a\in[0,1]$,
\begin{align*}
|f_T(\alpha)-f_T(a)|
&\leq\ee[W]^{1/(2T)}\|(\alpha-a)Z\|_{L^{2T}(\mu)}\\
&\leq\overline Z|\alpha-a|.
\end{align*}
The constant is independent of $T$. Consequently, for every sequence $\alpha_T\to\alpha$ in $[0,1]$, equation~\eqref{eq:limsup_bound} implies
\begin{equation}\label{eq:liminf_bound}
|f_T(\alpha_T)-\phi(\alpha)|
\leq\overline Z|\alpha_T-\alpha|+|f_T(\alpha)-\phi(\alpha)|
\longrightarrow 0.
\end{equation}
The same argument applies along any subsequence of depths tending to infinity.

Set $\alpha_\infty^\ast:=2/(\underline Z+\overline Z)\in(0,1)$. Its endpoint errors have equal magnitudes and opposite signs, so
\[
\phi(\alpha_\infty^\ast)=\frac{\overline Z-\underline Z}{\overline Z+\underline Z}.
\]
If $\alpha<\alpha_\infty^\ast$, then $\phi(\alpha)\geq1-\alpha\underline Z>\phi(\alpha_\infty^\ast)$, and if $\alpha>\alpha_\infty^\ast$, then $\phi(\alpha)\geq\alpha\overline Z-1>\phi(\alpha_\infty^\ast)$. Hence $\alpha_\infty^\ast$ is the unique minimizer of $\phi$.

By compactness of $[0,1]$, any subsequence of $(\alpha_T^\ast)$ has a further convergent subsequence, say $\alpha_{T_k}^\ast\to\bar\alpha$. Minimality gives
\[
f_{T_k}(\alpha_{T_k}^\ast)\leq f_{T_k}(\alpha_\infty^\ast).
\]
Using \eqref{eq:limsup_bound} and \eqref{eq:liminf_bound}, we obtain $\phi(\bar\alpha)\leq\phi(\alpha_\infty^\ast)$, so $\bar\alpha=\alpha_\infty^\ast$. All accumulation points therefore coincide, proving $\alpha_T^\ast\to\alpha_\infty^\ast$.

It remains to check prediction convergence with the depth-dependent parameters. Since $0<\alpha_\infty^\ast<2/\overline Z$, there is a compact interval $J\subset(0,2/\overline Z)$ containing $\alpha_T^\ast$ for all sufficiently large $T$. Every eigenvalue of $\bLam$ lies in $[1,\overline Z]$, including the eigenvalue $1$ on the orthogonal complement of $\bfm$. Thus,
\[
\rho:=\max_{a\in J,\,\lambda\in[1,\overline Z]}|1-a\lambda|<1,
\qquad
\|(\bI-\alpha_T^\ast\bLam)^T\|\leq\rho^T
\]
for all sufficiently large $T$, almost surely. At each fixed depth, Lemmas~\ref{lem:readout} and~\ref{lem:CA_output} give
\[
\lim_{L_{\rm te}\to\infty}\hat y_q^{\mathrm{(te)}}
=\bw^{\mathrm{(te)}\top}
\left[\bI-(\bI-\alpha_T^\ast\bLam^{\mathrm{(te)}})^T\right]\bx_q^{\mathrm{(te)}}
\]
on a common probability-one event for the countably many depths. The difference from $\bw^{\mathrm{(te)}\top}\bx_q^{\mathrm{(te)}}$ is bounded in absolute value by $\|\bw^{\mathrm{(te)}}\|\|\bx_q^{\mathrm{(te)}}\|\rho^T$ and therefore tends to zero. This proves the stated context-then-depth prediction limit and completes Theorem~\ref{thm:one_param_optimal}(2).
\end{proof}

\section{Proof of Theorem \ref{thm:two_param_optimal}} \label{app:two_param}

We now turn our attention to the model with the pair of learnable parameters $(\alpha,\beta)$. As before, we begin by determining a simplified form for the loss $\ell(\alpha, \beta)$ and collect some properties. Recall, $Z = 1 + \|\bfm\|^2$.

Throughout this appendix, the two-parameter optimization analysis concerns integer depths $T\geq2$. At fixed depth, we suppress the dependence of $S$, $A$, $B$, $\ell$, and the profiled minimizer $\alpha^\ast(\beta)$ on $T$, and constants in fixed-depth bounds may also depend on $T$.

Let
\[
S:= S(Z, \beta) = \frac{(1+\beta Z)^T - 1}{\beta} \quad {\rm where} \quad S(Z, 0) = \lim_{\beta \to 0}S(Z, \beta) = TZ,
\]
\[
W = \frac{Z-1}{Z},
\]
and
\[
A(\beta) = \ee[WS^2], \quad B(\beta) = \ee[WS].
\]

\begin{lemma}[Two-parameter loss] \label{lem:two_param_loss_properties}
    The loss $\ell(\alpha, \beta)$ has the form
    \[
     \ell(\alpha, \beta) = \ee\left[\frac{Z - 1}{Z}\left( \frac{\alpha}{\beta} \big[(1 + \beta Z)^T - 1\big] - 1\right)^2\right]
    \]
    up to the additive constant $\ee[1/Z]$, or equivalently $\ell(\alpha,\beta) = \ee[W(\alpha S(Z,\beta) - 1)^2]$, which also covers $\beta = 0$ through $S(Z,0) = TZ$. Moreover, for fixed $\beta \in \rr$, the loss is strongly convex in $\alpha$, with unique minimizer:
    \begin{equation} \label{eq:unique_beta_min}
    \alpha^\ast(\beta) := \frac{B(\beta)}{A(\beta)}.
      \end{equation}
\end{lemma}

\begin{proof}
We have
\begin{align*}
    \frac{\by^\top \bF^\top \bx_q}{L} &= \frac{\bw^\top \bX \bF^\top \bx_q}{L} +  \frac{\beps^\top \bF^\top \bx_q}{L} \\
    &= \frac{\alpha}{\beta} \bw^\top (\bfM^T - \bI) \bx_q +  \frac{\beps^\top \bF^\top \bx_q}{L} \\
    &= - \frac{\alpha}{\beta} \bw^\top \bx_q + \frac{\alpha}{\beta} \bw^\top (\bI + \beta \hat \bLam)^T\bx_q +   \frac{\beps^\top \bF^\top \bx_q}{L}
\end{align*}
and so
\[
y_q - \hat y_q = \left(1 + \frac{\alpha}{\beta}\right) \bw^\top \bx_q - \frac{\alpha}{\beta} \bw^\top (\bI + \beta \hat \bLam)^T\bx_q -   \frac{\beps^\top \bF^\top \bx_q}{L} + \varepsilon_q.
\]
Taking $L\to\infty$ at fixed $T,\alpha,\beta$, with convergence of the noise term in mean square and of the expected loss justified by Lemma~\ref{lem:readout}, and omitting $\ee[\varepsilon_q^2|\bLam]$ which does not depend on $(\alpha,\beta)$, we have, up to the additive constant $\ee[1/Z]$,
\begin{align*}
\ell(\alpha, \beta| \bLam) &\propto \left(1 + \frac{\alpha}{\beta}\right)^2 \ee[ (\bw^\top \bx_q)^2|\bLam] + \left(\frac{\alpha}{\beta}\right)^2 \ee[(\bw^\top (\bI + \beta \bLam)^T\bx_q)^2|\bLam] \\
&\qquad\qquad -  \frac{2\alpha}{\beta}\left(1 + \frac{\alpha}{\beta}\right)\ee[\bw^\top \bx_q  \bw^\top (\bI + \beta \bLam)^T\bx_q|\bLam] \\
    &= \left(1 + \frac{\alpha}{\beta}\right)^2 \tr(\bSig_w \bLam) + \left(\frac{\alpha}{\beta}\right)^2 \tr(\bSig_w \bLam (\bI + \beta \bLam)^{2T}) \\
    &\qquad\qquad -  \frac{2\alpha}{\beta}\left(1 + \frac{\alpha}{\beta}\right) \tr(\bSig_w \bLam(\bI + \beta \bLam)^T).
\end{align*}
In the first line we used that $\ee[\varepsilon_q] = 0$ to get rid of cross terms and in the second, we use that $\bSig_w$ commutes with $\bLam$. Recalling that
\[
\bSig_w = \frac{\bfm\bfm^\top}{(1 + \|\bfm\|^2)^2} = \frac{1}{Z}\bigl(\bI - \bLam^{-1}\bigr),
\]
the loss simplifies to
\begin{align*}
      \ell(\alpha, \beta| \bLam) &= \left(1 + \frac{\alpha}{\beta}\right)^2 \frac{1}{Z}\sum_{i=1}^d (\gamma_i - 1)  + \left(\frac{\alpha}{\beta}\right)^2 \frac{1}{Z}\sum_{i=1}^d (\gamma_i - 1)(1 + \beta\gamma_i)^{2T} \\
      &\qquad\qquad - 2 \frac{\alpha}{\beta}\left(1 + \frac{\alpha}{\beta}\right) \frac{1}{Z}\sum_{i=1}^d (\gamma_i - 1)(1 + \beta\gamma_i)^{T} \\
      &= \left(1 + \frac{\alpha}{\beta}\right)^2 \frac{Z - 1}{Z} + \left(\frac{\alpha}{\beta}\right)^2 \frac{Z - 1}{Z}(1 + \beta Z)^{2T} \\
      &\qquad\qquad - 2 \frac{\alpha}{\beta}\left(1 + \frac{\alpha}{\beta}\right) \frac{Z - 1}{Z}(1 + \beta Z)^{T} \\
      &= \frac{Z - 1}{Z}\left( \frac{\alpha}{\beta} \big[(1 + \beta Z)^T - 1\big] - 1\right)^2.
\end{align*}

For the unconditional loss, one simply takes expectation:
\[
  \ell(\alpha, \beta) = \ee\left[\frac{Z - 1}{Z}\left( \frac{\alpha}{\beta} \big[(1 + \beta Z)^T - 1\big] - 1\right)^2\right].
\]
Expanding the above square, we see
\[
\ell(\alpha, \beta) = \alpha^2 \ee[WS^2] - 2 \alpha \ee[WS] + \ee[W] = \alpha^2 A(\beta)  - 2 \alpha B(\beta) + \ee[W]
\]
is quadratic in $\alpha$ with leading coefficient $A(\beta)>0$ (Lemma~\ref{lem:A_beta_properties} below)---hence strongly convex---and thus has a unique minimizer at $B(\beta)/A(\beta)$.
\end{proof}

In light of Lemma \ref{lem:two_param_loss_properties}, profiling out $\alpha$, we can write a reduced loss
\[
  F_T(\beta)
  := \min_{\alpha\in\mathbb{R}} \ell(\alpha,\beta)
  =  \ell\bigl(\alpha^\ast(\beta),\beta\bigr) = \ee[W] - \frac{B(\beta)^2}{A(\beta)}.
\]

At this point, we notice that neither the loss $\ell(\alpha, \beta)$ nor the reduced loss $F_T(\beta)$ are convex. Hence, the subsequent analysis establishing the convergence of gradient flow will require a different and more involved route than in the one-parameter setup.

\subsection{Convergence of gradient flow}

In this section, we establish that at sufficiently large finite depth $T$, gradient flow on the two-parameter loss converges under suitable initialization. We first show that, regardless of initialization, the iterates $(\alpha_t)_{t \geq 0}$ remain bounded. Then, for a fixed initialization $\beta_0 \in (-2/\overline{Z}, 0)$ and $\alpha_0 = \alpha^\ast(\beta_0)$, we show that the iterates $(\beta_t)_{t \geq 0}$ remain in a compact subinterval of $(-2/\overline{Z}, 0)$, which does not depend on the depth, once $T\geq T'(\beta_0)$, as specified in Lemma~\ref{lem:loss_ineq}. Under this setup where the iterates $(\alpha_t, \beta_t)_{t \geq 0}$ remain bounded, we conclude by showing they must converge as $t \to \infty$. No property of the law of $Z$ beyond Assumption~\ref{assump:support_m} is used in this appendix.

We first establish some preliminary properties of the function $A(\beta)$.

\begin{lemma}\label{lem:A_beta_properties}
    Suppose Assumption \ref{assump:support_m} holds. Then,
    \begin{enumerate}
        \item $A(\beta)$ is continuous and $\lim_{|\beta| \to \infty}A(\beta) = +\infty$.
        \item $A^\ast := \inf_{\beta \in \rr} A(\beta) > 0$.
    \end{enumerate}
\end{lemma}

\begin{proof}
\begin{enumerate}
    \item The binomial expansion gives the polynomial
    \[
    S(Z,\beta)=\sum_{j=1}^T\binom{T}{j}\beta^{j-1}Z^j.
    \]
    Since $Z$ is bounded, $A$ is a polynomial and hence continuous. For each fixed $T\geq2$,
    \[
    A(\beta)=\ee[(Z-1)Z^{2T-1}]\,\beta^{2T-2}
    +O(|\beta|^{2T-3})
    \qquad (|\beta|\to\infty).
    \]
    Its leading coefficient is strictly positive because $\P(Z>1)>0$. The degree $2T-2$ is positive and even, so $A(\beta)\to\infty$ as $|\beta|\to\infty$.
    \item  By part 1, we may choose $C > 0$ so large that $A(\beta) \geq A(0)$ for all $|\beta| > C$, and then $A^\ast = \min_{|\beta| \leq C} A(\beta) =: A_{\min}$, where the minimum is attained by continuity of $A$. Assume for contradiction that $A_{\min} = 0$. Then, $S(Z, \beta^\ast) = 0$ almost surely on $\{Z > 1\}$ for some $\beta^\ast$, and $Z>1$ almost surely under Assumption \ref{assump:support_m}, since the continuous law of $\bfm$ gives $\P(\bfm=\bzero_d)=0$. If $\beta^\ast = 0$, this is impossible because $S(Z,0) = TZ > 0$. If $\beta^\ast \neq 0$, it forces $(1+\beta^\ast Z)^T = 1$ almost surely. For odd $T$, the only real solution of $r^T = 1$ is $r=1$, which would give $Z = 0$, and for even $T$, the real solutions are $r = \pm 1$, where $r=1$ is again impossible and $r=-1$ would make $Z = -2/\beta^\ast$ almost surely constant. Both cases contradict $1 \leq \underline{Z} < \overline{Z}$. Hence $A^\ast = A_{\min} > 0$.
    \end{enumerate}
\end{proof}

We can show that the iterates $(\alpha_t)_{t \geq 0}$ remain bounded regardless of initialization.

\begin{lemma} \label{lem:alpha_t_bounded}
        Suppose Assumption \ref{assump:support_m} holds. Then,
     \[
     \sup_{\beta \in \rr} |\alpha^\ast(\beta)| \leq \sqrt{\frac{\ee[W]}{A^\ast}} =: M^\ast < \infty.
     \]
     Moreover, letting $(\alpha_t, \beta_t)_{0 \leq t < \tau}$ be any gradient flow trajectory on $\ell$, defined on an interval $[0,\tau)$ with $0 < \tau \leq \infty$, we have
     \[
     \sup_{0 \leq t < \tau} |\alpha_t| \leq M^\ast + \sqrt{\frac{\ell(\alpha_0, \beta_0)}{A^\ast}} =: M_{\alpha_0, \beta_0} < \infty.
     \]
\end{lemma}

\begin{proof}
    By the Cauchy-Schwarz inequality,
    \[
    B(\beta)^2 = \ee[W^{1/2}\cdot W^{1/2}S(\beta)]^2 \leq \ee[W] \ee[W S(\beta)^2] = \ee[W] A(\beta).
    \]
    By Lemma \ref{lem:A_beta_properties},
    \[
    \sup_{\beta \in \rr} |\alpha^\ast(\beta)| =  \sup_{\beta \in \rr} \frac{|B(\beta)|}{A(\beta)} \leq \sqrt{\frac{\ee[W]}{A^\ast}} = M^\ast < \infty.
    \]
    Now, setting $\Delta_t = \alpha_t - \alpha^\ast(\beta_t)$, notice that we may write
    \[
    \ell(\alpha, \beta) = A(\beta) \Delta^2 + \underbrace{\ee[W] - \frac{B(\beta)^2}{A(\beta)}}_{:=C(\beta) \geq 0}
    \]
    where the non-negativity of $C(\beta)$ follows from the inequality $ B(\beta)^2 \leq \ee[W] A(\beta)$. Recall that the loss $\ell$ is non-increasing along gradient flow on $[0,\tau)$, which follows from $\frac{\d}{\d t} \ell(\alpha_t, \beta_t) = - \|\nabla \ell(\alpha_t, \beta_t)\|^2 \leq 0$. Hence, $\ell(\alpha_t, \beta_t) \leq \ell(\alpha_0, \beta_0)$ for all $t \in [0,\tau)$ and so
    \[
    A(\beta_t) \Delta_t^2 = \ell(\alpha_t, \beta_t) - C(\beta_t) \leq \ell(\alpha_0, \beta_0).
    \]
    Since $A(\beta_t) \geq A^\ast$ by Lemma \ref{lem:A_beta_properties}, we obtain $\sup_{0 \leq t < \tau} |\Delta_t| \leq \sqrt{\ell(\alpha_0, \beta_0)/ A^\ast}$ and so
    \[
    \sup_{0 \leq t < \tau}|\alpha_t| \leq  \sup_{0 \leq t < \tau}|\alpha^\ast(\beta_t)| +  \sup_{0 \leq t < \tau}|\Delta_t|  \leq M^\ast + \sqrt{\frac{\ell(\alpha_0, \beta_0)}{A^\ast}} < \infty
    \]
    as desired.
\end{proof}

Now, focusing on the iterates $(\beta_t)_{t \geq 0}$, we show that, for all sufficiently large depths, they remain in a single compact subinterval of $(-2/\overline{Z}, 0)$ which does not depend on the depth. We use the same weighted probability measure as in Appendix~\ref{app:one-param}, namely
\[
\d\mu = \frac{W}{\ee[W]} \,\d \P
\]
so that $\mu$ is obtained by reweighting $\d \P$ by $W$ and $\ee_\mu[Y] = \ee[WY]/\ee[W]$ for a random variable $Y$. As noted in the proof of Theorem~\ref{thm:one_param_optimal_limit}, Assumption~\ref{assump:support_m} gives $Z>1$ and $W>0$ almost surely, so $\mu$ and $\P$ have the same null sets and $Z$ has the same essential endpoints $\underline{Z} < \overline{Z}$ under both measures. In particular,
\begin{equation} \label{eq:endpoint-mass}
\mu(Z \leq x) > 0 \quad \text{and} \quad \mu(Z \geq y) > 0 \qquad \text{for all } \underline{Z} < x < y < \overline{Z}.
\end{equation}
Beyond $Z>1$ almost surely and $1 \leq \underline{Z} < \overline{Z} < \infty$, the positivity of the $\mu$-mass near the two essential endpoints is the only property of the law of $Z$ used in the remainder of this appendix, so in particular $Z$ need not have a density and may have atoms.

Writing
\[
u_T(\beta, Z) := (1 + \beta Z)^T,
\]
a short calculation shows
\begin{equation}  \label{eq:FT-variance-form}
  F_T(\beta)
  = \ee[W]
  \frac{\var_\mu(u_T(\beta,Z))}
       {\ee_\mu\bigl[(1-u_T(\beta,Z))^2\bigr]}
\end{equation}
for any $\beta \neq 0$ which will provide a convenient representation of the loss in subsequent proofs. For $\beta\in(-2/\overline{Z},0)$ we also define
\begin{equation} \label{eq:phi-def}
\phi(\beta) := \max \{|1 + \beta \underline{Z}|, |1 + \beta \overline{Z}|\}.
\end{equation}
For such $\beta$ and $z \in [\underline{Z},\overline{Z}]$ we have $1+\beta z \in (-1,1)$, so $\phi(\beta)<1$, and also $\phi(\beta)>0$, since $1+\beta\underline{Z}$ and $1+\beta\overline{Z}$ cannot both vanish. By convexity of $z \mapsto |1+\beta z|$,
\begin{equation} \label{eq:uT-phi-bound}
|u_T(\beta,z)| \leq \phi(\beta)^T \qquad \text{for all } z \in [\underline{Z},\overline{Z}],
\end{equation}
and, for every $\beta \in [-2/\overline{Z},0]$, $|u_T(\beta,z)|\leq 1$ on $[\underline{Z},\overline{Z}]$, so the denominator in \eqref{eq:FT-variance-form} is at most $4$. Finally, we record an elementary lower bound for variances. If $Z'$ is an independent copy of $Z$ under $\mu$, then $\var_\mu(h(Z)) = \tfrac12 \ee_\mu[(h(Z)-h(Z'))^2]$, and hence, if $\mathcal{A}, \mathcal{B} \subset \rr$ are disjoint Borel sets with $|h(z)-h(z')| \geq \vartheta$ for all $z \in \mathcal{A}$ and $z' \in \mathcal{B}$, then, restricting the integration to $\{Z \in \mathcal{A},\, Z' \in \mathcal{B}\}\cup\{Z \in \mathcal{B},\, Z' \in \mathcal{A}\}$,
\begin{equation} \label{eq:pair-variance}
\var_\mu(h(Z)) \geq \mu(Z \in \mathcal{A})\,\mu(Z \in \mathcal{B})\,\vartheta^2.
\end{equation}

\begin{lemma}[Loss near the boundary of the stability interval] \label{lem:boundary_barriers}
Suppose Assumption~\ref{assump:support_m} holds and fix $\rho \in (0,1)$. There exist constants $\delta_L, \delta_R \in (0, 1/(4\overline{Z})]$, $c_L, c_R > 0$, $r_L, r_R \in (\rho, 1)$ and an integer $N_L \geq 1$, depending only on $\rho$ and on the law of $Z$, such that
\begin{enumerate}
    \item $F_T(\beta) \geq c_R\, r_R^{2T}$ for every integer $T \geq 1$ and every $\beta \in [-\delta_R, 0]$, and
    \item $F_T(\beta) \geq c_L\, r_L^{2T}$ for every integer $T \geq N_L$ and every $\beta \in [-2/\overline{Z}, -2/\overline{Z}+\delta_L]$.
\end{enumerate}
\end{lemma}

\begin{proof}
Write $a:=\underline{Z}$, $b:=\overline{Z}$, and fix $x := a + (b-a)/3$ and $y := a+2(b-a)/3$, so that $a<x<y<b$. Set $\Delta := y-x>0$ and
\[
p := \mu(Z \leq x) > 0, \qquad q := \mu(Z \geq y) > 0,
\]
where the positivity follows from \eqref{eq:endpoint-mass}.

\paragraph{The boundary at zero.} Define
\[
\delta_R := \min\left\{\frac{1}{4b}, \frac{1-\rho}{2x}\right\}, \qquad r_R := 1-\delta_R x, \qquad c_R := \ee[W]\,p\,q\,(1-e^{-\Delta/b})^2.
\]
Since $\delta_R x \leq (1-\rho)/2$, we have $r_R \geq (1+\rho)/2 > \rho$, and clearly $r_R<1$. Fix $T \geq 1$ and $\beta = -s$ with $0 < s \leq \delta_R$, and put $h := Ts$. Since $sz \leq sb \leq 1/4$ for $z\in[a,b]$, the function $z \mapsto u_T(-s,z) = (1-sz)^T$ is positive and decreasing on $[a,b]$. Hence, for $z \leq x$ and $z' \geq y$,
\begin{align*}
u_T(-s,z) - u_T(-s,z') &\geq (1-sx)^T - (1-sy)^T = (1-sx)^T\left[1 - \left(1 - \frac{s\Delta}{1-sx}\right)^T\right] \\
&\geq (1-sx)^T\left(1 - e^{-h\Delta}\right),
\end{align*}
where we used $0 < s\Delta/(1-sx) < 1$, the inequality $1-v \leq e^{-v}$, and $(1-sx)^{-1} \geq 1$. By \eqref{eq:pair-variance} with $\mathcal{A} = (-\infty, x]$ and $\mathcal{B} = [y, \infty)$,
\begin{equation} \label{eq:right-var}
\var_\mu(u_T(-s,Z)) \geq p\,q\,(1-sx)^{2T}\left(1-e^{-h\Delta}\right)^2.
\end{equation}
At the same time, Bernoulli's inequality $(1-v)^T \geq 1 - Tv$ for $v\in[0,1]$ gives, pointwise on $[a,b]$,
\[
0 \leq 1 - (1-sZ)^T \leq \min\{1, TsZ\} \leq \min\{1, hb\},
\]
and hence
\[
\ee_\mu\bigl[(1-u_T(-s,Z))^2\bigr] \leq \min\{1,hb\}^2 .
\]
Finally, for every $h>0$,
\begin{equation} \label{eq:concavity-gap}
1 - e^{-h\Delta} \geq (1-e^{-\Delta/b}) \min\{1, hb\}:
\end{equation}
if $hb \leq 1$, this follows from the concavity of $v \mapsto 1-e^{-v\Delta/b}$ on $[0,1]$, evaluated at $v = hb$, and if $hb \geq 1$, from monotonicity. Combining \eqref{eq:FT-variance-form}, \eqref{eq:right-var}, the denominator bound and \eqref{eq:concavity-gap}, the factors $\min\{1,hb\}^2$ cancel and
\[
F_T(-s) \geq \ee[W]\,p\,q\,(1-e^{-\Delta/b})^2 (1-sx)^{2T} \geq c_R\, r_R^{2T}.
\]
The bound is uniform in $h = Ts$, so in particular it covers the regime $T|\beta| \to 0$, in which both the variance and the denominator in \eqref{eq:FT-variance-form} are of order $\min\{1,hb\}^2$. Since $F_T = \ee[W] - B^2/A$ is continuous on $\rr$ by Lemma~\ref{lem:A_beta_properties}, the bound extends to $s=0$, proving (1).

\paragraph{The boundary at $-2/\overline{Z}$.} Let
\[
q_0 := \max\{|1-2a/b|, |1-2x/b|\} < 1,
\]
where the strict inequality holds because $0<a<x<b$. Choose
\[
\eta := \min\left\{\frac{b-x}{2}, \frac{b\,(1-\max\{\rho,q_0\})}{8}\right\}, \qquad r_0 := 1 - \frac{2\eta}{b},
\]
so that $b - \eta > x$ and $r_0 \geq 1 - (1-\max\{\rho,q_0\})/4 > \max\{\rho, q_0\}$. Define
\begin{gather*}
\delta_L := \frac{\min\{1,\, r_0-\rho,\, r_0-q_0\}}{4b}, \qquad r_L := r_0 - b\delta_L, \qquad Q_L := q_0 + b\delta_L, \\
q' := \mu(Z \geq b-\eta) > 0, \qquad c_L := \frac{\ee[W]\,p\,q'}{16}.
\end{gather*}
Then $r_L \geq (3r_0+\rho)/4 > \rho$, $r_L<1$, and $r_L - Q_L \geq (r_0-q_0)/2 > 0$. Fix $\beta = -2/b + s$ with $0 \leq s \leq \delta_L$. For $z \in [a,x]$, since the affine function $z\mapsto 1-2z/b$ attains the maximum of its absolute value over $[a,x]$ at an endpoint,
\[
|1+\beta z| \leq |1-2z/b| + sz \leq q_0 + b\delta_L = Q_L ,
\]
where the absolute value in the definition of $q_0$ is needed because $1-2a/b<0$ when $a>b/2$. For $z \in [b-\eta, b]$,
\[
1 + \beta z = 1 - \frac{2z}{b} + sz \leq 1 - \frac{2(b-\eta)}{b} + \delta_L b = -r_L < 0.
\]
Consequently $|u_T(\beta,z)| \leq Q_L^T$ for $z \leq x$ and $|u_T(\beta,z')| \geq r_L^T$ for $z' \geq b-\eta$, so that, by the reverse triangle inequality and for either parity of $T$,
\[
|u_T(\beta,z) - u_T(\beta,z')| \geq r_L^T - Q_L^T \qquad (z \leq x,\ z' \geq b-\eta).
\]
Since the denominator in \eqref{eq:FT-variance-form} is at most $4$ on $[-2/b,0]$, \eqref{eq:pair-variance} with $\mathcal{A} = (-\infty, x]$ and $\mathcal{B} = [b-\eta, \infty)$ yields
\[
F_T(\beta) \geq \frac{\ee[W]}{4}\, p\, q'\, \bigl(r_L^T - Q_L^T\bigr)^2 .
\]
Let $N_L \geq 1$ be such that $(Q_L/r_L)^T \leq 1/2$ for all $T \geq N_L$, which is possible because $Q_L < r_L$. For $T \geq N_L$, $(r_L^T - Q_L^T)^2 \geq r_L^{2T}/4$, and (2) follows. All constants depend only on $\rho$ and on the law of $Z$.
\end{proof}

\begin{proposition}[Depth-uniform confinement of small sublevel sets] \label{prop:confinement}
Suppose Assumption~\ref{assump:support_m} holds, and fix $\rho \in (0,1)$ and $C>0$. Let $\delta_L,\delta_R$ be the constants of Lemma~\ref{lem:boundary_barriers} for this $\rho$, and set
\[
K_0 := \bigl[-2/\overline{Z} + \delta_L,\, -\delta_R\bigr].
\]
Then $K_0$ is a nonempty compact subinterval of $(-2/\overline{Z},0)$, and there exists an integer $T_C \geq 2$, depending only on $\rho$, $C$ and the law of $Z$, such that for every $T \geq T_C$,
\begin{equation} \label{eq:sublevel-confinement}
\bigl\{\beta \in [-2/\overline{Z},0] : F_T(\beta) \leq C\rho^{2T}\bigr\} \subset K_0 .
\end{equation}
\end{proposition}

\begin{proof}
Since $\delta_L + \delta_R \leq 1/(2\overline{Z}) < 2/\overline{Z}$, the interval $K_0$ is nonempty, and it lies inside $(-2/\overline{Z},0)$ because $\delta_L,\delta_R>0$. As $r_L, r_R > \rho$ and $c_L, c_R>0$, there is an integer $T_C \geq \max\{2, N_L\}$ such that $C\rho^{2T} < \min\{c_L r_L^{2T}, c_R r_R^{2T}\}$ for all $T \geq T_C$, and for instance any integer $T_C \geq \max\{2,N_L\}$ with $T_C > [\log(C/c_j)]_+/(2\log(r_j/\rho))$ for $j\in\{L,R\}$ works, where $[v]_+ := \max\{v,0\}$. For such $T$, Lemma~\ref{lem:boundary_barriers} gives $F_T(\beta) > C\rho^{2T}$ on $[-2/\overline{Z}, -2/\overline{Z}+\delta_L] \cup [-\delta_R, 0]$, which is the complement of the interior of $K_0$ in $[-2/\overline{Z},0]$. This proves \eqref{eq:sublevel-confinement}, and in fact the sublevel set lies in the interior of $K_0$.
\end{proof}

\begin{lemma}[Initialization below the boundary barriers] \label{lem:loss_ineq}
    Suppose Assumption \ref{assump:support_m} holds, and that one initializes gradient flow at $\alpha_0 = \alpha^\ast(\beta_0)$ with $\beta_0 \in (-2/\overline{Z}, 0)$. Set $\rho_0 := \phi(\beta_0) \in (0,1)$, with $\phi$ as in \eqref{eq:phi-def}, and $C_0 := \ee[W]/(1-\rho_0)^2$. Then, for every integer $T\geq1$,
    \begin{equation} \label{eq:initial-loss}
    \ell(\alpha_0,\beta_0) = F_T(\beta_0) \leq \frac{\ee[W]\,\rho_0^{2T}}{(1-\rho_0^T)^2} \leq C_0\,\rho_0^{2T}.
    \end{equation}
    Let $K_0 = K_0(\beta_0)$ and $T' = T'(\beta_0) := T_{C_0} \geq 2$ be the compact interval and the threshold provided by Proposition~\ref{prop:confinement} for $(\rho, C) = (\rho_0, C_0)$, and both depend on $\beta_0$ and on the law of $Z$, but not on the depth. Then, for all $T \geq T'$,
    \[
    \ell(\alpha_0, \beta_0) < \min \{ \ell(\alpha, -2/\overline{Z}), \ell(\alpha, 0) \}
    \]
    for all $\alpha \in \rr$, and, since gradient flow is non-increasing on the loss $\ell$, one has $\beta_t \in K_0 \subset (-2/\overline{Z}, 0)$ for every $t$ in the interval of existence of the trajectory.
\end{lemma}

\begin{proof}
    Since $\alpha_0 = \alpha^\ast(\beta_0)$, we have $\ell(\alpha_0,\beta_0) = F_T(\beta_0)$. By \eqref{eq:uT-phi-bound}, $|u_T(\beta_0,Z)| \leq \rho_0^T$ almost surely, so $\var_\mu(u_T(\beta_0,Z)) \leq \ee_\mu[u_T(\beta_0,Z)^2] \leq \rho_0^{2T}$, while $\ee_\mu[(1-u_T(\beta_0,Z))^2] \geq (1-\rho_0^T)^2 \geq (1-\rho_0)^2$. Inserting these bounds into \eqref{eq:FT-variance-form} proves \eqref{eq:initial-loss}.

    Now let $T \geq T'$. Recall from the proof of Lemma~\ref{lem:alpha_t_bounded} that $\ell(\alpha,\beta) = A(\beta)(\alpha-\alpha^\ast(\beta))^2 + F_T(\beta) \geq F_T(\beta)$ for all $\alpha \in \rr$. The two boundary points $-2/\overline{Z}$ and $0$ lie outside $K_0$, so Proposition~\ref{prop:confinement} gives $F_T(-2/\overline{Z}) > C_0\rho_0^{2T}$ and $F_T(0) > C_0\rho_0^{2T}$, and the strict inequality follows from \eqref{eq:initial-loss}. For the last claim, note that along the trajectory, for every $t \geq 0$ at which it is defined,
    \[
    F_T(\beta_t) \leq \ell(\alpha_t,\beta_t) \leq \ell(\alpha_0,\beta_0) \leq C_0\rho_0^{2T},
    \]
    so \eqref{eq:sublevel-confinement} places $\beta_t$ in $K_0$ whenever $\beta_t \in [-2/\overline{Z},0]$. In particular $\beta_0 \in K_0$. Since $t \mapsto \beta_t$ is continuous, leaving $K_0$ would require first visiting $[-2/\overline{Z},0]\setminus K_0$, which is impossible. Hence $\beta_t \in K_0$ for every $t$ in the interval of existence of the trajectory.
\end{proof}

\begin{theorem} \label{thm:two_param_GF_converges}
     Suppose Assumption \ref{assump:support_m} holds, and that one initializes gradient flow at $\alpha_0 = \alpha^\ast(\beta_0)$ with $\beta_0 \in (-2/ \overline{Z}, 0)$. Let $T'=T'(\beta_0)\geq2$ and $K_0 = K_0(\beta_0) \subset (-2/\overline{Z},0)$ be the threshold and the compact interval from Lemma~\ref{lem:loss_ineq}. Then, for every integer depth $T\geq T'$, the gradient flow trajectory $(\alpha_{t,T},\beta_{t,T})_{t\geq0}$ exists for all $t \geq 0$ and converges, with
    \[
    \beta^\ast_T := \lim_{t \to \infty} \beta_{t,T} \in K_0
    \]
    and
    \[
    \alpha^\ast_T := \lim_{t \to \infty} \alpha_{t,T} = \alpha^\ast(\beta_T^\ast).
    \]
    Moreover, $F_T(\beta^\ast_T) \leq \ell(\alpha_0,\beta_0) \leq C_0\rho_0^{2T}$, with $C_0,\rho_0$ as in Lemma~\ref{lem:loss_ineq}.
\end{theorem}

\begin{proof}
    Fix $T\geq T'(\beta_0)$ as in the statement. The gradient $\nabla\ell$ is a polynomial in $(\alpha,\beta)$, hence locally Lipschitz, so the gradient flow has a unique maximal solution. On its interval of existence, $\ell$ is non-increasing along the flow, so Lemma~\ref{lem:loss_ineq} gives $\beta_t \in K_0$, and Lemma \ref{lem:alpha_t_bounded} gives $|\alpha_t| \leq M_{\alpha_0,\beta_0}$. The trajectory therefore remains in the compact rectangle $[-M_{\alpha_0,\beta_0}, M_{\alpha_0,\beta_0}] \times K_0$, on which the vector field is bounded, and the standard continuation criterion for ordinary differential equations shows that the solution exists for all $t \geq 0$. In particular, the gradient flow trajectory $(\alpha_t, \beta_t)_{t \geq 0}$ is bounded. Since $\ell$ is real-analytic on the open set $\rr\times(-2/\overline{Z},0)$, the bounded trajectory $(\alpha_t,\beta_t)_{t\ge 0}$ converges to a single critical point by the Łojasiewicz gradient inequality. At a critical point, $\partial_\alpha \ell = 2A(\beta)\alpha - 2B(\beta) = 0$, so $\alpha^\ast_T = \alpha^\ast(\beta^\ast_T)$, and since $\beta_t \in K_0$ for all $t\geq0$ and $K_0$ is closed, $\beta^\ast_T \in K_0$. Finally, by continuity of $\ell$ and its monotonicity along the flow, $F_T(\beta^\ast_T) = \ell(\alpha^\ast_T,\beta^\ast_T) = \lim_{t\to\infty}\ell(\alpha_t,\beta_t) \leq \ell(\alpha_0,\beta_0)$.
\end{proof}

\subsection{Asymptotic depth behavior}

Fix $\beta_0\in(-2/\overline Z,0)$ independently of depth, with $\alpha_0=\alpha^\ast(\beta_0)$ at each depth, and let $T'=T'(\beta_0)$ be the threshold above. Having established that gradient flow converges to a stationary point at each fixed depth $T\geq T'$, we now study how these stationary points behave as $T\to\infty$. Our goal is to understand the asymptotic depth behavior of the limits of gradient flow and to show that, as in the one-parameter case, the limiting predictor is Bayes optimal. To make explicit the dependence on depth, let $\ell^{(T)}$ denote the loss $\ell$ where the CA model is at depth $T$.

In the two-parameter setting, the loss is not jointly convex, so we cannot identify the gradient-flow limit $(\alpha_T^\ast,\beta_T^\ast)$ as a global minimizer of $\ell$. However, for each $T\geq T'$, what is guaranteed by Theorem~\ref{thm:two_param_GF_converges} is that
\[
 (\alpha_T^\ast,\beta_T^\ast)\ \text{is a stationary point of } \ell^{(T)},
 \qquad
 \alpha_T^\ast = \alpha^\ast(\beta_T^\ast),
 \qquad\text{and}\qquad
 \beta_T^\ast \in K_0,
\]
where the compact interval $K_0 = K_0(\beta_0) \subset (-2/\overline{Z},0)$ does not depend on $T$, so in particular
\[
  F'_T(\beta_T^\ast) = 0.
\]
Thus, the large-depth behavior of the gradient-flow limits is governed by the stationary points of the one-dimensional function $\beta\mapsto F_T(\beta)$. It will be convenient to work with a logarithmically rescaled version of $F_T$:
\begin{equation} \label{eq:G_T}
  G_T(\beta) := \frac{1}{2T}\log \left( \frac{F_T(\beta)}{\ee[W]} \right).
\end{equation}
The function $G_T$ is well-defined and smooth on any interval where $F_T > 0$, and Lemma~\ref{lem:endpoint-localization} below shows that this is the case, for all sufficiently large $T$, on every compact interval $K \subset I \setminus\{-\alpha^\ast\}$, where $I := (-2/\overline{Z},0)$ and $\alpha^\ast = 2/(\underline{Z}+\overline{Z})$. On such an interval, $G_T$ is a strictly monotone transform of $F_T$, so $F'_T(\beta)=0$ if and only if $G'_T(\beta)=0$. Our task is therefore to understand the asymptotic behavior of $G_T$ and $G'_T$ as $T\to\infty$, and to show that stationary points of $F_T$ (and $G_T$) concentrate near a single value of $\beta$.

\begin{lemma}[Uniform endpoint localization]\label{lem:endpoint-localization}
  Suppose Assumption~\ref{assump:support_m} holds, and fix any compact interval $K \subset I \setminus \{-\alpha^\ast\}$, where $I := (-2/\overline{Z},0)$ and $\alpha^\ast := 2/(\underline{Z}+\overline{Z})$. Let $z^\ast\in\{\underline{Z},\overline{Z}\}$ denote the endpoint at which $\phi(\beta)=|1+\beta z^\ast|$ is attained, which is the same for every $\beta \in K$. Write $q_\beta(z) := 1+\beta z$ and, with all expectations under $\mu$,
  \begin{gather*}
  m_T := \ee_\mu[u_T(\beta,Z)],\qquad Q_T := \ee_\mu[u_T(\beta,Z)^2],\\
  V_T := Q_T - m_T^2,\qquad D_T := \ee_\mu\bigl[(1-u_T(\beta,Z))^2\bigr],
  \end{gather*}
  so that $F_T = \ee[W]\,V_T/D_T$ by \eqref{eq:FT-variance-form}, where $Q_T(\beta)>0$, since $q_\beta(Z)=0$ $\mu$-almost surely would force $Z=-1/\beta$ almost surely, contradicting $\underline{Z}<\overline{Z}$. For each $T\in\nn$ and $\beta\in K$ define the probability measure
  \[
    \d\nu_{T,\beta}(z)
    := \frac{|1+\beta z|^{2T}}
             {Q_T(\beta)}\,\d\mu(z).
  \]
  Then the following hold uniformly for $\beta\in K$ as $T\to\infty$:
  \begin{enumerate}
    \item $\nu_{T,\beta}$ converges weakly to the point mass at the active endpoint $z^\ast$,
    \item the localization functional
      \[
      \cL_T(\beta) := \frac{\ee_\mu\bigl[(Z-z^\ast)^2\,|q_\beta(Z)|^{2T-2}\bigr]}{Q_T(\beta)}
      \]
      satisfies $\sup_{\beta\in K}\cL_T(\beta)\longrightarrow 0$,
    \item there exist $\kappa_K\in(0,1)$ and $N_K<\infty$ such that $V_T(\beta)\geq\kappa_K Q_T(\beta)>0$ for all $\beta\in K$ and $T\geq N_K$, so in particular $F_T>0$ on $K$ for all $T \geq N_K$, and
    \item the reduced loss satisfies $F_T(\beta)^{1/(2T)}\longrightarrow\phi(\beta)$.
  \end{enumerate}
\end{lemma}

\begin{proof}
Write $a:=\underline Z$, $b:=\overline Z$. Since $K$ is an interval that omits $-\alpha^\ast$, the active endpoint is constant on $K$, with $z^\ast=a$ when $\beta>-\alpha^\ast$ and $z^\ast=b$ when $\beta<-\alpha^\ast$. In the first case $q_\beta(a)>0$, and in the second $q_\beta(b)<0$. By compactness,
\[
s_+:=\max_{\beta\in K}|\beta|<\infty,
\qquad
0<\phi_-:=\min_{\beta\in K}\phi(\beta)\leq\phi(\beta)\leq\rho_K:=\max_{\beta\in K}\phi(\beta)<1 .
\]
All constants below may depend on $K$ and $\mu$, but not on $\beta\in K$ or on $T$.

\paragraph{Two basic estimates.}
By \eqref{eq:uT-phi-bound}, $|q_\beta(z)|\leq\phi(\beta)$ on $[a,b]$, hence
\begin{equation}\label{eq:AT-upper}
Q_T(\beta)\leq\phi(\beta)^{2T}.
\end{equation}
For the reverse direction, fix $\varepsilon\in(0,b-a)$. On the compact set $\{(\beta,z):\beta\in K,\ z\in[a,b],\ |z-z^\ast|\geq\varepsilon\}$ the continuous ratio $|q_\beta(z)|/\phi(\beta)$ is strictly less than one, because $z^\ast$ is the unique maximizer of $z \mapsto |q_\beta(z)|$ on $[a,b]$. Hence there is $r_\varepsilon\in[0,1)$ with
\begin{equation}\label{eq:endpoint-gap}
|q_\beta(z)|\leq r_\varepsilon\,\phi(\beta)
\qquad\text{whenever }\beta\in K,\ z\in[a,b],\ |z-z^\ast|\geq\varepsilon.
\end{equation}
Choose $\delta_\varepsilon\in(0,\varepsilon)$ so small that $\theta_\varepsilon:=1-s_+\delta_\varepsilon/\phi_-$ satisfies $r_\varepsilon<\theta_\varepsilon<1$. If $|z-z^\ast|<\delta_\varepsilon$, the reverse triangle inequality gives $|q_\beta(z)|\geq\phi(\beta)-s_+\delta_\varepsilon\geq\theta_\varepsilon\phi(\beta)$, using $\phi(\beta)\geq\phi_-$. Since $z^\ast$ is an essential endpoint of $Z$ under $\mu$, \eqref{eq:endpoint-mass} gives $p_\varepsilon:=\mu(|Z-z^\ast|<\delta_\varepsilon)>0$, and therefore
\begin{equation}\label{eq:AT-lower}
Q_T(\beta)\geq p_\varepsilon\,\theta_\varepsilon^{2T}\,\phi(\beta)^{2T}\qquad(\beta\in K,\ T\geq1).
\end{equation}

\paragraph{Concentration at the active endpoint.}
By \eqref{eq:endpoint-gap} and \eqref{eq:AT-lower},
\begin{equation}\label{eq:endpoint-tail}
\nu_{T,\beta}(|z-z^\ast|\geq\varepsilon)
\leq\frac{(r_\varepsilon\phi(\beta))^{2T}}{Q_T(\beta)}
\leq\frac{1}{p_\varepsilon}\left(\frac{r_\varepsilon}{\theta_\varepsilon}\right)^{2T}\longrightarrow0
\end{equation}
uniformly on $K$. For any continuous $f:[a,b]\to\rr$, let $\omega_f(\varepsilon):=\sup_{|z-s|\leq\varepsilon}|f(z)-f(s)|$. Then, by \eqref{eq:endpoint-tail},
\[
\sup_{\beta\in K}\left|\int f\,\d\nu_{T,\beta}-f(z^\ast)\right|
\leq\omega_f(\varepsilon)+\frac{2\|f\|_\infty}{p_\varepsilon}\left(\frac{r_\varepsilon}{\theta_\varepsilon}\right)^{2T}.
\]
First let $T\to\infty$ and then $\varepsilon\downarrow0$. Uniform continuity of $f$ proves the uniform weak convergence in part (1).

\paragraph{Localization of the second moment.}
Take $\varepsilon$ so small that, in addition, $s_+\varepsilon\leq\phi_-/2$. On $\{|Z-z^\ast|<\varepsilon\}$ we then have $|q_\beta(Z)|\geq\phi(\beta)-s_+\varepsilon\geq\phi_-/2$, hence
\[
\ee_\mu\bigl[(Z-z^\ast)^2|q_\beta(Z)|^{2T-2}\,\one_{\{|Z-z^\ast|<\varepsilon\}}\bigr]
\leq\frac{4\varepsilon^2}{\phi_-^2}\,\ee_\mu\bigl[|q_\beta(Z)|^{2T}\,\one_{\{|Z-z^\ast|<\varepsilon\}}\bigr]
\leq\frac{4\varepsilon^2}{\phi_-^2}\,Q_T(\beta).
\]
On the complementary event, \eqref{eq:endpoint-gap} and \eqref{eq:AT-lower} give, for $T\geq1$,
\[
\frac{\ee_\mu\bigl[(Z-z^\ast)^2|q_\beta(Z)|^{2T-2}\,\one_{\{|Z-z^\ast|\geq\varepsilon\}}\bigr]}{Q_T(\beta)}
\leq\frac{(b-a)^2\,(r_\varepsilon\phi(\beta))^{2T-2}}{p_\varepsilon\,\theta_\varepsilon^{2T}\,\phi(\beta)^{2T}}
\leq\frac{(b-a)^2}{p_\varepsilon\,\theta_\varepsilon^{2}\,\phi_-^{2}}\left(\frac{r_\varepsilon}{\theta_\varepsilon}\right)^{2T-2},
\]
which tends to zero uniformly on $K$ (if $r_\varepsilon=0$ the right-hand side is interpreted as zero). Consequently $\limsup_{T\to\infty}\sup_{\beta\in K}\cL_T(\beta)\leq4\varepsilon^2/\phi_-^2$, and letting $\varepsilon\downarrow0$ proves part (2).

\paragraph{Variance comparability.}
Let $\mathcal{A}:=\{Z\leq (a+b)/2\}$ if $z^\ast=b$ and $\mathcal{A}:=\{Z\geq(a+b)/2\}$ if $z^\ast=a$, so that $p:=\mu(\mathcal{A})>0$ by \eqref{eq:endpoint-mass}, and let $r:=r_{(b-a)/2}$, $\theta:=\theta_{(b-a)/2}$ and $p':=p_{(b-a)/2}$ be the constants of \eqref{eq:endpoint-gap} and \eqref{eq:AT-lower} for $\varepsilon = (b-a)/2$. On $\mathcal{A}$ we have $|Z-z^\ast|\geq(b-a)/2$, so $|u_T(\beta,Z)|\leq(r\phi(\beta))^T$ there, and \eqref{eq:AT-lower} yields
\[
\frac{|\ee_\mu[u_T(\beta,Z)\one_{\mathcal{A}}]|}{\sqrt{Q_T(\beta)}}
\leq\frac{(r\phi(\beta))^T}{\sqrt{p'}\,\theta^T\phi(\beta)^T}
=\frac{1}{\sqrt{p'}}\left(\frac r\theta\right)^T\longrightarrow0
\]
uniformly on $K$, while the Cauchy--Schwarz inequality gives
\[
|\ee_\mu[u_T(\beta,Z)\one_{\mathcal{A}^\comp}]|
\leq\sqrt{\mu(\mathcal{A}^\comp)}\,\sqrt{Q_T(\beta)}
=\sqrt{1-p}\,\sqrt{Q_T(\beta)}.
\]
Hence $\limsup_{T\to\infty}\sup_{\beta\in K}|m_T|/\sqrt{Q_T}\leq\sqrt{1-p}<1$. Fix $\gamma\in(\sqrt{1-p},1)$, and then there is $N_K$ such that $m_T^2\leq\gamma^2Q_T$ on $K$ for all $T\geq N_K$, so $V_T=Q_T-m_T^2\geq(1-\gamma^2)Q_T$, and part (3) holds with $\kappa_K:=1-\gamma^2$. Since $D_T \geq (1-\rho_K^T)^2 > 0$ by \eqref{eq:uT-phi-bound}, it follows that $F_T = \ee[W]V_T/D_T > 0$ on $K$ for $T \geq N_K$.

\paragraph{The reduced loss.}
By \eqref{eq:uT-phi-bound}, $|u_T|\leq\rho_K^T$ on $K\times[a,b]$, so $(1-\rho_K^T)^2\leq D_T\leq(1+\rho_K^T)^2$, and these bounds hold for both parities of $T$. For $T\geq N_K$, combining them with $\kappa_K Q_T\leq V_T\leq Q_T$ and with \eqref{eq:AT-upper}--\eqref{eq:AT-lower} gives
\[
\left(\frac{\ee[W]\,\kappa_K\,p_\varepsilon}{(1+\rho_K^T)^2}\right)^{1/(2T)}\theta_\varepsilon\,\phi(\beta)
\leq F_T(\beta)^{1/(2T)}
\leq\left(\frac{\ee[W]}{(1-\rho_K^T)^2}\right)^{1/(2T)}\phi(\beta).
\]
Both prefactors tend to one independently of $\beta$, so $\limsup_{T\to\infty}\sup_{\beta\in K}|F_T(\beta)^{1/(2T)}-\phi(\beta)|\leq(1-\theta_\varepsilon)\rho_K$. Since $\theta_\varepsilon\geq1-s_+\varepsilon/\phi_-$, letting $\varepsilon\downarrow0$ proves part (4).
\end{proof}

The above result will be used to prove the following lemma establishing the uniform convergence of the derivatives of $G_T$ over a compact interval omitting the point $-\alpha^\ast$. This will then yield a contradiction as the limiting derivative is strictly nonzero away from $-\alpha^\ast$. Recall from Appendix \ref{app:one-param}, we defined
\[
\alpha^\ast = \frac{2}{\underline{Z} + \overline{Z}}.
\]
\begin{lemma}[Asymptotics of the derivative] \label{lem:GT-derivative}
  Suppose Assumption~\ref{assump:support_m} holds. Fix a compact interval
  $K \subset I \setminus \{-\alpha^\ast\}$, where $I := (-2/\overline{Z},0)$, and let $N_K$ be as in Lemma~\ref{lem:endpoint-localization}, so that $G_T$ is well-defined on $K$ for $T \geq N_K$. Then:
  \begin{enumerate}
    \item
      The derivatives $G'_T(\beta)$ converge uniformly on $K$ to
      the derivative of $\log\phi(\beta)$:
      \[
        G'_T(\beta)
        \longrightarrow
        \frac{d}{d\beta}\log\phi(\beta)
      \]
      uniformly for $\beta\in K$ as $T\to\infty$.
    \item
      In particular, since $-\alpha^\ast\notin K$,
      there exist constants $T_K<\infty$ and $c_K>0$ such that
      \[
        |G'_T(\beta)| \geq c_K>0,   \qquad \forall\,\beta\in K,\ T\ge T_K.
      \]
  \end{enumerate}
\end{lemma}

\begin{proof}
Use the notation and the uniform estimates from Lemma~\ref{lem:endpoint-localization} and its proof, and let $T\geq N_K$, so that $V_T\geq\kappa_KQ_T>0$ and $D_T\geq(1-\rho_K^T)^2>0$ on $K$. For fixed $T$, $m_T$ and $Q_T$ are polynomials in $\beta$ with finite coefficients, since $Z$ is bounded. Differentiating these polynomials gives
\[
m_T'=T\int zq_\beta(z)^{T-1}\,\d\mu(z),\qquad
Q_T'=2T\int zq_\beta(z)^{2T-1}\,\d\mu(z).
\]
Set
\[
\lambda(\beta):=\frac{z^\ast}{q_\beta(z^\ast)}.
\]
Only the endpoint value appears in this denominator, and $|q_\beta(z^\ast)|=\phi(\beta)\geq\phi_->0$, so $\lambda$ is continuous and bounded on $K$. The identity
\[
z-\lambda(\beta)\,q_\beta(z)
=\frac{z-z^\ast}{q_\beta(z^\ast)}
\]
therefore gives the exact representations
\begin{gather*}
Q_T'-2T\lambda Q_T=\frac{2T}{q_\beta(z^\ast)}\int(z-z^\ast)q_\beta(z)^{2T-1}\,\d\mu(z),\\
m_T'-T\lambda m_T=\frac{T}{q_\beta(z^\ast)}\int(z-z^\ast)q_\beta(z)^{T-1}\,\d\mu(z).
\end{gather*}
By the Cauchy--Schwarz inequality and the definition of $\cL_T$,
\begin{equation}\label{eq:moment-derivative-local}
\begin{gathered}
\frac{|Q_T'-2T\lambda Q_T|}{2TQ_T}
\leq\frac{\ee_\mu\bigl[|Z-z^\ast|\,|q_\beta(Z)|^{T-1}\,|q_\beta(Z)|^{T}\bigr]}{\phi_-\,Q_T}
\leq\frac{\sqrt{\cL_T(\beta)}}{\phi_-},
\\
\frac{|m_T'-T\lambda m_T|}{T\sqrt{Q_T}}
\leq\frac{\ee_\mu\bigl[|Z-z^\ast|\,|q_\beta(Z)|^{T-1}\bigr]}{\phi_-\,\sqrt{Q_T}}
\leq\frac{\sqrt{\cL_T(\beta)}}{\phi_-},
\end{gathered}
\end{equation}
and both right-hand sides tend to zero uniformly on $K$ by Lemma~\ref{lem:endpoint-localization}(2). Neither bound requires $q_\beta$ to be nonvanishing on $[a,b]$.

Since $V_T=Q_T-m_T^2$ and $V_T'=Q_T'-2m_Tm_T'$, we have the exact decomposition
\[
V_T'-2T\lambda V_T
=\bigl(Q_T'-2T\lambda Q_T\bigr)-2m_T\bigl(m_T'-T\lambda m_T\bigr),
\]
and $|m_T|\leq\sqrt{Q_T}$. Dividing by $2TV_T\geq2T\kappa_KQ_T$ and using \eqref{eq:moment-derivative-local},
\begin{equation}\label{eq:VT-log-derivative-main}
\sup_{\beta\in K}\left|\frac{V_T'}{2TV_T}-\lambda(\beta)\right|
\leq\frac{2}{\kappa_K\phi_-}\sup_{\beta\in K}\sqrt{\cL_T(\beta)}\longrightarrow0.
\end{equation}
It remains to check that the denominator contributes no logarithmic derivative. Since $|q_\beta(z)|\leq\rho_K$ on $K\times[a,b]$, the derivative formulas above give the crude bounds $|m_T'|\leq Tb\rho_K^{T-1}$ and $|Q_T'|\leq2Tb\rho_K^{2T-1}$. Because $D_T'=-2m_T'+Q_T'$ and $D_T\geq(1-\rho_K^T)^2$,
\begin{equation}\label{eq:DT-log-derivative}
\sup_{\beta\in K}\left|\frac{D_T'}{2TD_T}\right|
\leq\frac{b\,(\rho_K^{T-1}+\rho_K^{2T-1})}{(1-\rho_K^T)^2}\longrightarrow0.
\end{equation}
By \eqref{eq:G_T} and \eqref{eq:FT-variance-form}, $G_T=\frac{1}{2T}\bigl(\log V_T-\log D_T\bigr)$, so
\[
G'_T=\frac{1}{2T}\frac{V_T'}{V_T}-\frac{1}{2T}\frac{D_T'}{D_T},
\]
and \eqref{eq:VT-log-derivative-main}--\eqref{eq:DT-log-derivative} give $G'_T\to\lambda$ uniformly on $K$. Finally, $z^\ast$ is constant on $K$ and $\phi(\beta)=|1+\beta z^\ast|$ there, so
\[
\frac{\d}{\d\beta}\log\phi(\beta)
=\frac{\d}{\d\beta}\log|1+\beta z^\ast|
=\lambda(\beta),
\]
which proves part (1). The function $\lambda$ is continuous and strictly positive when $z^\ast=a$, and strictly negative when $z^\ast=b$. Hence
$c_K:=\tfrac12\min_{\beta\in K}|\lambda(\beta)|>0$.
Uniform convergence supplies $T_K \geq N_K$ such that $|G'_T(\beta)-\lambda(\beta)|\leq c_K$ for all $\beta\in K$ and $T\geq T_K$, proving part (2).
\end{proof}

The large-depth conclusion of Theorem \ref{thm:two_param_optimal} will now follow by a proof by contradiction, since if the trained parameters $\beta^\ast_T$ did not converge to $-\alpha^\ast$, a subsequence would accumulate at a point of the depth-independent compact interval $K_0$ of Theorem~\ref{thm:two_param_GF_converges} different from $-\alpha^\ast$, where Lemma~\ref{lem:GT-derivative} forbids stationary points at large depth.

\begin{theorem} \label{thm:two_param_optimal_limit}
        Suppose the assumptions of Theorem \ref{thm:two_param_GF_converges} hold and, for $T\geq T'(\beta_0)$, let $(\alpha^\ast_T, \beta^\ast_T)$ denote the limit of gradient flow on the two-parameter, $T$-layer, CA model whose existence is guaranteed by Theorem \ref{thm:two_param_GF_converges}. Then,
    \[
    \lim_{T \to \infty} \alpha^\ast_T = \alpha^\ast = \frac{2}{\underline{Z} + \overline{Z}}
    \]
    and
    \[
     \lim_{T \to \infty} \beta^\ast_T = - \alpha^\ast.
    \]
     The resulting predictor is Bayes optimal in the context-then-depth limit.
\end{theorem}

\begin{proof}
    By Theorem~\ref{thm:two_param_GF_converges}, $\beta^\ast_T \in K_0$ for every $T \geq T'$, where $K_0 = K_0(\beta_0) \subset (-2/\overline{Z},0)$ is compact and does not depend on $T$, and $F'_T(\beta^\ast_T) = 0$ since $(\alpha^\ast_T,\beta^\ast_T)$ is a stationary point of $\ell^{(T)}$ with $\alpha^\ast_T = \alpha^\ast(\beta^\ast_T)$. We assume for contradiction that $\beta^\ast_T \not\to -\alpha^\ast$. Then there exist $\varepsilon>0$ and depths $T_k \to \infty$ with $|\beta^\ast_{T_k} + \alpha^\ast| \geq \varepsilon$ for all $k$. By compactness of $K_0$, passing to a further subsequence we may assume that
    \[
    \bar \beta := \lim_{k \to \infty} \beta^\ast_{T_k} \in K_0 \subset (-2/\overline{Z}, 0), \qquad |\bar\beta + \alpha^\ast| \geq \varepsilon .
    \]
    Choose a compact interval
    \[
      K \subset (-2/\overline{Z},0)\setminus\{-\alpha^\ast\}
    \]
    such that $\bar\beta \in \interior(K)$. Then $\beta^\ast_{T_k}\in K$ for all sufficiently large $k$. By Lemma~\ref{lem:endpoint-localization}(3), $F_{T_k} > 0$ on $K$ once $T_k \geq N_K$, so $G_{T_k}$ is well-defined at $\beta^\ast_{T_k}$ and stationarity gives
    \[
    G'_{T_k}(\beta^\ast_{T_k}) = \frac{F'_{T_k}(\beta^\ast_{T_k})}{2T_k\,F_{T_k}(\beta^\ast_{T_k})} = 0 .
    \]
    By Lemma \ref{lem:GT-derivative}, however, $|G'_{T_k}(\beta^\ast_{T_k})| \geq c_K > 0$ for all sufficiently large $k$---a contradiction. Therefore $\beta^\ast_T \to -\alpha^\ast$:
     \begin{equation}\label{eq:beta_limit}
     \lim_{T \to \infty} \beta^\ast_T = - \alpha^\ast.
    \end{equation}
    The existence of the limit \eqref{eq:beta_limit} is thus a consequence of the depth-uniform confinement $\beta^\ast_T \in K_0$ established in Theorem~\ref{thm:two_param_GF_converges}, rather than a hypothesis.
    To account for the changing depth, we restore the depth in the notation of \eqref{eq:unique_beta_min}, writing $S_T(Z,\beta)$, $A_T(\beta) := \ee[WS_T(Z,\beta)^2]$ and $B_T(\beta) := \ee[WS_T(Z,\beta)]$ for $S$, $A$ and $B$ at depth $T$, so that the profiled minimizer is $\alpha_T^\ast(\beta)=B_T(\beta)/A_T(\beta)$. Since $S_T(Z,\beta)=-(1-u_T(\beta,Z))/\beta$, reweighting by $\d\mu=W\,\d\P/\ee[W]$ gives, for $\beta\neq0$,
    \[
    \alpha_T^\ast(\beta)
    =-\beta\,\frac{\ee_\mu[1-u_T(\beta,Z)]}
                        {\ee_\mu[(1-u_T(\beta,Z))^2]}.
    \]
    Recall from Theorem~\ref{thm:two_param_GF_converges} that $\beta_T^\ast\in K_0$ for all $T\geq T'$, where $K_0\subset(-2/\overline Z,0)$ is compact and independent of the depth. Since $\underline Z\geq1$, compactness gives
    \[
    \widetilde\rho:=\max_{\beta\in K_0,\,\sigma\in[1,\overline Z]}|1+\beta\sigma|<1,
    \qquad
    \sup_{\beta\in K_0,\,z\in[\underline Z,\overline Z]}|u_T(\beta,z)|\leq\widetilde\rho^{\,T}.
    \]
    Writing $N_T(\beta)=\ee_\mu[1-u_T(\beta,Z)]$ and $D_T(\beta)=\ee_\mu[(1-u_T(\beta,Z))^2]$, uniformly on $K_0$ we have
    \[
    |N_T-1|\leq\widetilde\rho^{\,T},\qquad
    |D_T-1|\leq2\widetilde\rho^{\,T}+\widetilde\rho^{\,2T},\qquad
    D_T\geq(1-\widetilde\rho^{\,T})^2.
    \]
    Thus, using $|D_T-N_T|\leq\widetilde\rho^{\,T}+\widetilde\rho^{\,2T}$,
    \[
    \sup_{\beta\in K_0}|\alpha_T^\ast(\beta)+\beta|
    \leq\max_{\beta\in K_0}|\beta|\,
       \frac{\widetilde\rho^{\,T}+\widetilde\rho^{\,2T}}{(1-\widetilde\rho^{\,T})^2}\longrightarrow0.
    \]
    Evaluating at $\beta_T^\ast$ yields $\alpha_T^\ast=\alpha_T^\ast(\beta_T^\ast)=-\beta_T^\ast+o(1)\to\alpha^\ast$, and therefore $\alpha_T^\ast/\beta_T^\ast\to-1$.

    It remains to verify prediction convergence for this changing pair of parameters. At each fixed $T\geq T'$, Lemma~\ref{lem:CA_output} and the covariance law of large numbers give
    \[
    H_T:=\lim_{L\to\infty}\frac{1}{L}\bX\bF_T^\top
    =\frac{\alpha_T^\ast}{\beta_T^\ast}
       \left[(\bI+\beta_T^\ast\bLam)^T-\bI\right]
    \]
    almost surely. Every eigenvalue of $\bLam=\bI+\bfm\bfm^\top$ lies in $[1,\overline Z]$, including the eigenvalue $1$ on the orthogonal complement of $\bfm$. The same $\widetilde\rho$ therefore gives $\|(\bI+\beta_T^\ast\bLam)^T\|\leq\widetilde\rho^{\,T}$ for all $T\geq T'$, so
    \[
    \|H_T-\bI\|
    \leq\left|1+\frac{\alpha_T^\ast}{\beta_T^\ast}\right|
        +\left|\frac{\alpha_T^\ast}{\beta_T^\ast}\right|\widetilde\rho^{\,T}
    \longrightarrow0.
    \]
    Lemma~\ref{lem:readout} gives $\lim_{L\to\infty}\hat y_q=\bw^\top H_T\bx_q$ on a common probability-one event for the countably many depths. Its difference from $\bw^\top\bx_q$ is bounded by $\|\bw\|\|\bx_q\|\|H_T-\bI\|$ and hence tends to zero. Applied to the test prompt, this proves the claimed context-then-depth Bayes-prediction limit.
\end{proof}

Theorems~\ref{thm:two_param_GF_converges} and~\ref{thm:two_param_optimal_limit} together prove Theorem \ref{thm:two_param_optimal}, with $T'(\beta_0)$ the depth threshold of Lemma~\ref{lem:loss_ineq}. Of Assumption~\ref{assump:support_m}, the arguments of this appendix use only that $Z>1$ almost surely, that $1 \leq \underline{Z} < \overline{Z} < \infty$, and that every neighborhood of each essential endpoint $\underline{Z}$, $\overline{Z}$ carries positive mass, so in particular no density of $Z$ is required.

\section{Non-asymptotic guarantees}\label{app:finite-context}

Theorems~\ref{thm:one_param_optimal} and~\ref{thm:two_param_optimal} send the test-context length to infinity at each fixed depth. This appendix records what survives at finite context length $L$ and finite depth $T$ for the one-parameter model of Section~\ref{sec:CA_model}, at a parameter $\alpha$ fixed independently of the test prompt, where the excess risk splits exactly into a bias that contracts geometrically in the depth and a noise term of order $d/L$ from the finite context. These are prediction guarantees for an already-chosen parameter, not statements about training at finite context or from finitely many prompts.

We use the notation of the proof of Lemma~\ref{lem:one_param_loss_properties}, namely $\widehat\bLam=\bX\bX^\top/L$, $\bfM=\bI-\alpha\widehat\bLam$ and $\bQ_T=\sum_{k=0}^{T-1}\bfM^k$, so that $\bF_T=\alpha\bQ_T\bX$ and $\alpha\widehat\bLam\bQ_T=\bI-\bfM^T$. The error is measured against the task-conditional oracle $\ip{\bw}{\bx_q}$ of \eqref{eq:regression_view} rather than against $y_q$, so that the irreducible noise $\sigma_\tau^2=\rc^2/Z$ does not appear
\begin{equation}\label{eq:fc-excess}
\cE_{L,T}(\alpha;\tau,\bX):=\ee\!\left[(\hat y_q-\ip{\bw}{\bx_q})^2\mid\tau,\bX\right],
\end{equation}
where the expectation averages over the context residuals and the query with the task $\tau=(\bfm,\rc)$ and the context covariates $\bX$ held fixed. Conditionally on the task, write $\bX=\bLam^{1/2}\boldsymbol{G}$, where $\boldsymbol{G}\in\rr^{d\times L}$ has independent standard Gaussian entries and is independent of the task, and for $\delta\in(0,1)$ set
\begin{equation}\label{eq:fc-event}
\mathsf{S}_\delta:=\left\{(1-\delta)^2\bI\preceq\frac{\boldsymbol{G}\boldsymbol{G}^\top}{L}\preceq(1+\delta)^2\bI\right\}.
\end{equation}

\begin{theorem}[Finite-context excess risk]\label{thm:fc-main}
Suppose Assumption~\ref{assump:support_m} holds.
\begin{enumerate}
\item For every real $\alpha$ and every integer $T\geq1$,
\begin{equation}\label{eq:fc-risk-identity}
\cE_{L,T}(\alpha;\tau,\bX)
=\bw^\top\bfM^T\bLam\bfM^T\bw
+\frac{\sigma_\tau^2\alpha^2}{L}\tr\!\left(\bLam\bQ_T\widehat\bLam\bQ_T\right).
\end{equation}
\item Let $L>d$, let $\delta\in(\sqrt{d/L},1)$, and let
\begin{equation}\label{eq:fc-window}
0<\alpha<\frac{2}{(1+\delta)^2\overline{Z}},
\qquad
\rho_\delta(\alpha):=\max\left\{\bigl|1-\alpha(1-\delta)^2\bigr|,\ \bigl|1-\alpha(1+\delta)^2\overline{Z}\bigr|\right\}<1.
\end{equation}
Then $\P(\mathsf{S}_\delta^\comp)\leq2\exp[-(\delta\sqrt L-\sqrt d)^2/2]$, and on $\mathsf{S}_\delta$
\begin{equation}\label{eq:fc-main-bound}
\cE_{L,T}(\alpha;\tau,\bX)
\leq\rc^2\left[
W\rho_\delta(\alpha)^{2T}
+\frac{(d+Z-1)\bigl(1+\rho_\delta(\alpha)^T\bigr)^2}{ZL(1-\delta)^2}
\right]
\end{equation}
simultaneously for every integer $T\geq1$ and every $\alpha$ satisfying \eqref{eq:fc-window},
\rev{where $W:=(Z-1)/Z$}.
\end{enumerate}
\end{theorem}

\begin{proof}
(1) This is the computation in the proof of Lemma~\ref{lem:one_param_loss_properties}, which gives
\[
\ip{\bw}{\bx_q}-\hat y_q=\bw^\top\bfM^T\bx_q-\frac{\alpha}{L}\beps^\top\bX^\top\bQ_T\bx_q .
\]
Conditionally on $(\tau,\bX)$ the query is centered with covariance $\bLam$, the residual vector $\beps$ is centered with covariance $\sigma_\tau^2\bI_L$, and the two are independent, so the cross term has zero conditional expectation, while the two squares contribute the two terms of \eqref{eq:fc-risk-identity}. Compared with the display in that proof, only the term $\sigma_\tau^2$ is absent, because the error is measured against $\ip{\bw}{\bx_q}$ rather than $y_q$. No inverse is used, so the identity holds when $\widehat\bLam$ is singular and when $\alpha=0$.

(2) The Gaussian singular-value bound \citep[Corollary 5.35]{vershynin2012}, applied to $\boldsymbol{G}^\top\in\rr^{L\times d}$ with $t=\delta\sqrt L-\sqrt d>0$, gives $(1-\delta)\sqrt L\leq s_{\min}(\boldsymbol{G}^\top)\leq s_{\max}(\boldsymbol{G}^\top)\leq(1+\delta)\sqrt L$ with probability at least $1-2e^{-t^2/2}$, which is the stated bound on $\P(\mathsf{S}_\delta^\comp)$. Now work on $\mathsf{S}_\delta$. Conjugating \eqref{eq:fc-event} by $\bLam^{1/2}$ gives $(1-\delta)^2\bLam\preceq\widehat\bLam\preceq(1+\delta)^2\bLam$, and $\bI\preceq\bLam\preceq\overline{Z}\bI$ places every eigenvalue of $\widehat\bLam$ in $[(1-\delta)^2,(1+\delta)^2\overline{Z}]$. Since $\lambda\mapsto|1-\alpha\lambda|$ is convex, its maximum over that interval is attained at an endpoint, so $\|\bfM\|\leq\rho_\delta(\alpha)<1$. With $\|\bLam\|=Z$ and $\|\bw\|^2=\rc^2(Z-1)/Z^2$, the bias term of \eqref{eq:fc-risk-identity} is at most $Z\rho_\delta^{2T}\rc^2(Z-1)/Z^2=\rc^2W\rho_\delta^{2T}$. For the noise term, $\widehat\bLam$ is positive definite on $\mathsf{S}_\delta$, and since all the matrices involved are polynomials in $\widehat\bLam$, the identity $\alpha\widehat\bLam\bQ_T=\bI-\bfM^T$ gives
\[
\boldsymbol{R}_T:=\alpha^2\bQ_T\widehat\bLam\bQ_T=(\bI-\bfM^T)^2\widehat\bLam^{-1}\succeq0,
\qquad
\|\boldsymbol{R}_T\|\leq\frac{(1+\rho_\delta^T)^2}{(1-\delta)^2},
\]
the norm bound using $\|\bfM\|\leq\rho_\delta$ and the lower bound $(1-\delta)^2$ on the spectrum of $\widehat\bLam$. As $\bLam=\bI+\bfm\bfm^\top$, we have $\tr(\bLam\boldsymbol{R}_T)=\tr(\boldsymbol{R}_T)+\bfm^\top\boldsymbol{R}_T\bfm\leq(d+Z-1)\|\boldsymbol{R}_T\|$. Substituting this and $\sigma_\tau^2=\rc^2/Z$ into \eqref{eq:fc-risk-identity} proves \eqref{eq:fc-main-bound}. The event $\mathsf{S}_\delta$ is a property of $\boldsymbol{G}$ alone and involves neither $T$ nor $\alpha$, so the same realization serves every depth and every admissible parameter at once.
\end{proof}

\paragraph{Interpretation of the bound.}
The first term of \eqref{eq:fc-main-bound} is the finite-depth bias and is geometric in $T$, and the second is the finite-context noise term and is of order $d/L$ once $\rho_\delta^T$ is small. The contraction factor $\rho_\delta$ controls the sample spectrum and tends to $\max\{|1-\alpha|,|1-\alpha\overline{Z}|\}$ as $\delta\downarrow0$, but it need not match the faster rate along the task direction that governs the infinite-context population loss of Appendix~\ref{app:one-param}. The population-trained parameter is admissible, since by Theorem~\ref{thm:one_param_optimal_limit}, $\alpha^\ast_T\to\alpha^\ast=2/(\underline{Z}+\overline{Z})$, and choosing $\delta>0$ small enough that $(1+\delta)^2\overline{Z}<\underline{Z}+\overline{Z}$ places $\alpha^\ast$ strictly inside the window \eqref{eq:fc-window}, so the trained parameters are admissible for all sufficiently large depths.

\section{Figures \ref{fig:numerics_combined} and \ref{fig:less-restricted-diagnostics} experiment details}\label{app:more-less-restricted-experiments}
For Figure \ref{fig:numerics_combined}, models are trained using standard gradient descent (in place of gradient flow) on the non-asymptotic objective $\ell_{N, L_{\rm tr}}$ with $N=2000$ and $L_{\rm tr}=100$, and the protocol for Figure \ref{fig:less-restricted-diagnostics} is stated separately below. Data is generated according to Section \ref{sec:data} with $\rc \sim \mathcal{N}(0,1)$, $d_1=d_2 = 16$, and $\bfm$ drawn from a spherical Gaussian with norm mapped uniformly on $[0,5]$.

\paragraph{Further details for less-restricted LCA parametrizations (Figure \ref{fig:less-restricted-diagnostics}).} We use the synthetic data-generating model discussed above. Throughout these experiments, we keep $\bW_t^K=\bW_t^Q=\bI_d$ fixed and train only the layerwise injection and value matrices $(\bW_t^S,\bW_t^V)_{t=0}^{T-1}$. We consider three increasingly flexible parametrizations where in all cases the weights are untied across layers unlike in the one- and two-parameter simplifications on which we base our theoretical results in Section~\ref{sec:main_results}. In the untied scalar model,
\[
    \bW_t^S=\alpha_t \bI_d,\qquad \bW_t^V=\beta_t \bI_d,
\]
where each layer has its own learnable pair $(\alpha_t,\beta_t)$. In the diagonal model,
\[
    \bW_t^S=\operatorname{diag}(\ba_t),\qquad
    \bW_t^V=\operatorname{diag}(\bb_t),
\]
where $\ba_t,\bb_t\in\rr^d$. In the full-matrix model, $\bW_t^S,\bW_t^V\in\rr^{d\times d}$ are unconstrained trainable matrices.

All three models are initialized near the theory-predicted limiting parameters $(\alpha^\ast,-\alpha^\ast)$ from Theorems~\ref{thm:one_param_optimal}--\ref{thm:two_param_optimal}. Specifically, with initial noise variance $\sigma_{\rm init}=0.1$, for the untied scalar model we initialize
\[
    \alpha_t=\alpha^\ast+\sigma_{\rm init} g_t^S,\qquad
    \beta_t=-\alpha^\ast+\sigma_{\rm init} g_t^V,
\]
with independent standard Gaussian variables $g_t^S,g_t^V$. For the diagonal model, we initialize
\[
    \ba_t=\alpha^\ast\mathbf{1}_d+\sigma_{\rm init}\mathbf{g}_t^S,\qquad
    \bb_t=-\alpha^\ast\mathbf{1}_d+\sigma_{\rm init}\mathbf{g}_t^V,
\]
where $\mathbf{g}_t^S$ and $\mathbf{g}_t^V$ have i.i.d. standard Gaussian entries and $\mathbf{1}_d =[1 \cdots 1]^\top \in \rr^d$. For the full-matrix model, we initialize
\[
    \bW_t^{S}=\alpha^\ast \bI_d+\sigma_{\rm init}\mathbf{G}_t^S,\qquad
    \bW_t^{V}=-\alpha^\ast \bI_d+\sigma_{\rm init}\mathbf{G}_t^V,
\]
where the entries of $\mathbf{G}_t^S, \mathbf{G}_t^V$ are i.i.d. standard Gaussian. Thus, each model has initialization biased near the theoretically optimal tied-scalar solution obtained in the one- and two-parameter models. However, the noise level $\sigma_{\rm init}$ is large with respect to the magnitude $|\alpha^\ast|$ as the distribution on $\bfm$ yields $\alpha^\ast = 2/27 \approx 0.074$.

We train each model by gradient descent, but---unlike the protocol for Figure \ref{fig:numerics_combined}---on freshly sampled batches of $512$ prompts for $500$ steps rather than on a fixed pool of $N$ prompts, at training context length $L_{\rm tr}=500$, with $2000$ held-out prompts used for validation and $2000$ for the diagnostics reported below. After training, in Figure \ref{fig:less-restricted-diagnostics}, we display the scalar projections
\[
    \widehat{\alpha}_t=d^{-1}\operatorname{tr}(\bW_t^S),
    \qquad
    \widehat{\beta}_t=d^{-1}\operatorname{tr}(\bW_t^V),
\]
as well as the whitening diagnostic $\|\bX\bF_t^\top/L-\bI\|_F/\sqrt d$ following the heuristic laid out at the end of Section \ref{sec:main_results}. The points/curves in Figure~\ref{fig:less-restricted-diagnostics} are averaged over ten independent runs.

\rev{\section{Non-Gaussian data}\label{app:nongaussian}

Gaussianity in \eqref{eq:data_dist} is used for the linearity of the Bayes predictor. The infinite-context losses of Lemmas~\ref{lem:one_param_loss_properties} and~\ref{lem:two_param_loss_properties} involve the data only through second moments, so their minimizers do not depend on the law of $u$, $\bxi$ and $\bta$ beyond its mean and variance. We test this by drawing $u$ and every coordinate of $\bxi$ and $\bta$ from four laws with mean zero and unit variance, Rademacher, $\sqrt{3}\,{\rm Unif}(-1,1)$, Laplace and ${\rm Exp}(1)-1$.

Figure~\ref{fig:nongaussian} (\textbf{Left}, \textbf{Middle}) repeats the experiment of Figure~\ref{fig:numerics_combined} under each law, with its setup ($d_1=d_2=10$, $\|\bfm\|\sim{\rm Unif}(0,1)$ and $y_i=\ip{\bw}{\bx_i}$, so that $\ip{\bw}{\bx_q}$ is the Bayes predictor under every law) and protocol. Both LCA models track the Gaussian errors at every $L_{\rm te}$ and $T$ up to run-to-run variation. Figure~\ref{fig:nongaussian} (\textbf{Right}) compares trained parameters with the theory. As $L_{\rm tr}$ grows, the trained $\alpha$ approaches the infinite-context minimizer $\alpha^\ast_{10}$ under every law.

\begin{figure}[ht]
  \begin{center}
    \centerline{\includegraphics[width=1.0\textwidth]{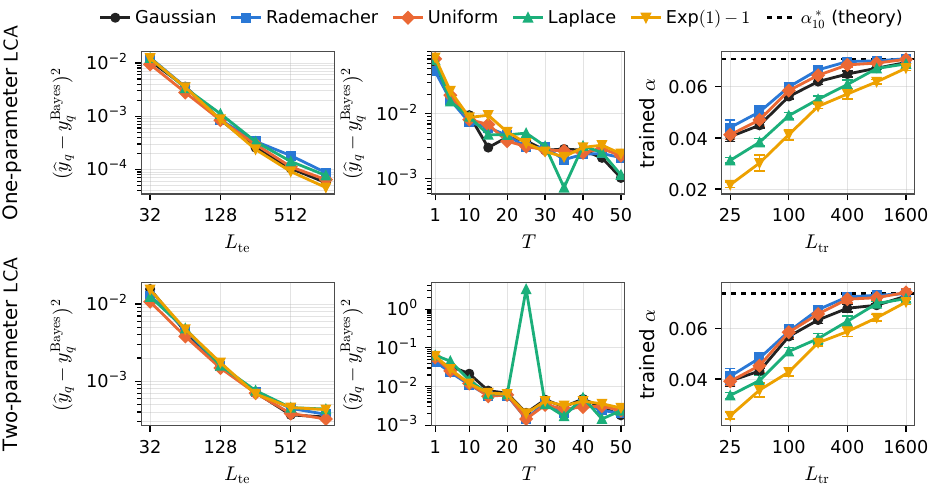}}
    \caption{Non-Gaussian latent and noise. (\textbf{Top}) one-parameter and (\textbf{Bottom}) two-parameter LCA. (\textbf{Left}) error against $L_{\rm te}$ at $T=10$ and (\textbf{Middle}) against $T$ at $L_{\rm te}=64$, with the setup and protocol of Figure~\ref{fig:numerics_combined} (one training run per point). (\textbf{Right}) trained $\alpha$ at $T=10$ against $L_{\rm tr}$ (mean and s.e., $5$ runs of full-batch Adam on $\ell_{N,L_{\rm tr}}$, $N=2000$, $d_1=d_2=16$, $\|\bfm\|\sim{\rm Unif}(0,5)$), with $\alpha^\ast_{10}$ dashed.}
    \label{fig:nongaussian}
  \end{center}
    \vskip -0.2in
\end{figure}
}

\section{Table \ref{tab:avmnist_main} experiment details}
\label{app:avmnist_details}

This appendix gives the implementation details for the corrupted avMNIST experiments in Table~\ref{tab:avmnist_main}. The goal is to test whether the two architectural mechanisms isolated by the theory---cross-attention from an evolving state back to a fixed covariate sequence, and repeated re-injection of that fixed sequence---remain useful after replacing the linearized model \eqref{eq:LSA_LCA_model} by conventional deep-learning components: convolutional token encoders, softmax multi-head attention, pre-layer normalization, MLP blocks, dropout, and a classification head.

Throughout this appendix, we use the same notation as Section~\ref{sec:CA_model} whenever possible.  In particular, $\bX$ denotes the fixed encoded multimodal token sequence for one example, $\bF_t$ denotes the evolving state after $t$ layers, $\bA_t$ denotes the attention update, and $\bS_t$ denotes the re-injection term.  The only notational difference is that tokens are written as rows.  Thus, in this appendix $\bX,\bF_t\in\rr^{L\times d}$, whereas the theoretical sections write tokens as columns.

\subsection{Dataset, normalization, and task}

We use a paired audio-visual MNIST dataset~\citep{avmnist}.  Each example is a triple
\[
    (I_i,A_i,c_i),
    \qquad
    I_i\in\rr^{1\times 28\times 28},\quad
    A_i\in\rr^{1\times 28\times 28},\quad
    c_i\in\{0,1,\ldots,9\},
\]
where $I_i$ is an image, $A_i$ is a paired audio spectrogram, and $c_i$ is the digit class.  The experiments use the standard split stored in the repository: $60{,}000$ training examples and $10{,}000$ test examples.  The image and audio arrays have spatial resolution $28\times 28$; a singleton channel dimension is added before passing them to the convolutional encoders. The data $ (I_i,A_i)$ is corrupted entry-wise by Gaussian noise to enforce difficulty in the learning task. For each modality separately, we add i.i.d.\ $\mathcal{N}(0,\sigma^2)$ noise with $\sigma = 5\hat{s}$, where $\hat{s}$ is the empirical standard deviation of that modality over the dataset, and clip the result back to the original value range. This noise is large relative to the signal it corrupts, which is what prevents the task from saturating and accounts for the accuracies reported in Table~\ref{tab:avmnist_main}. Thereafter, the data is normalized to $[0,1]$.

Each model outputs logits $z_i\in\rr^{10}$ and is trained with the cross-entropy loss
\[
    \ell_{\rm CE}(z_i,c_i)=-\log\left(\frac{\exp(z_{i,c_i})}{\sum_{k=0}^{9}\exp(z_{i,k})}\right).
\]

\subsection{Common convolutional tokenization}

For token dimension $d$, each modality encoder maps a $28\times 28$ input to $16$ tokens in $\rr^d$:
\[
    \Phi_I:\rr^{1\times 28\times 28}\to\rr^{16\times d},
    \qquad
    \Phi_A:\rr^{1\times 28\times 28}\to\rr^{16\times d}.
\]
The two encoders have the same architecture but do not share weights.  Each encoder consists of two convolution--ReLU layers, a $2\times 2$ max-pooling layer, two additional convolution--ReLU layers, a second $2\times 2$ max-pooling layer, adaptive average pooling to a $4\times 4$ grid, flattening into $16$ tokens, and a final layer normalization.  The hidden convolutional width is taken to be 16.

We add learned positional embeddings and learned modality embeddings:
\[
    \bX_I^{(0)} = \Phi_I(I)+\boldsymbol{P}_I+\boldsymbol{M}_I,
    \qquad
    \bX_A^{(0)} = \Phi_A(A)+\boldsymbol{P}_A+\boldsymbol{M}_A,
\]
where $\boldsymbol{P}_I,\boldsymbol{P}_A\in\rr^{16\times d}$ are learned positional embeddings and $\boldsymbol{M}_I,\boldsymbol{M}_A\in\rr^{1\times d}$ are learned modality embeddings broadcast over tokens.  The fixed multimodal sequence is
\[
    \bX = \operatorname{concat}(\boldsymbol{c},\bX_I^{(0)},\bX_A^{(0)})
    \in\rr^{L\times d},
\]
where $\boldsymbol{c}\in\rr^d$ is a learned $\mathsf{CLS}$ token.

\subsection{Paper-inspired multimodal architecture}

The four paper-inspired models are empirical analogues of the CA recurrence \eqref{eq:recurrence}.  They keep a fixed encoded sequence $\bX$ and update an evolving state
\[
    \bF_t\in\rr^{L\times d},
    \qquad t=0,1,\ldots,T.
\]
The state is initialized from the fixed sequence by a learned affine map
\[
    \bF_0=\bX\bW_0^\top+\one b_0^\top.
\]
This differs from the simplified theoretical initialization $\bF_0=0$.

Each fusion layer has three possible components: an attention update, an optional re-injection of $\bX$, and a feed-forward residual block.  We use pre-layer normalization $\mathsf{LN}(\cdot)$ throughout.  In the full model, the attention update is
\[
    \bA_{t-1}
    =\mathsf{MHA}_t\bigl(\mathsf{LN}(\bF_{t-1}),\mathsf{LN}(\bX),\mathsf{LN}(\bX)\bigr),
\]
Thus, as in \eqref{eq:def_attn}, the evolving state attends back to the fixed encoded covariates.

The empirical re-injection term is a gated version of the theoretical term $\bS_{t-1}=\bW^S_{t-1}\bX$ which is used to provide stability during training.  After the attention residual,
\[
    \widetilde{\bF}_t=\bF_{t-1}+\bA_{t-1},
\]
we set
\[
    \bS_{t-1}
    =\boldsymbol{G}_{t-1}(\widetilde{\bF}_t)\odot
      \left(\bX(\bW^S_{t-1})^\top+\one (b^S_{t-1})^\top\right),
\]
where $\odot$ denotes entrywise multiplication and
\[
    \boldsymbol{G}_{t-1}(\widetilde{\bF}_t)
    =\sigma\left(\widetilde{\bF}_t(\bW^G_{t-1})^\top+
    \one (b^G_{t-1})^\top\right)
\]
is a learned sigmoid gate. The full update is therefore
\[
    \overline{\bF}_t=\widetilde{\bF}_t+\bS_{t-1},
    \qquad
    \bF_t=\overline{\bF}_t+\mathsf{FFN}_t\bigl(\mathsf{LN}(\overline{\bF}_t)\bigr),
\]
where $\mathsf{FFN}_t$ is a two-layer pointwise MLP with GELU activation, dropout, and hidden width $d_h$.

The classifier uses the final $\mathsf{CLS}$ representation:
\[
    h=\bF_T[\mathsf{CLS}],
\]
and maps $h$ to logits in $\rr^{10}$ using layer normalization followed by a two-layer MLP head with GELU activation and dropout.

\paragraph{Ablation axes.}
The four paper-inspired models differ only in whether the attention update uses $\bF_t$--$\bX$ cross-attention and whether the re-injection term $\bS_t$ is present:
\begin{itemize}
    \item \textbf{Full paper-inspired model (Inj. \& cross-attn)}: uses the cross-attention update above and includes $\bS_t$.
    \item \textbf{No injection}: uses the same $\bF_t$--$\bX$ cross-attention update but sets $\bS_t\equiv 0$.
    \item \textbf{No cross-attn}: keeps the re-injection term $\bS_t$, but replaces the cross-attention update by self-attention on the evolving state,
    \[
        \bA_{t-1}
        =\mathsf{MHA}_t\bigl(\mathsf{LN}(\bF_{t-1}),\mathsf{LN}(\bF_{t-1}),\mathsf{LN}(\bF_{t-1})\bigr).
    \]
    \item \textbf{No inj. \& no cross-attn}: uses the self-attention update above and sets $\bS_t\equiv 0$.
\end{itemize}
These ablations isolate the two mechanisms appearing in \eqref{eq:recurrence}--\eqref{eq:def_attn}: access to the fixed sequence $\bX$ through cross-attention, and layer-wise re-injection of $\bX$ through $\bS_t$.

\subsection{Bidirectional cross-attention baseline}

The standard bidirectional baseline does not maintain a single evolving state $\bF_t$ and does not re-inject a fixed sequence.  Instead, it maintains separate image and audio streams
\[
    \bU_t,\bV_t\in\rr^{16\times d},
    \qquad
    \bU_0=\Phi_I(I)+\boldsymbol{P}_I,
    \qquad
    \bV_0=\Phi_A(A)+\boldsymbol{P}_A.
\]
At layer $t$, the image stream attends to the audio stream and the audio stream attends to the image stream:
\[
    \widetilde{\bU}_t
    =\bU_{t-1}
      +\mathsf{MHA}^I_t\bigl(\mathsf{LN}(\bU_{t-1}),\mathsf{LN}(\bV_{t-1}),\mathsf{LN}(\bV_{t-1})\bigr),
\]
\[
    \widetilde{\bV}_t
    =\bV_{t-1}
      +\mathsf{MHA}^A_t\bigl(\mathsf{LN}(\bV_{t-1}),\mathsf{LN}(\bU_{t-1}),\mathsf{LN}(\bU_{t-1})\bigr).
\]
Each stream is then updated by its own feed-forward residual block:
\[
    \bU_t=\widetilde{\bU}_t+\mathsf{FFN}^I_t(\mathsf{LN}(\widetilde{\bU}_t)),
    \qquad
    \bV_t=\widetilde{\bV}_t+\mathsf{FFN}^A_t(\mathsf{LN}(\widetilde{\bV}_t)).
\]
The final representation is obtained by mean-pooling each stream and concatenating the two pooled vectors:
\[
    h=\operatorname{concat}\left(
    \frac{1}{16}\sum_{j=1}^{16}\bU_T[j],
    \frac{1}{16}\sum_{j=1}^{16}\bV_T[j]
    \right)\in\rr^{2d}.
\]
A layer-normalized two-layer MLP head maps $h$ to logits in $\rr^{10}$.  This baseline lacks the two paper-specific components: a single evolving state repeatedly attending to a fixed $\bX$, and recurrent re-injection of $\bX$.

\subsection{Training and reporting}

All models in Table~\ref{tab:avmnist_main} are trained under the same optimization budget.  The hyperparameters are
\[
    d=16,
    \qquad
    T=10,
    \qquad
    \text{number of heads}=4,
    \qquad
    d_h=32,
    \qquad
    g=4,
    \qquad
    \text{dropout}=0.1.
\]
We train for $100$ epochs using AdamW with standard specifications: learning rate $3\times 10^{-4}$ and weight decay $10^{-4}$.  The batch size is $256$. Obtaining the results displayed in Table \ref{tab:avmnist_main} totaled approximately five hours of runtime on a single NVIDIA H200 GPU.

\rev{\section{Peculiar-velocity prediction on CosmoBench}\label{app:cosmobench}

Section~\ref{sec:numerics} describes the task and its hierarchical structure. This appendix gives the data, the training objective, the metric, the inputs, the models and the protocol behind Table~\ref{tab:cosmobench}. We follow the notation of Appendix~\ref{app:avmnist_details}, except that $\bX$ now holds the halo positions and the fixed encoded input is written $\boldsymbol{M}$.

\subsection{Data, training objective and metric}

We use CAMELS, one of the benchmark's three simulation suites, with its training, validation and test splits of $600$, $200$ and $200$ clouds. Each cloud holds between $588$ and $4{,}511$ halos in a periodic box of side $B=25\,h^{-1}\mathrm{Mpc}$. Models are trained by the mean squared error over halos and evaluated by the pooled $R^2$ over all halos of a split,
\begin{equation}\label{eq:cosmo-r2}
R^2 \;=\; 1-\frac{\sum_{c}\sum_{i\le N_c}\|\hat\bv_{c,i}-\bv_{c,i}\|^2}{\sum_{c}\sum_{i\le N_c}\|\bar\bv-\bv_{c,i}\|^2},
\end{equation}
where $c$ indexes the clouds of the split and $\bar\bv$ is the mean velocity over the split.

\subsection{Inputs}

Every model receives the same per-halo input. Halo $i$ is represented by the periodic encoding of its position, $(\sin(2\pi\bx_i/B),\cos(2\pi\bx_i/B))\in\rr^6$, concatenated with $p$ engineered features $\phi(\bx_i;\bX)\in\rr^p$ computed from the positions of its neighbours. The input is thus $\Phi(\bX)\in\rr^{N\times(6+p)}$, a function of the positions alone.

\subsection{Models}

A per-halo encoder $E:\rr^{6+p}\to\rr^d$ (one two-layer MLP per feature group, fused by learned softmax weights) gives the fixed encoded input
\[
\boldsymbol{M} \;=\; E(\Phi(\bX))\in\rr^{N\times d},
\]
the analogue of $\bX$ in Section~\ref{sec:CA_model} with halos as tokens, and the state is initialized by a learned affine map $\bF_0=\boldsymbol{M}\bW_0^\top+\one b_0^\top$. The full paper-inspired model is as in Appendix~\ref{app:avmnist_details} with $\boldsymbol{M}$ in place of $\bX$, a single stream of halo features in place of the two modality streams, and a per-halo output in place of the $\mathsf{CLS}$ token. Explicitly, for layer $t=1,\ldots,T$,
\[
\bA_{t-1}=\mathsf{MHA}_t\bigl(\mathsf{LN}(\bF_{t-1}),\mathsf{LN}(\boldsymbol{M}),\mathsf{LN}(\boldsymbol{M})\bigr),
\qquad
\widetilde{\bF}_t=\bF_{t-1}+\bA_{t-1},
\]
\[
\bS_{t-1}=\boldsymbol{G}_{t-1}(\widetilde{\bF}_t)\odot\bigl(\boldsymbol{M}(\bW^S_{t-1})^\top+\one(b^S_{t-1})^\top\bigr),
\]
\[
\overline{\bF}_t=\widetilde{\bF}_t+\bS_{t-1},
\qquad
\bF_t=\overline{\bF}_t+\mathsf{FFN}_t\bigl(\mathsf{LN}(\overline{\bF}_t)\bigr),
\]
with components defined as in Appendix~\ref{app:avmnist_details}. The prediction is $\widehat{\bV}=g(\mathsf{LN}(\bF_T))$, where $g:\rr^d\to\rr^3$ is a two-layer MLP with GELU activation applied to each halo, in place of the classification head of Appendix~\ref{app:avmnist_details}. As a control, we also train the ``no inj.\ \& no cross-attn'' model of Table~\ref{tab:avmnist_main} with the same encoder $E$ and output map $g$. It replaces the cross-attention update by self-attention on the evolving state and sets $\bS_t\equiv0$, as in Appendix~\ref{app:avmnist_details}, so that each layer is the standard encoder layer of \citet{vaswani2023attentionneed}.

\subsection{Protocol}

Both models have width $d=32$, $T=4$ layers, eight heads and dropout $0.1$. At this width the paper-inspired model has fewer than half the parameters of the benchmark's graph network. Both are trained with AdamW (learning rate $3\times10^{-4}$, weight decay $10^{-4}$), with random rotations of the periodic cube as training augmentation, and with three seeds each. The checkpoint is selected by the validation pooled $R^2$, and the test split is evaluated once, at that checkpoint. Table~\ref{tab:cosmobench} reports the test results of the paper-inspired model and we note that the standard encoder model reaches comparable accuracy. The graph network's value is the pooled $R^2$ of the benchmark's released network, since the benchmark reports per-simulation $R^2$, computed for each cloud and averaged. The least-squares value is the benchmark's published one, in its own convention, in which $R^2$ is one minus the ratio of the mean squared error and the velocity variance, each averaged over clouds. In that convention the paper-inspired model reaches $0.302$.}

\end{document}